\documentclass[preprint,12pt]{colt2025} % Anonymized submission
\title[Private Bernoulli Bandits]{Optimal Regret of Bernoulli Bandits under Global Differential Privacy}

\usepackage{array}
\usepackage{makecell}
\usepackage{float,xcolor,tcolorbox}
\usepackage{mathtools}
\usepackage{amsmath, amssymb, natbib, graphicx, url}
\usepackage{algorithm}
\usepackage{algpseudocode}
\usepackage{dsfont,enumitem}
\usepackage{footnote}
\usepackage{caption}
\usepackage{subcaption}
\usepackage{mwe}
\usepackage{cancel}
\usepackage{layouts}
\makesavenoteenv{tabular}
\makesavenoteenv{table}
\newcommand{\reg}{\mathrm{Reg}}

\newcommand{\mech}{\mathcal{M}}
\newcommand{\real}{\mathbb{R}}
\newcommand{\E}{\mathbb{E}}

\newcommand{\horizon}{T}

\newcommand{\episode}{\ell}
\newcommand{\arms}{K}
\newcommand{\pulls}{N}

\def\abs#1{\left| #1 \right|}

\newcommand \pol {\ensuremath{\pi}}

\newcommand \model {\ensuremath{\nu}}
\newcommand \defn {\mathrel{\triangleq}}
\newcommand \dd {\,\mathrm{d}}
\DeclareMathOperator*{\argmax}{arg\,max}
\DeclareMathOperator*{\argmin}{arg\,min}

\newcommand \expect {\mathop{\mbox{\ensuremath{\mathbb{E}}}}\nolimits}

\newcommand \ind[1] {\mathds{1}\left\{#1\right\}}

\newcommand{\dham}{d_{\text{Ham}}}
\newcommand{\eps}{\epsilon}

\newcommand{\View}{\mathrm{View}}

\makeatletter
\newtheorem{rep@theorem}{\rep@title}
\newcommand{\newreptheorem}[2]{%
	\newenvironment{rep#1}[1]{%
		\def\rep@title{\textbf{#2} \ref{##1}}%
		\begin{rep@theorem}}%
		{\end{rep@theorem}}}
\makeatother
\newreptheorem{theorem}{Theorem}
\newreptheorem{lemma}{Lemma}
\newreptheorem{proposition}{Proposition}
\newreptheorem{assumption}{Assumption}
\newreptheorem{corollary}{Corollary}

\DeclareRobustCommand{\bigO}{\text{\usefont{OMS}{cmsy}{m}{n}O}}
\allowdisplaybreaks%
\newif\ifdoublecol
\doublecolfalse

\usepackage[nohints]{minitoc}
\renewcommand \thepart{}
\renewcommand \partname{}
\usepackage{mathrsfs}
\DeclareRobustCommand{\bigO}{%
  \text{\usefont{OMS}{cmsy}{m}{n}O}%
}

\usepackage{thmtools} 
\usepackage{thm-restate}

\usepackage{tikz}
\usetikzlibrary{automata}
\usetikzlibrary{bayesnet, arrows, calc, shapes, backgrounds,arrows,decorations.pathmorphing,fit,positioning}
\tikzset{
   container/.style = {rectangle, rounded corners, draw=yellow, dashed,
fit=#1, inner sep=6mm, node contents={}},
circle-label/.style = {circle, draw}
        }
\tikzset{box/.style={draw, diamond, thick, text centered, minimum height=0.5cm, minimum width=1cm, text width=0.9cm}}
\tikzset{line/.style={draw, thick, -latex'}}
\allowdisplaybreaks

\newcommand{\dpklucb}{{\color{blue}\ensuremath{\mathsf{DP\text{-}KLUCB}}}}
\newcommand{\dpimed}{{\color{blue}\ensuremath{\mathsf{DP\text{-}IMED}}}}
\newcommand{\dpucb}{{\ensuremath{\mathsf{DP\text{-}UCB}}}}
\newcommand{\dpse}{{\ensuremath{\mathsf{DP\text{-}SE}}}}

\newcommand{\adapklucb}{{\ensuremath{\mathsf{AdaP\text{-}KLUCB}}}}

\newcommand{\klucb}{{\ensuremath{\mathsf{KL\text{-}UCB}}}}

\newcommand{\imed}{{\ensuremath{\mathsf{IMED}}}}

\newcommand \deffn {\mathrel{\triangleq}}

\newcommand{\cF}{\mathcal{F}}

\renewcommand{\ln}{\log}

\newcommand{\ie}{\emph{i.e.}} % it must go between commas, so no space

\newcommand{\de}{\delta}

\newcommand{\nn}{\nonumber\\}
\newcommand{\n}{\nonumber}

\newcommand*{\e}{\mathrm{e}}
\newcommand{\rd}{\mathrm{d}}
\newcommand{\since}[1]{\quad\left(\mbox{#1}\right)}
\newcommand{\tmu}{\tilde{\mu}}
\newcommand{\onex}[1]{\mathbbm{1}\left[#1\right]}
\newcommand{\base}{\alpha}
\newcommand{\calT}{\mathcal{T}}
\newcommand{\dep}{\mathrm{d}_{\epsilon}}
\newcommand{\Aa}{A_a}
\newcommand{\Aap}{A_{a'}}
\newcommand{\mustar}{\mu^{\star}}

\newtheorem{reduction}{Reduction}
\usepackage{times}
\usepackage{amsmath,bbm}
\coltauthor{%
 \Name{Achraf Azize}\footnote{Part of the work is done during Achraf Azize's visit to Kyoto University.} \Email{achraf.azize@ensae.fr}\\
 \addr FairPlay Joint Team, CREST, ENSAE Paris
 \AND
 \Name{Yulian Wu}\footnote{Part of the work is done during Yulian Wu's internship at RIKEN AIP and visit to Kyoto University.} \Email{yulian.wu@kaust.edu.sa}\\
 \addr King Abdullah University of Science and Technology
 \AND
 \Name{Junya Honda} \Email{honda@i.kyoto-u.ac.jp}\\
 \addr Kyoto University and RIKEN AIP
 \AND
 \Name{Francesco Orabona} \Email{francesco@orabona.com}\\
 \addr King Abdullah University of Science and Technology
 \AND
 \Name{Shinji Ito} \Email{shinji@mist.i.u-tokyo.ac.jp}\\
 \addr The University of Tokyo and RIKEN AIP
 \AND
 \Name{Debabrota Basu} \Email{debabrota.basu@inria.fr}\\
 \addr Univ. Lille, Inria, CNRS, Centrale Lille, UMR 9189-CRIStAL
}

\begin{document}

\maketitle

\begin{abstract}%
As sequential learning algorithms are increasingly applied to real life, ensuring data privacy while maintaining their utilities emerges as a timely question. In this context, regret minimisation in stochastic bandits under $\epsilon$-global Differential Privacy (DP) has been widely studied. 
The present literature poses a gap between the regret lower and upper bounds in this setting. Unlike bandits without DP, no algorithm in this setting matches the lower bound with the same constants. There is a significant gap between the best-known regret lower and upper bound, though they ``match'' in the order.
Thus, we revisit the regret lower and upper bounds of $\epsilon$-global DP algorithms for Bernoulli bandits and improve both. 
First, we prove a tighter regret lower bound involving a novel information-theoretic quantity characterising the hardness of $\epsilon$-global DP in stochastic bandits. This quantity smoothly interpolates between Kullback–Leibler divergence and Total Variation distance, depending on the privacy budget $\epsilon$. Our lower bound strictly improves on the existing ones across all $\epsilon$ values. 
Then, we choose two asymptotically optimal bandit algorithms, \ie, $\klucb{}$ and $\imed{}$, and propose their DP versions using a unified blueprint, \ie, (a) running in arm-dependent phases, and (b) adding Laplace noise to achieve privacy. For Bernoulli bandits, we analyse the regrets of these algorithms and show that their regrets asymptotically match our lower bound up to a constant arbitrary close to 1. This refutes the conjecture that forgetting past rewards is necessary to design optimal bandit algorithms under global DP. 
At the core of our algorithms lies a new concentration inequality for sums of Bernoulli variables under Laplace mechanism, which is a new DP version of the Chernoff bound. This result is universally useful as the DP literature commonly treats the concentrations of Laplace noise and random variables separately, while we couple them to yield a tighter bound. 
Finally, our numerical experiments validate that \dpklucb{} and \dpimed{} achieve lower regret than the existing $\epsilon$-global DP bandit algorithms.
\end{abstract}
  
\section{Introduction}
Multi-armed bandit is a classical setup of sequential decision-making under partial information, where the agent collects more information about an environment by interacting with it. To understand the setting, let us consider a clinical trial, where a doctor has $K$ candidate medicines to choose from and wants to recommend ``effective" medicines to their patients. At each step $t$ of the trial, a new patient $p_t$ arrives, the doctor prescribes $a_t \in [K] \defn \{1, \dots, K\}$ one of the $K$ medicines, and observes the reaction of the patient to the medicine. The observations are quantified as rewards, such that $r_t = 1$ if the patient $p_t$ is cured and $0$ otherwise. To design an algorithm recommending ``effective" medicines, the doctor can use a regret-minimising bandit algorithm~\citep{thompson1933likelihood}, \ie, a bandit algorithm that aims to maximise the expected number of cured patients during the trial. 

Following the trial, the doctor wants to release the trial results to the public, \ie, the sequence of medicines $(a_1, \dots, a_T)$, in order to communicate the findings. However, the doctor fears that publishing the results may compromise the privacy of the patients who participated in the trial. Specifically, the rewards $(r_1, \dots, r_T)$ constitute the private information that needs to be protected, since rewards in clinical trials may reveal sensitive information about the health condition of the patients. In addition to clinical trials, many applications of bandits, such as recommendation systems~\citep{silva2022multi}, online advertisement~\citep{chen2014combinatorial}, crowd-sourcing~\citep{zhou2014optimal}, user studies~\citep{losada2022day}, hyper-parameter tuning~\citep{li2017hyperband}, communication networks~\citep{lindstaahl2022measurement},  and pandemic mitigation~\citep{libin2019bayesian}), involve sensitive user data, and thus invokes the data privacy concerns. Motivated by the privacy concerns in bandits, \textit{we study the privacy-utility trade-off in stochastic multi-armed bandits}. 

We adhere to Differential Privacy (DP)~\citep{dwork2014algorithmic} as the privacy framework, and regret minimisation~\citep{auer2002finite} in stochastic bandits as the utility measure. DP has been studied for multi-armed bandits under different bandit settings: finite-armed stochastic~\citep{Mishra2015NearlyOD, dpseOrSheffet, localDP, lazyUcb, azize2022privacy, hu2022near, azizeconcentrated, wang2024optimal}, adversarial~\citep{guha2013nearly, agarwal2017price, tossou2017achieving}, linear~\citep{hanna2022differentially, li2022differentially, azizeconcentrated}, contextual linear~\citep{shariff2018differentially, neel2018mitigating, zheng2020locally, azizeconcentrated}, and kernel bandits~\citep{pavlovic2025differentially}, among others. Most of these works were for regret minimisation, but the problem has also been explored for best-arm identification, with fixed confidence~\citep{azize2023complexity, azize2024differentially}
and fixed budget~\citep{chen2024fixed}. The problem has also been studied under three different DP trust models: (a) global DP where the users trust the centralised decision maker~\citep{Mishra2015NearlyOD, shariff2018differentially, dpseOrSheffet, azize2022privacy, hu2022near}, (b) local DP where each user deploys a local perturbation mechanism to send a ``noisy" version of the rewards to the policy~\citep{basu2019differential,localDP, zheng2020locally, han2021generalized}, and (c) shuffle DP where users still feed their data to a local perturbation, but now they trust an intermediary to apply a uniformly random permutation on all users' data before sending to the central servers~\citep{tenenbaum2021differentially, garcelon2022privacy, chowdhury2022shuffle}. 

% Thanks to privacy amplifications due to permutation, each user needs to add less noise to each data, and thus, the utility is better than local DP while the privacy guarantee is stronger than global DP.
% Under different trust models: global, local and shuffle. 
% Under different utility measures: regret minimisation, best arm identification under fixed confidence, under fixed confidence.
In this paper, \textit{we focus on $\epsilon$-pure DP, under a global trust model, in stochastic finite-armed bandits, with the aim of regret minimisation}. 

\paragraph{Related Works.} This problem setting has been studied by~\citet{Mishra2015NearlyOD, dpseOrSheffet, lazyUcb, azize2022privacy, hu2022near}.  \dpucb{}~\citep{Mishra2015NearlyOD} was the first DP version of the Upper Confidence Bound (UCB) algorithm~\citep{auer2002finite} that achieved logarithmic regret. \dpucb{} uses the tree-based mechanism~\citep{dpContinualObs, treeMechanism2} to compute privately the sum of rewards. For each arm, the tree mechanism maintains a binary tree of depth $\log(\horizon)$ over the $\horizon$ streaming reward observations. As a result, the noise added to the sum of rewards has a scale of $\mathcal{O} \left( \log^{2.5}(\horizon)/{\epsilon} \right)$ for rewards in $[0,1]$. \dpucb{} builds a high probability upper bound on the means using the noisy sum of rewards to design a private UCB index and yields a regret bound of $\bigO \left( \sum_a \frac{\log(T)}{\Delta_a} + \textcolor{blue}{ K \log^{2.5}(T)/\epsilon} \right)$, where $\Delta_a$ is the difference between the mean reward of an optimal arm and arm $a$. This upper bound has an additional $\log^{1.5}(T)$ factor compared to the $\textcolor{blue}{\Omega(K \log(T)/\epsilon)}$ regret lower bound, first proved by~\cite{shariff2018differentially}.  

% On the other hand, for Bernoulli bandits, the additional regret lower bound due to privacy is $\Omega\left( \frac{K \log(T)}{\epsilon} \right)$, as discussed in Section~\ref{sec:4:reg_low}. This means that \dpucb{} has an extra multiplicative $\log^{1.5}(T)$ regret compared to the lower bound.

% where each node in this tree holds an i.i.d sample from a Laplace distribution with zero mean and scale $\left( \log(\horizon)/\epsilon \right)$. At each step $\step$, the mechanism yields the sum of the first $\step$ observations and the $\log(\horizon)$ nodes on the root-to-the-leaf path in the binary tree as the private empirical mean. As a result, the noise added to the UCB index per time-step is $\mathcal{O} \left( \log(\horizon)^{2.5}/{\epsilon} \right)$, %(due to a union bound in the analysis)
% which is responsible for the extra multiplicative factor $\log(\horizon)^{1.5}$ in regret compared to the lower bound.  

\dpse{}~\citep{dpseOrSheffet} was the first DP bandit algorithm to eliminate the additional multiplicative factor $\log^{1.5}(\horizon)$ in the regret. \dpse{} is a DP version of the Successive Elimination algorithm~\citep{se}. \dpse{} runs in \textit{independent} episodes. At each episode, the algorithm explores a set of active arms uniformly. At the end of an episode, \dpse{} eliminates provably sub-optimal arms, but \emph{only uses the samples collected at the current episode} to decide the arms to eliminate. Due to the addition of the Laplace noise to the sum of rewards, each arm is explored longer, resulting in the additional $\textcolor{blue}{\mathcal{O}\left(K\log(T)/\epsilon \right)}$ in the regret.

% The additional exploration at each phase can be derived from the concentration of Laplace random variables, and is responsible for the additional $\mathcal{O}\left(K\log(T)/\epsilon \right)$ in the regret, which matches the regret lower bound.
% However, \dpse{} has two main drawbacks. The algorithm is not anytime since the \dpse{} needs to know the $\horizon$ in advance to decide the length of each episode. Also, in general, Successive Elimination algorithms have the drawback of not matching the problem-dependent regret lower bound exactly and committing to one arm when all the other arms have been eliminated.
% The drawbacks were that the algorithm was not anytime and was optimal only asymptotically, \,  i.e. the horizon should be big enough. 
A careful reading of \dpse{} suggests that running the algorithm in independent episodes while forgetting the previous samples shreds the extra $\log^{1.5}(\horizon)$ in the regret. These ingredients, \ie, running in independent phases with forgetting and adding Laplace noise, have been further adapted to UCB in~\citet{lazyUcb, azize2022privacy} and to Thompson Sampling in~\citet{hu2022near}. The state-of-the art regret upper bound is thus $\bigO\left (\sum_a \log(T)/\min\{\Delta_a, \textcolor{blue}{\epsilon}\} \right)$. Similarly, \cite{azize2022privacy} use the same three components of doubling, forgetting, and Laplace mechanism to propose \adapklucb{} that achieves $\bigO\left (\frac{C_1(\tau) \Delta_a }{\min\{ \mathrm{kl}(\mu_a, \mu^*) ,  \textcolor{blue}{C_2 \epsilon \Delta_a}\}} \log(\horizon)\right)$ regret for $\tau>3$. Though the regret of \adapklucb{} is order-optimal, we observe that $C_1(\tau)$ and $C_2$ are not universal constants, \ie, may depend on the environment. %This implies that in practice, the algorithm might yield worst performance than the lower bound by some factors depending on the environment.

On the other hand, \cite{azize2022privacy} improve the problem-dependent regret lower bound of~\cite{shariff2018differentially} to $\Omega\left(\sum_a \log(T)\frac{\Delta_a}{\min(d_a, \textcolor{blue}{6 \epsilon t_a})}\right)$. Here, $d_a$ is the Kullback-Leibler (KL) indistinguishability gap for arm $a$ characterising the hardness of non-private bandits~\citep{lai1985asymptotically}, and $t_a$ is a ``Total Variation" (TV) version of $d_a$ characterising the hardness of private bandits. For Bernoulli bandits, $t_a = \Delta_a$ and $d_a \approx \Delta_a^2$. Under these approximations, the lower bound of~\cite{azize2022privacy} recovers that of~\cite{shariff2018differentially}, and the regret upper bounds of~\cite{dpseOrSheffet, azize2022privacy, hu2022near} match approximately the lower bound. However, the approximation $d_a \approx \Delta_a^2$ can be arbitrarily bad, exposing a gap between the state-of-the-art upper and lower bounds in DP bandits.
% This gap is reflected in the numerical performance of \adapklucb{} (Figure 3,~\citep{azize2022privacy}). 
This motivates us to ask:
\begin{center}
    Q1. \textit{Can we derive matching regret upper and lower bounds \emph{up to the same constant} for $\epsilon$-global DP Bernoulli bandits?}
\end{center}

Additionally, following the triumph of doubling and forgetting as an algorithmic blueprint in DP bandits, \cite{lazyUcb} conjectured that forgetting is necessary for designing any $\epsilon$-global DP bandit algorithm with an optimal regret upper bound matching the lower bound. 
% However, we do not know any formal proof supporting this claim. T
Thus, we wonder:
\begin{center}
    Q2. \textit{Is it possible to design an optimal $\epsilon$-global DP bandit algorithm without applying \emph{forgetting}?}
\end{center}

\paragraph{Aim and Contributions.} To address these questions, we revisit regret minimisation for Bernoulli bandits under $\epsilon$-global DP. Our main goal is to provide matching regret upper and lower bounds \emph{up to the same constant}. Answering this question leads to the following contributions:

1. \textit{Tighter regret lower bound}: In Theorem~\ref{thm:low_bound}, we provide a new asymptotic regret lower bound for any consistent $\epsilon$-global DP policy. This result is a strict improvement over the lower bound of~\cite{azize2022privacy} for all $\epsilon$.
This lower bound depends on a new information-theoretic quantity $\mathrm{d}_\eps$ (Eq.~\eqref{eq:d_eps}) interpolating smoothly between  KL and TV depending on $\epsilon$. 
This quantity also indicates a smooth transition between high and low privacy regimes, where the impact of DP does and does not appear, respectively. 
In addition to the existing techniques, our proof applies a new ``double change" of environment idea to couple the impacts of DP and bandit feedback (Lemma~\ref{lem:doub_env}).

2. \textit{Tighter concentration inequality}: In Proposition~\ref{prop:conc}, we provide a DP version of Chernoff-style concentration bound for sum of Bernoullis with added Laplace noise. $\mathrm{d}_\eps$ naturally appears in this bound. Also, the bound suggests that as long as the number of summed Laplace noise is negligible compared to the number of summed Bernoullis, the effect of the noise is comparable to having one Laplace noise in the dominant term of the bound. This bound is universally interesting for DP literature as the concentrations of random variables and Laplace noises are commonly treated separately unlike the coupled treatment in Proposition~\ref{prop:conc}.

3. \textit{Algorithm design and tighter regret upper bounds}: Based on the concentration bound of Proposition~\ref{prop:conc}, we modify the generic blueprint used by~\citet{dpseOrSheffet, azize2022privacy, hu2022near}. We (a) get rid of ``reward-forgetting" and thus summing all rewards at each phase, and (b) develop new private indexes using $\mathrm{d}_\eps$. We also run the algorithms in geometrically increasing arm-dependent batches, with ratio $\alpha >1$.
We instantiate these modifications for two algorithms that achieve constant optimal regrets withour privacy, \ie, KL-UCB and IMED, to propose \dpklucb{} and \dpimed{} (Algorithm~\ref{alg_DP}). We analyse the regret of both algorithms (Theorem~\ref{thm:upp_bound}) and show that their regret upper bounds match asymptotically the regret lower bound of Theorem~\ref{thm:low_bound} up to the constant $\alpha$, which can be set arbitrarily close to $1$.

We also validate experimentally that our algorithms \dpimed{} and~\dpklucb{} achieve the lowest regret among DP bandit algorithms in the literature. Finally, in Appendix~\ref{app:adap_cont}, we extend the adaptive continual release model of~\cite{jain2023price} to bandits and show that this definition is equivalent to the classic $\epsilon$-global DP notion adopted in the DP bandit literature~\citep{Mishra2015NearlyOD, azize2022privacy, azizeconcentrated}. This result can be of independent interest.

\section{Background}\label{sec:backgr}
In this section, we formalise the essential components of our work, \ie, the stochastic bandit problem, regret minimisation as a utility measure, and Differential Privacy (DP) as the privacy constraint.

\paragraph{Stochastic Bandits.} A stochastic bandit problem is a sequential game between a policy $\pi$ and a stochastic environment $\nu$~\citep{thompson1933likelihood,lai1985asymptotically}.  
The game is played over $T$ rounds, where $T \in \{1, 2, \dots \}$ is a natural number called the \textit{horizon}. At each step $t \in \{1, \dots, T\}$, the policy $\pi$ chooses an action $a_t \in [K]$. The stochastic environment, which is a collection of distributions $\nu \defn (P_a : a \in [K])$, samples a reward $r_t \sim P_{a_t}$ and reveals it to the policy $\pi$.  The interaction between
the policy $\pi$ and environment $\nu \defn (P_a : a \in [K])$ over $T$ steps induces a probability measure on the
sequence of outcomes $H_T \defn (a_1, r_1, a_2, r_2, \dots , a_T, r_T)$. Let each $P_a$ be a probability measure
on $(\real, \mathcal{B}(\real))$ with $\mathfrak{B}$ being the Borel set. For each $t \in [\horizon]$, let $\Omega_t = ([\arms]\times \real)^t \subset \real^{2t}$ and $\mathcal{F}_t = \mathfrak{B}(\Omega_t)$. First, we formalise the definition of a policy.

\begin{definition}[Policy]\label{def:pol}
    A policy $\pol$ is a sequence $(\pol_t)_{t=1}^{\horizon}$ , where $\pol_t$ is a probability kernel from $(\Omega_t , \mathcal{F}_t)$ to $([\arms], 2^{[\arms]} )$. Since $[\arms]$ is discrete, we adopt the convention that for $a \in [\arms]$, $\pol_t (a \mid  a_1 , r_1 , \dots , a_{t - 1} , r_{t - 1} ) = \pol_t (\{a\} \mid a_1 , r_1 , \dots , a_{t - 1} , r_{t - 1} )~.$
\end{definition}

The interaction probability measure on $(\Omega_T , \mathcal{F}_T)$ depends on the environment and the policy:
(a) the conditional distribution of action $a_t$ given $ a_1 , r_1 , \dots , a_{t - 1} , r_{t - 1}$ is $\pi(a_t \mid H_{t-1})$, and (b) the conditional distribution of reward $r_t$ given $ a_1 , r_1 , \dots , a_{t - 1} , r_{t - 1}, a_t$ is $P_{a_t}$. To construct the probability measure, let $\lambda$ be a $\sigma$-finite measure on $(\real, \mathcal{B}(\real))$ for which $P_a$ is absolutely continuous with respect to $\lambda$ for all $a \in [\arms]$. Let $p_a = dP_a/d\lambda$ be the Radon–Nikodym derivative of $P_a$ with respect to $\lambda$. Letting $\rho$ be the counting measure with $\rho(B) = |B|$, the density $p_{\model \pol} : \Omega_\horizon \rightarrow \real$ can now be defined with respect to the product measure $(\rho \times \lambda)^\horizon$ by
\begin{equation}
     p_{\model \pol}(a_1 , r_1 , \dots , a_\horizon , r_\horizon ) \deffn \prod_{t=1}^\horizon \pol_t(a_t \mid a_1 , r_1 , \dots , a_{t-1} , r_{t-1} ) p_{a_t} (r_t)
\end{equation}
and $\mathbb{P}_{\model \pol} (B) \deffn \int_B p_{\model \pol}(\omega) (\rho \times \lambda)^\horizon (\dd \omega)$ for all $B \in \mathcal{F}_\horizon$. So $(\Omega_\horizon, \mathcal{F}_\horizon,  \mathbb{P}_{\model \pol})$ is a probability space over histories induced by the interaction between $\pol$ and $\model$.

\paragraph{Regret minimisation.} We study regret minimisation as the utility measure~\citep{lai1985asymptotically}. Informally, the regret of a policy is the deficit suffered by the learner relative to the optimal policy which knows the environment and always plays the optimal arm. Let $\model = (P_a : a \in [\arms])$ a bandit instance and define $\mu_a(\nu) = \int_{- \infty}^\infty x \dd P_a(x)$ the mean of arm $a$'s reward distribution. We assume throughout that $\mu_a(\nu)$ exists and is finite for all actions. Let $\mu^\star(\nu) = \max_{a \in [\arms]} \  \mu_a(\nu)$ the largest mean among all the arms. The regret of policy $\pol$ on bandit instance $\nu$ is
\begin{equation}\label{eq:DefReg}
    \reg_{\horizon}(\pi, \nu) \deffn \horizon \mu^\star(\nu)-\expect_{\model \pol}\left[\sum_{t=1}^{\horizon} r_{t}\right] = \sum_{a = 1}^\arms \Delta_{a}(\nu) \mathbb{E}_{\model \pol}\left[\pulls_{a}(\horizon)\right].
\end{equation}
where  $\pulls_{a}(\horizon) \deffn \sum_{t=1}^{\horizon} \ind{a_{t}=a} $ and $\Delta_{a}(\nu) \deffn \mu^\star(\nu) - \mu_a(\nu)$.
The expectation is taken with respect to the probability measure $\mathbb{P}_{\model \pol}$ on action-reward sequences induced by the interaction of $\pol$ and $\model$. 
Hereafter, we drop the dependence on $\nu$ when the context is clear.
% The regret can be decomposed in terms of the loss due to pulling each of the sub-optimal arms, where $\pulls_{a}(\horizon) \deffn \sum_{t=1}^{\horizon} \ind{a_{t}=a} $ is the number of times the arm $a$ is played till $\horizon$, and $\Delta_{a}(\nu) \deffn \mu^\star(\nu) - \mu_a(\nu)$ is the sub-optimality gap. Hereafter, we drop the dependence on $\nu$ when the context is clear.

For many classes of bandits, it is possible to define a notion of instance-dependent optimality that characterises the hardness of regret minimisation. Specifically, for any consistent policy $\pi$ over a class of bandits $\mathcal{E} \deffn \mathcal{M}_1 \times \dots \times \mathcal{M}_\arms$, \ie, a policy $\pi \in \Pi_\mathrm{cons}(\mathcal{E})$ verifies $ \lim_{T \rightarrow \infty} \frac{\reg_T(\pol, \nu)}{T^p} = 0$ for all $\nu \in \mathcal{E}$ and all $p > 0$,
then the regret of $\pi$ on any environment $\nu \in \mathcal{E}$ is lower bounded by
\begin{equation}\label{eq:prb_dep_low}
        \liminf_{T \rightarrow \infty} \ \frac{\reg_T(\pi, \nu)}{\log(T)} \geq \sum_{a: \Delta_a(\nu) > 0} \frac{\Delta_a(\nu)}{ \mathrm{KL}_\mathrm{inf} (P_a, \mu^\star, \mathcal{M}_a)},
\end{equation}
where $
  \mathrm{KL}_\mathrm{inf} (P, \mu^\star, \mathcal{M}) \deffn \inf_{P' \in \mathcal{M}} \left\{ \mathrm{KL}(P,P'): \mu(P') > \mu^\star \right\} $,
  and $\mathrm{KL}$ is the Kullback-Leibler divergence, \ie, for two probability distributions $P, Q$ on $(\Omega, \mathcal{F})$, the KL divergence is $\mathrm{KL}(P, Q) \triangleq \int \log\left(\frac{\dd P}{\dd Q} (\omega) \right) \dd P (\omega)$ when $P \ll Q$, and $+\infty$ otherwise. The lower bound of Equation~\eqref{eq:prb_dep_low} is tight for many classes of bandits, and the ``KL-inf'' is a fundamental quantity that characterises the complexity of regret minimisation in bandits. 

% Next, we focus on a specific class of bandits: Bernoulli bandits.

\paragraph{Bernoulli bandits.} A Bernoulli bandit is a stochastic environment where the distribution of each arm follows a Bernoulli distribution. Let $\mu \in [0,1]^K$, then $\nu^\mathcal{B}_\mu = (\textrm{Bernoulli}(\mu_a) : a \in [K])$ is a Bernoulli environment. For Bernoulli bandits, $ \mathrm{KL}_\mathrm{inf} (P_a, \mu^\star, \mathcal{M}_a) = \mathrm{kl}(\mu_a, \mu^\star)$, where $\mathrm{kl}$ is the relative entropy between Bernoullis, \ie, $\mathrm{kl}(p,q) \defn p \log(p/q) + (1 - p) \log((1 - p)/(1 - q))$
for $p,q \in [0,1]$ and singularities are defined by taking limits. Using the ``optimism in the face of uncertainty" principle, it is possible to design algorithms tailored for Bernoulli bandits, such as KL-UCB~\citep{KLUCBJournal} or IMED~\citep{honda2015non}, that achieve the lower bound of Equation~\eqref{eq:prb_dep_low} asymptotically, \emph{up to the same constant}. 

\paragraph{Differential Privacy (DP).} 
% Differential Privacy (DP)~\citep{dwork2014algorithmic} is considered the gold standard for privacy-preserving data analysis. It ensures that the same conclusions can be drawn regardless of whether an individual chooses to participate in the dataset. 
DP~\citep{dwork2014algorithmic} guarantees that any sequence of algorithm outputs is ``essentially'' equally likely to occur, regardless of the presence or absence of any individual. The probabilities are taken over random choices made by the algorithm, and ``essentially" is captured by closeness parameters that we call privacy budgets. Formally, DP is a constraint on the class of mechanisms, where a mechanism $\mech$ is a randomised algorithm that takes as input a dataset $D \deffn \{ x_1, \dots, x_T\} \in \mathcal{X}^T$ and outputs $o \sim \mech_D$. The probability space is over the coin flips of the mechanism $\mech$. Given some event $E$ in the output space $(\mathcal{O}, \cF)$, we note $\mech_D(E) \deffn \mech(E | D)$ the probability of observing the event $E$ given that the input of the mechanism is $D$.
% Thus, DP can be seen as a constraint on the class of mechanisms, where a mechanism $\mech$ is a randomised algorithm. $\mech$ takes as input a dataset $D \deffn \{ x_1, \dots, x_T\} \in \mathcal{X}^T$, which is a collection of $T$ data points from the input universe $\mathcal{X}$. $\mech$ outputs a distribution $\mech_D \in \mathcal{P}(\mathcal{O})$, where $\mathcal{P}(\mathcal{O})$ is the set of distributions over the probability space $(\mathcal{O}, \cF)$, and $\mathcal{O}$ is the output space. The probability space is over the coin flips of the mechanism $\mech$. Given some measurable event $E$ in  $(\mathcal{O}, \cF)$, we note $\mech_D(E) \deffn \mech(E | D)$ the probability of observing the event $E$ given that the input of the mechanism is $D$. We are now ready to formally define DP.
\begin{definition}[$\epsilon$-DP~\citep{Dwork_Calibration}]\label{Def_DP}
    A mechanism $\mech$ satisfies $\epsilon$-DP for a given $\epsilon \geq 0$, 
         if 
        \begin{align}\label{eq:DP}
            \forall D \sim D', \; \forall E \in \mathcal{O} , \; \mech_{D}(E) \leq e^\epsilon\mech_{D'}(E),
        \end{align}
    where $D \sim D'$ if and only if $\dham(D, D') \deffn \sum_{t=1}^T \ind{D_t \neq D'_t} \leq 1$, \ie, $D$ and $D'$ differ by at most one record, and are said to be neighbouring datasets.
\end{definition}

% DP is a worst-case constraint on the class of randomised mechanisms. A mechanism $\mech$ satisfies DP if the mechanism behaves ``similarly" on \emph{all} neighbouring datasets $D$ and $D'$, even for very ``unlikely" realisations of the mechanism $\mech$. 
DP is widely adopted as a privacy framework since the definition enjoys different interesting properties, and can be achieved by combining simple basic mechanisms. Hereafter, we mainly use two important DP properties: post-processing (Proposition~\ref{prp:post_proc}) and group privacy (Proposition~\ref{prp:grp_priv}), and we use the Laplace mechanism (Theorem~\ref{thm:laplace}) to achieve DP.

% Definition + Laplace Mechanism + Post-processing

\paragraph{Bandits under DP.} We extend DP to bandits by reducing a policy $\pi = (\pi_1, \dots, \pi_T)$ to a ``batch'' mechanism $\mathcal{M}^\pi$~\citep{azizeconcentrated}. Different ways of reducing a policy to a batch mechanism differ on the input representation and the nature of the mechanism. 

(a) In Table DP, we represent each user $u_t$ by the vector $x_t \deffn (x_{t, 1}, \dots, x_{t, K}) \in \real^K$ of all its $K$ ``potential rewards." This is the vector of potential rewards since the policy only observes $r_t \deffn x_{t, a_t}$ when it recommends action $a_t$. In Table DP, the induced ``batch" mechanism $\mech^\pi$ from the policy $\pi$ takes as input a table of rewards $\textbf{x} \triangleq \lbrace(x_{t, i})_{i \in [\arms]} \rbrace_{t \in [\horizon]}  \in (\real^\arms)^\horizon$, and outputs a sequence of actions $\textbf{a} \defn (a_1, \dots, a_T) \in [K]^T$ with probability 
% Specifically, $\mech^\pol$ is a mechanism that takes as input a table of rewards $\textbf{x} \in (\real^\arms)^\horizon$, and outputs a sequence of actions $(a_1, \dots, a_T) \in [K]^\horizon$, with probability 
% \begin{align*}
%  \mech^\pol: ~~ (\real^\arms)^\horizon &\rightarrow \mathcal{P}([K]^\horizon)\\
%  \textbf{x}~ &\rightarrow \mech^\pol_\textbf{x} \, ,
% \end{align*}
% \begin{equation}\label{eq:mech_tab}
$\mech^\pol_\textbf{x}(\textbf{a}) \deffn \prod_{t = 1}^\horizon \pol_t\left(a_t | a_1, x_{1, a_1}, \dots a_{t-1}, x_{t -1, a_{t-1}}\right).$
% \end{equation}    
This is the probability of observing $(a_1, \dots, a_T)$ when $\pi$ interacts with the table of rewards $\textbf{x}$. $\mech^\pol_\textbf{x}$ is a distribution over sequences of actions since $\sum_{\mathbf{a} \in [K]^T} \mech^\pol_\textbf{x}(\mathbf{a}) = 1$.

(b) In View DP, the induced ``batch" mechanism from the policy $\pi$ takes as input a list of rewards and outputs a sequence of actions. The difference is in the representation of the input dataset. Since in bandits, the policy only observes the reward corresponding to the action chosen, another natural choice for the input is a list of rewards, \ie, $\mathbf{r} \deffn \{r_1, \dots, r_T \} \in \real^T$. Thus, now the induced ``batch" mechanism $\mathcal{V}^\pol$ from the policy $\pi$ takes as input a list of rewards $\mathbf{r} \deffn \{r_1, \dots, r_T \} \in \real^T$, and outputs a sequence of actions $\mathbf{a} \deffn (a_1, \dots, a_T) \in [K]^T$, with probability
% \begin{equation}\label{eq:mech_view}
$\mathcal{V}^\pol_{\mathbf{r}}(\mathbf{a}) \defn \prod_{t = 1}^\horizon \pol_t(a_t | a_1, r_1, \dots a_{t-1},r_{t-1}).$
% \end{equation}
% Specifically,
% \begin{align*}
%  \mathcal{V}^\pol: ~~ \real^\horizon &\rightarrow \mathcal{P}([K]^\horizon)\\
%  \mathbf{r}~ &\rightarrow \mathcal{V}^\pol_{\mathbf{r}}\, ,
% \end{align*}
This is the probability of observing $\textbf{a}$ when $\pi$ interacts with $\textbf{r}$. $\mathcal{V}^\pol_\mathbf{r}$ is a distribution over sequences of actions, since $\sum_{\mathbf{a} \in [K]^T} \mathcal{V}^\pol_{\mathbf{r}}(\mathbf{a}) = 1$. 

\begin{definition}[Table DP and View DP~\citep{azizeconcentrated}]\label{def:table_view_dp} (a)  A policy $\pol$ satisfies $\epsilon$-Table DP if and only if $\mathcal{M}^\pol$ is $\epsilon$-DP. (b) A policy $\pol$ satisfies $\epsilon$-View DP if and only if $\mathcal{V}^\pol$ is $\epsilon$-DP.
    % \begin{itemize}
    %     \item A policy $\pol$ satisfies $\epsilon$-Table DP if and only if $\mathcal{M}^\pol$ is $\epsilon$-DP.
    %     \item A policy $\pol$ satisfies $\epsilon$-View DP if and only if $\mathcal{V}^\pol$ is $\epsilon$-DP.
    % \end{itemize}
\end{definition}\vspace{-0.5em}

Table DP and View DP have been formalised in~\citet{azizeconcentrated}, but have been used interchangeably in the private bandit literature, e.g. Table DP in~\citet{guha2013nearly, Mishra2015NearlyOD, neel2018mitigating} and View DP in~\citet{dpseOrSheffet, lazyUcb, azize2022privacy}. For $\epsilon$-pure, these two definitions are equivalent.

\begin{proposition}[$\epsilon$-global DP, Proposition 1 in~\citet{azizeconcentrated}]\label{prop:tab_view}
 For any policy $\pi$, we have that: 
    $\pi$ is $\epsilon$-Table DP $\Leftrightarrow$ $\pi$ is $\epsilon$-View DP. 
\end{proposition}\vspace{-0.5em}

Thus, we refer to any policy that verifies $\epsilon$-Table DP or $\epsilon$-View DP as an \textbf{$\epsilon$-global DP} policy. In Appendix~\ref{app:adap_cont}, we also extend the interactive DP definition of~\citet{jain2023price} to bandits and show that $\epsilon$-global DP is equivalent to it. In the following, our main goal is to design an $\epsilon$-global DP policy that minimises the regret $\reg_T(\pi, \nu)$ on any Bernoulli environment $\nu$.

% we also show that, for pure DP, $\epsilon$-global DP is equivalent to an extension of the adaptive continual release model to bandits~\

% In simple words, a policy is View DP if the sequence of output actions is "essentially" the same when the policy interacts with two neighbouring, fixed-in-advance lists of rewards.

% View DP is a formalisation of the definition adopted in~\cite{dpseOrSheffet, lazyUcb, hanna2022differentially}. View DP deals with the three challenges of defining DP in the following: 
 
% Just define view dp

% say that it is equivalent to Table DP (satml paper)

% new claim: View DP is also equivalent to adaptive DP, i.e. when rewards are chosen adaptively by a privacy adversary

% \paragraph{Goal.} Our main goal is to design an $\epsilon$-global DP policy that minimises the regret $\reg_T(\pi, \nu)$ on any Bernoulli environment $\nu$.

% Summarise + put many things in the appendix
\section{Regret Lower Bound under $\epsilon$-global DP}\label{sec:lower_bound}
In this section, we present a new regret lower bound for Bernoulli bandits under $\epsilon$-global DP. We compare this result to the lower bound of~\cite{azize2022privacy}, and provide a proof.

\begin{theorem}[Regret lower bound under $\epsilon$-global DP]\label{thm:low_bound}
For every $\epsilon$-global DP consistent policy over the class of Bernoulli bandits, we have
\begin{equation}\label{eq:reg_low}
    \liminf _{T \rightarrow \infty}\ \frac{\mathrm{Reg}_{T}(\pi, \nu)}{\log (T)} \geq \sum_{a: \Delta_{a}>0} \frac{\Delta_a}{\textrm{d}_\epsilon \left(\mu_a, \mu^\star\right)},
\end{equation}
where 
%\begin{equation}\label{eq:d_eps}
%    \mathrm{d}_\epsilon(\mu_a, \mu^\star) \defn \inf_{\mu \in [\mu_a \wedge \mu^\star , \mu_a \vee \mu^\star]}\left \{\epsilon \abs{\mu - \mu_a} + \mathrm{kl}(\mu, \mu^\star)\right\}~.
%\end{equation}
\begin{equation}\label{eq:d_eps}
    \mathrm{d}_\epsilon(x, y) \defn \inf_{z \in [x \wedge y , x \vee y]}\left \{\epsilon \abs{z - x} + \mathrm{kl}(z, y)\right\}~,\quad x\in \mathbb{R},\,y\in[0,1].
\end{equation}

\end{theorem}
For any suboptimal arm $a$, $\mu^\star > \mu_a$ and $\mathrm{d}_\epsilon(\mu_a, \mu^\star) = \inf_{\mu \in [\mu_a, \mu^\star]}\left \{\epsilon (\mu - \mu_a) + \mathrm{kl}(\mu, \mu^\star)\right\}$.

\paragraph{Implications of Theorem~\ref{thm:low_bound}.}\,

(a) Theorem~\ref{thm:low_bound} improves the lower bound of~\cite{azize2022privacy}.
Specifically, Theorem 3 in~\cite{azize2022privacy}, adapted to Bernoulli bandits, gives a lower bound of
\begin{equation}\label{eq:reg_low_22}
    \liminf _{T \rightarrow \infty}\  \frac{\mathrm{Reg}_{T}(\pi, \nu)}{\log (T)} \geq \sum_{a: \Delta_{a}>0} \frac{\Delta_a}{\min \{\mathrm{kl}(\mu_a, \mu^\star), 6 \epsilon \Delta_a \}}~.
\end{equation}
Theorem~\ref{thm:low_bound} is a strict improvement on the lower bound of~\cite{azize2022privacy} since $\textrm{d}_\epsilon \left(\mu_a, \mu^\star\right) \leq \min \{\mathrm{kl}(\mu_a, \mu^\star), \epsilon \Delta_a \} \leq  \min \{\mathrm{kl}(\mu_a, \mu^\star), 6 \epsilon \Delta_a \}$, for any $\epsilon, \mu_a$ and $\mu^\star$.

(b) Solving the constrained optimisation problem defining $\mathrm{d}_\epsilon$ for Bernoulli variables gives
\begin{equation}
\mathrm{d}_\epsilon(\mu_a,\mu^\star)=\left\{\begin{aligned}
    &\mathrm{kl}\left(\mu_a,\mu^\star\right) \quad  \text{if} \quad \epsilon \geq \log\frac{\mu^\star}{\mu_a} + \log \frac{1 - \mu_a}{1 - \mu^\star}\\
        &\mathrm{kl}\left(\frac{\mu^\star}{\mu^\star + (1 - \mu^\star)e^{\epsilon}}, \mu^\star\right)+ \epsilon \left( \frac{\mu^\star}{\mu^\star + (1 - \mu^\star)e^{\epsilon}}-\mu_a\right)   \quad \text{if not} 
    \end{aligned}\right.
\end{equation}
This suggests the existence of two privacy regimes: a low privacy regime when $\epsilon \geq \log\frac{\mu^\star}{\mu_a} + \log \frac{1 - \mu_a}{1 - \mu^\star}$, and a high privacy regime when  $\epsilon \leq \log\frac{\mu^\star}{\mu_a} + \log \frac{1 - \mu_a}{1 - \mu^\star}$. In the low privacy regime, $\mathrm{d}_\epsilon(\mu_a,\mu^\star)$ just reduces to the non-private $\mathrm{kl}\left(\mu_a,\mu^\star\right)$, and privacy can be achieved for \emph{free}. In the high privacy regime, $\mathrm{d}_\epsilon(\mu_a,\mu^\star)$ can be written as the sum of two terms, \ie, a KL term between Bernoullis with means $\frac{\mu^\star}{\mu^\star + (1 - \mu^\star)e^{\epsilon}}$ and $ \mu^\star$, and TV distance between Bernoullis with means $\frac{\mu^\star}{\mu^\star + (1 - \mu^\star)e^{\epsilon}}$ and $ \mu_a$. At the limit, we have that $d_\eps(\mu_a, \mu^\star) \sim_{\epsilon \rightarrow 0} \eps \times \Delta_a$.
% and $d_\eps(\mu_a, \mu^\star) \sim_{\epsilon \rightarrow \infty} \mathrm{kl}(\mu_a, \mu^\star)$.

% (maybe give the  exact expression of d eps and explain the regimes where there is an improvement) % (b) Interpolates smoothly between the TV and the KL
(c) Theorem~\ref{thm:low_bound} can be generalised beyond Bernoulli bandits: for a class $\mathcal{E}$ of unstructured stochastic bandits, \ie,  $\mathcal{E} \defn \mathcal{M}_1 \times \dots \times \mathcal{M}_K$, the lower bound becomes
\begin{equation}
    \liminf _{T \rightarrow \infty}\  \frac{\mathrm{Reg}_{T}(\pi, \nu)}{\log (T)} \geq \sum_{a: \Delta_{a}>0} \frac{\Delta_a}{\textrm{d}_\textrm{inf}\left(P_{a}, \mu^\star, \mathcal{M}_a, \epsilon\right)},
\end{equation}
where $\textrm{d}_\textrm{inf}\left(P_{a}, \mu^\star, \mathcal{M}_a, \epsilon\right) \triangleq \inf_{P' \in \mathcal{M}_a} \left\{\textrm{d}_\epsilon^{\mathcal{M}_a}(P_a, P') :  \mu(P') > \mu^\star \right\},$ and 
\[
\textrm{d}_\epsilon^{\mathcal{M}_a}(P_a, P') \triangleq \inf_{Q \in \mathcal{M}_a} \{ \epsilon  \textrm{TV}(P_a, Q) + \textrm{KL}(Q, P'):  \mu(P_a) \leq \mu(Q) \leq \mu(P')\},
\]
for $P_a, P' \in \mathcal{M}_a$ such that $\mu(P_a) \leq \mu(P')$.

% \dbcomment{Is it a divergence? Can we comment on some of its properties?} Yes, it is a divergence i.e. $d_\eps(p,q) \geq 0, d_\eps(p,q) = 0  \Leftrightarrow  p=q$.

Before providing the proof, we introduce maximal couplings.
\begin{definition}[Maximal Couplings]\label{def:max_coup} Let $\mathbb{P}$ and $\mathbb{Q}$ be two probability distributions that share the same $\sigma$-algebra and $\Pi(\mathbb{P}, \mathbb{Q})$ be the set of all couplings between $\mathbb{P}$ and $\mathbb{Q}$. We denote by $c_\infty(\mathbb{P}, \mathbb{Q})$ the maximal coupling between $\mathbb{P}$ and $\mathbb{Q}$, \ie, the coupling that verifies for any measurable $A$, 
\begin{align*}
    \mathbb{P}_{(X,Y) \sim c_\infty(\mathbb{P}, \mathbb{Q})} [ X \in A] &= \mathbb{P}_{X \sim \mathbb{P}} [ X \in A], 
    \mathbb{P}_{(X,Y) \sim c_\infty(\mathbb{P}, \mathbb{Q})} [ Y \in A] = \mathbb{P}_{Y \sim \mathbb{Q}} [ Y \in A], \\
    \mathbb{P}_{(X,Y) \sim c_\infty(\mathbb{P}, \mathbb{Q})} [ X \neq Y] &= \inf_{c \in \Pi(\mathbb{P}, \mathbb{Q})} \mathbb{P}_{(X,Y) \sim c} [ X \neq Y] = \mathrm{TV}(\mathbb{P}, \mathbb{Q})\,.
\end{align*}
\end{definition}

\begin{proof}[Proof of Theorem~\ref{thm:low_bound}]
Without loss of generality, suppose that we have a 2-armed Bernoulli bandit instance $\nu = (P_1, P_2)$ with means $(\mu_1, \mu_2)$ where $\mu_1 \geq \mu_2$. Let $\pi$ be an $\epsilon$-global DP consistent policy. We also introduce \emph{two} other environments $\nu' = (P_1, P'_2)$ and $\nu'' = (P_1, P''_2)$ that only differ at the distribution of the second arm, where $\mu_2 \leq \mu_2' \leq \mu_1 \leq \mu''_2$, \ie, arm 1 is still optimal in environment $\nu'$ but is not optimal in environment $\nu''$.

% Providing a regret lower bound is equivalent to controlling the probability of the event 
% $\Omega \triangleq \{ N_2(T) \leq n_2 \}$,
% for some $n_2$ to be fine-tuned later that may depend on the horizon $T$. 

% To control this probability, we introduce two other environments $\nu' = (P_1, P'_2)$ and $\nu'' = (P_1, P''_2)$ that only differ at the distribution of the second arm, where $\mu_2 \leq \mu_2' \leq \mu_1 \leq \mu''_2$, \ie, arm 1 is still optimal in environment $\nu'$ but is not optimal in environment $\nu''$.

The main idea is to control the probability of the event $\Omega  \triangleq \{ N_2(T) \leq n_2 \}$ in an augmented coupled history space, for some $n_2$ to be fine-tuned later (that may depend on the horizon $T$). 

\noindent\underline{Step 1: Building the coupled bandit environment $\gamma$.} We build a coupled bandit environment $\gamma$ of $\nu$ and $\nu'$. The policy $\pi$ interacts with the coupled environment $\gamma$ up to a given time horizon $T$ to produce an augmented history $\lbrace (a_t, r_t, r'_t)\rbrace_{t=1}^{T}$. The steps of this interaction process are: (a) The probability of choosing an action $a_t = a$ at time $t$ is dictated only by the policy $\pi_t$ and $a_1, r_1, a_2, r_2, \dots, a_{t-1}, r_{t-1}$, \ie, the policy ignores $\{r'_s\}_{s = 1}^{t - 1}$. (b) The distribution of pair of rewards $(r_t, r'_t)$ is $c_{a_t} \triangleq c_\infty(P_{a_t}, P'_{a_t})$ the maximal coupling of $(P_{a_t}, P'_{a_t})$ and is conditionally independent of the previous observed history $\lbrace (a_s, r_s, r'_s)\rbrace_{t=1}^{t - 1}$.
% \begin{enumerate}
% \item[1.] The probability of choosing an action $a_t = a$ at time $t$ is dictated only by the policy $\pi_t$ and $a_1, r_1, a_2, r_2, \dots, a_{t-1}, r_{t-1}$, \ie, the policy ignores $\{r'_s\}_{s = 1}^{t - 1}$.
% \item[2.] The distribution of pair of rewards $(r_t, r'_t)$ is $c_{a_t} \triangleq c_\infty(P_{a_t}, P'_{a_t})$ and is conditionally independent of the previous observed history $\lbrace (a_s, r_s, r'_s)\rbrace_{t=1}^{t - 1}$.
% \end{enumerate}
The distribution of the augmented history induced by the interaction of $\pi$ and the coupled environment can be defined as $p_{\gamma \pi}(a_1 , r_1 , r'_1 \dots , a_T , r_T, r'_T ) \triangleq \prod_{t=1}^T \pi_t(a_t \mid a_1 , r_1 , \dots , a_{t-1} , r_{t-1} ) c_{a_t} (r_t, r'_t)$.
% \begin{align*}
%      p_{\gamma \pi}(a_1 , r_1 , r'_1 \dots , a_T , r_T, r'_T ) \triangleq \prod_{t=1}^T \pi_t(a_t \mid a_1 , r_1 , \dots , a_{t-1} , r_{t-1} ) c_{a_t} (r_t, r'_t)~.
% \end{align*}

Again, we introduce the notation $\textbf{a} \defn (a_1, \dots, a_T)$, $\textbf{r} \defn (r_1, \dots, r_T)$, and $\textbf{r'} \defn (r'_1, \dots, r'_T)$.

\noindent\underline{Step 2: Probability decomposition.} We introduce 
% \begin{align}
  $ L \triangleq \{ \mathrm{dham}(\textbf{r}, \textbf{r'}) \leq (1 + \alpha) n_2 \mathrm{TV}(P_2, P'_2) \}$,
% \end{align} 
and
% \begin{align}
$A \defn \left\{
\sum_{t=1}^{T} \log \frac{\dd P_{a_t}'(r'_t)}{\dd P_{a_t}''(r'_t)}
\leq (1 + \alpha)\mathrm{kl}(\mu'_2,\mu''_2)n_2
\right\}$ 
% =\left\{ \sum_{t=1}^{T} \mathds{1}(a_t = 2) \log  \frac{\dd P_{2}'(r'_t)}{\dd P_{2}''(r'_t)} \leq (1 + \alpha)\mathrm{kl}(\mu'_2,\mu''_2)n_2 \right\}$,
% \end{align}
for some $\alpha > 0$, where $\mathrm{dham}(\textbf{r}, \textbf{r'}) \triangleq \sum_{t=1}^T \mathds{1}_{r_t \neq r'_t}$. Also, here for Bernoullis, we have $\mathrm{TV}(P_2, P'_2) = \mu_2' - \mu_2$.

Event $L$ will be used to do a change of measure from environment $\nu$ to $\nu'$ using the group privacy property of $\pi$, then event $A$ will be used to do a classic ``Lai-Robbins" change of measure using the KL from environment $\nu'$ to $\nu''$.

First, we start with the decomposition
\begin{align}
    \mathbb{P}_{\nu \pi}(N_2(T) \leq n_2 ) =   \mathbb{P}_{\gamma \pi} (\Omega \cap L \cap A) + \mathbb{P}_{\gamma \pi} (\Omega \cap L \cap A^c) + \mathbb{P}_{\gamma \pi} (\Omega \cap L^c)~.
\end{align}

\noindent\underline{Step 3: Controlling each probability.}
Using Lemma~\ref{lem:doub_env}, which formalises the ``double" change of environment idea, we get
\begin{align}
    \mathbb{P}_{\gamma \pi} \left(\Omega \cap L \cap A\right) \leq e^{(1 + \alpha) n_2 \left(\epsilon \mathrm{TV}(P_2, P'_2) + \mathrm{kl}(\mu_2',\mu''_2) \right) }  \frac{O(T^a)}{T - n_2},
\end{align}
for any $a>0$. Using Lemma~\ref{lem:p_omega_L_A_c} and Lemma~\ref{lem:p_omega_L_c}, we control the probabilities $\mathbb{P}_{\gamma \pi} (\Omega \cap L \cap A^c) = o_T(1) \quad \text{and} \quad \mathbb{P}_{\gamma \pi} \left(\Omega \cap L^c \right) = o_T(1),$ 
for any choice of $n_2 = n_2(T)$ as a function of $T$ such that $n_2(T) \rightarrow \infty$ when $T \rightarrow \infty$.

\noindent\underline{Step 4: Putting everything together and choosing $n_2$.}
% To sum up, we have 
% \begin{align*}
%     \mathbb{P}_{\nu \pi}(N_2(T) \leq n_2 ) =   \mathbb{P}_{\gamma \pi} (\Omega \cap L \cap A) + \mathbb{P}_{\gamma \pi} (\Omega \cap L \cap A^c) + \mathbb{P}_{\gamma \pi} (\Omega \cap L^c)\\
%     \leq e^{(1 + \alpha) n_2 \left(\epsilon \mathrm{TV}(P_2, P'_2) + \mathrm{kl}(\mu_2',\mu''_2) \right) }  \frac{O(T^a)}{T - n_2} + o_T(1) + o_T(1),
% \end{align*}
% for any $\alpha, a>0$ and $n_2 = n_2(T)$ such that $n_2(T) \rightarrow \infty$ when $T \rightarrow \infty$.
First, we chose $n_2 = \frac{(1 - \alpha) \log(T)}{\epsilon \mathrm{TV}(P_2, P'_2) + \mathrm{kl}(\mu_2',\mu''_2)},$ and $a = \frac{\alpha^2}{2}$, to get $\exp\left((1 + \alpha) n_2 \left(\epsilon \mathrm{TV}(P_2, P'_2) + \mathrm{kl}(\mu_2',\mu''_2) \right) \right) \frac{O(T^a)}{T - n_2} = o_T(1)$.
% \begin{align*}
%     e^{(1 + \alpha) n_2 \left(\epsilon \mathrm{TV}(P_2, P'_2) + \mathrm{kl}(\mu_2',\mu''_2) \right) }  \times \frac{O(T^a)}{T - n_2} =\frac{T^{1 - \alpha^2}}{T - n_2} O(T^{\alpha^2/2}) = O(T^{- \alpha^2/2}) = o_T(1)~.
% \end{align*}

With this choice of $n_2$, we have now that $\mathbb{P}_{\nu \pi}(N_2(T) \leq n_2) = o_T(1)$, 
and thus, using Markov inequality, we get, for any $\alpha>0$, and all $\mu_2 \leq \mu_2' \leq \mu_1 \leq \mu''_2$.
\begin{align*}
    \mathbb{E}_{\nu \pi}\left [N_2(T)\right] \geq n_2 \mathbb{P}_{\nu \pi} \left( N_2(T) > n_2 \right) = \frac{ (1 - \alpha) \log(T)}{\epsilon \mathrm{TV}(P_2, P'_2) + \mathrm{kl}(\mu_2',\mu''_2)} (1 - o(1))~.
\end{align*}
Finally, taking $\alpha \rightarrow 0$, and the supremum over all $\mu'_2 \in [\mu_2, \mu_1]$ and $\mu''_2 \rightarrow \mu_1$, we get the result.
% Thus, for any $\mu_2 \leq \mu_2' \leq \mu_1 \leq \mu''_2$, by taking $\alpha \rightarrow 0$, we get
% \begin{align*}
%     \liminf _{T \rightarrow \infty}\  \frac{\mathbb{E}_{\nu, \pi}\left(N_2(T)\right)}{\log (T)} \geq  \frac{1}{\epsilon (\mu_2' - \mu_2) + \mathrm{kl}(\mu_2',\mu''_2)}~.
% \end{align*}
% Finally, by taking the supremum over all $\mu'_2 \in [\mu_2, \mu_1]$ and $\mu''_2 \rightarrow \mu_1$, we get the result.
\end{proof}

% Decompose the proof into lemmas that control each proba in the appendix

% add some comments on the double change of environment idea
% Introduce maximal couplings + Group privacy +  Build a "coupled" environment + "Double change" of environments idea

\paragraph{Key Changes in Proof.} The proof improves the lower bound of~\cite{azize2022privacy} by doing a ``double" change of environment. (a) The first change of environment uses the privacy property of the policy, and thus the TV transport. (b) The second change uses the classic ``Lai-Robbins" change of measure and thus the KL transport. By optimising for the ``in-between" environment, the double change always has smaller transport than any route led by purely KL or TV transport.
\section{Algorithm Design and Regret Analysis}\label{sec:alg_des}

In this section, we propose two algorithms, \dpklucb{} and \dpimed{}, presented in Algorithm \ref{alg_DP}. At the core of our algorithm design lies a new concentration bound for $\epsilon$-DP means of Bernoulli variables (Proposition~\ref{prop:conc}).  We analyse both the privacy and regret of our proposed algorithms, and show that their regret upper bound matches the lower bound up to a constant arbitrary close to 1.
% they match the lower bound up to a constant $\alpha > 1$, where $\alpha$ is a parameter that controls the geometrically increasing

First, we start with the concentration inequality for the private mean of IID Bernoullis.

\begin{proposition}[Concentration Bound of Private Mean]\label{prop:conc}
For $\mu\in(0,1)$ and $\epsilon>0$, let $
\tilde{S}_{n,m}=
\sum_{i=1}^{n} X_{i}+
\sum_{j=1}^m Y_{j}$, where 
$X_i \sim \mathrm{Ber}(\mu)$ and 
$Y_{j}\sim \mathrm{Lap}(1/\epsilon)$,
be the sum of $n$ independent Bernoulli random variables with mean $\mu$
and $m$ independent Laplace variables with scale $1/\epsilon$.
Let $x\in[0,1]$ and $\{n_m\}_{m\in\mathbb{N}}$ be a sequence such that $m/n_m=o(1)$.
Then, for any $a>0$ there exists a constant $\Aa>0$ such that for all $m\in \mathbb{N}$, %$\Pr\left[ \frac{\tilde{S}_{n_m,m}}{n_m} \le x \right] \le A_a \e^{-n_m  (\mathrm{d}_\epsilon(x, \mu)-a)}$ when $x \le \mu$. Also, $\Pr\left[ \frac{\tilde{S}_{n_m,m}}{n_m} \ge x \right] \le A_a \e^{-n_m  (\mathrm{d}_\epsilon(x, \mu)-a)}$ when $x\ge \mu$.
{\small\begin{align*}%\label{concentration_inequality}
\Pr\left[
\frac{\tilde{S}_{n_m,m}}{n_m}
\le x
\right]
&\le
A_a
\e^{-n_m  (\mathrm{d}_\epsilon(x, \mu)-a)},\text{ for } x\le \mu;
~~
\Pr\left[
\frac{\tilde{S}_{n_m,m}}{n_m}
\ge x
\right]
\le
A_a
\e^{-n_m  (\mathrm{d}_\epsilon(x, \mu)-a)},\text{ for } x\ge \mu~.
\end{align*}}
We recall that $\mathrm{d}_\epsilon(x, y) \defn \inf_{z\in [x\wedge y,x\vee y]} \mathrm{kl}(z,y)+\epsilon|z-x|$.
\end{proposition}

\paragraph{Discussions.} (a) This concentration bound can be seen as a private version of the Chernoff bound (Lemma~\ref{lem:ChernoffKL}), where $\mathrm{d}_\epsilon$ replaces the $\mathrm{kl}$ in the exponent. 
(b) As soon as the number of summed Laplace noises $m$ is negligible with respect to  the number of summed Bernoulli variables $n$, then the effect of $m$ on the dominant term is similar to when $m= 1$. 
% Also, the difference between $m=0$ and $m \geq 1$ is whether kl or $\mathrm{d}_\eps$ are used in the exponent.
% When $m= 0$, $\mathrm{d}_\eps$ should be changed to kl in the exponent. 
(c) This concentration bound is a tighter version of Lemma 4 in~\citep{azize2022privacy} with $m =1$. Lemma 4 of~\cite{azize2022privacy} and other works in bandits under DP~\citep{Mishra2015NearlyOD, dpseOrSheffet, lazyUcb, hu2022near} deal with the concentration of the noise and random variables separately-- they use an inequality $\mathrm{Pr}(X + Y \geq a) \leq \mathrm{Pr}(X \geq a) + \mathrm{Pr}(Y \geq 0)$, followed by a classic non-private concentration bound for the first term and concentration bound of Laplace noise for the second term. We improve this loose analysis by a coupled treatment of noise and variables.

% Comparison to~\cite{azize2022privacy, hu2022near}: 

\paragraph{Proof Sketch.} Proposition~\ref{prop:conc} is a corollary of the general Lemma~\ref{lem:PrivateSumUpper} that holds for any $n$ and $m$. To prove Lemma~\ref{lem:PrivateSumUpper}, we express $\Pr\left[\tilde{S}_{n,m} \ge x \right]$ in the form of a convolution of the sums of Bernoulli rewards and Laplace noises.
Even though we still resort to the Chernoff bound for each of the sums, considering the convolution of sums significantly improves the bound compared with the na\"ive use of the Chernoff bounds for noise and variables in $\tilde{S}_{n,m}$. The complete proof is in Appendix~\ref{app:conc}.

% To prove Lemma~\ref{lem:PrivateSumUpper}, we consider $\Pr\left[
% \tilde{S}_{n,m} \ge x
% \right]$ as a convolution. We decompose the convolution into two terms, which we upper bound more tightly. The complete proof is presented in Appendix~\ref{app:conc}.
% \setlength{\textfloatsep}{10pt}
\LinesNumbered
\SetAlgoVlined
\begin{algorithm}[t]
   % \begin{algorithmic}
%\begin{algorithm}[t]
%ruled,vlined,linesnumbered
\SetAlgoLined
%\DontPrintSemicolon
\SetKwInput{Input}{Input}
\SetKwInput{Parameter}{Parameter}
   
   \KwIn{$\epsilon$: privacy parameter, $K$: number of arms, $T$: horizon, $\{B_{m}\}_{m=0}^{\infty}$: batch sizes}
%   \BlankLine
%\Parameter{$M \in \mathbb{N}\cup \{\infty\}$.}
Pull each arm $B_0$ times and receive rewards $\{\{X_{i,n}\}_{n=1}^{B_0} \}_{i = 1}^K$\; 

Compute private reward sum $\tilde{S}_{i,0}=
\sum_{n=1}^{B_0}X_{i,n}+Y_{i,0}$ for $Y_{i,0} \sim \mathrm{Lap}(1/\epsilon)$\;

Compute private mean $\tilde{\mu}_{i,0}=\tilde{S}_{i,0}/B_0$\;

Set arm-dependent epoch $m_i:=0$ for each arm $i\in[K]$\;

Set cumulative pull number $n_{m_i}:=B_0$ for each arm $i\in[K]$\;

Set $t \gets K B_0 + 1$\;

 \While{
 	$t\le T$
 }{
%\If{there is an arm with $\tilde{\mu}_{i,m_i}\le -1$}{Set $i_t$ to such an arm. (Maybe this procedure is not necessary, but I think it makes the analysis a bit easier.)}
%\Else{
\label{line_selection1}	(\dpklucb{}): compute $i(t)\in \argmax_i \ \bar{\mu}_{i}(t)$ maximising the DP-KLUCB index given by
\begin{align}
\bar{\mu}_{i}(t)=\max\left\{\mu:
\mathrm{d}_\epsilon\left([\tilde{\mu}_{i,m_i}]_{0}^1, \mu\right)\le \frac{\log t}{n_{m_i}}\right\}
\label{ucb_score}
\end{align}%
(\dpimed{}): compute $i(t)\in \argmin_i \ I_i(t)$ minimising the DP-IMED index given by
\begin{align}
I_i(t)=n_{m_i} \mathrm{d}_\epsilon\left([\tilde{\mu}_{i,m_i}]_0^1, [\tilde{\mu}^*(t)]_0^1\right)
+\log n_{m_i},
\label{imed_score}
\end{align}
where $\tilde{\mu}^*(t)=\max_j \tilde{\mu}_{j,m_j}$ and $[x]_0^1=\max\{0, \min\{x,1\}\}$ is the clipping of $x$ onto $[0,1]$\;
%}
\label{line_compute_index}	

Pull arm $i(t)$ for $B_{m_{i(t)} + 1}$ times and receive rewards $\{X_{i(t),n}\}_{n=n_{m_{i(t)}}+1}^{n_{m_{i(t)} + B_{m_{i(t)} + 1}}}$\;

Update the noisy sum $\tilde{S}_{i(t),m_{i(t)}+1} \gets
\tilde{S}_{i(t),m_{i(t)}}
+\sum_{n=n_{m_{i(t)}} + 1}^{n_{m_{i(t)}+ B_{m_{i(t)} + 1}}} X_{{i(t)},n}+ Y_{i(t),m_{i(t)} + 1}$ 
 where $Y_{i(t),m_{i(t)}+1} \sim \mathrm{Lap}({1}/{\epsilon})$\;

Compute private mean $\tilde{\mu}_{i(t),m_{i(t)}+1}=\tilde{S}_{i(t),m_{i(t)}+1}/n_{m_{i(t)}+1}$\;

%$m_{i(t)}:=m_{i(t)}+1$, $n_{m_{i(t)}}:=n_{m_{i(t)}-1}+n_0 \base^{m_{i(t)}}$.\;
Update $m_{i(t)} \gets m_{i(t)}+1$,\, $n_{m_{i(t)}} \gets n_{m_{i(t)}}+B_{m_{i(t)}}$, $t \gets t+B_{m_i(t)}$\;
\label{line_selection2}

%$t:=t+n_0 \base^{m_{i(t)}}$.\;

 }
\caption{\dpklucb{} and \dpimed{}}\label{alg_DP}
%\end{algorithm}%
\end{algorithm}

% To prove Lemma~\ref{lem:PrivateSumUpper}, we express $\Pr\left[\tilde{S}_{n,m} \ge x \right]$ in a form of a convolution of the sums of Bernoulli rewards and Laplace noises.
% Even though we still resort to the Chernoff bound for each of the sums, separately considering them significantly improves the bound combined with the naive use of the Chernoff bound for $\tilde{S}_{n,m}$ itself. The complete proof is presented in Appendix~\ref{app:conc}.

%  Proof sketch needs improvement
\paragraph{Algorithm Design.} Based on Proposition~\ref{prop:conc}, we propose \dpklucb{} and \dpimed{} in Algorithm~\ref{alg_DP}. Both algorithms run in arm-dependent phases (Line 9 in Algorithm~\ref{alg_DP}), and add Laplace noise to achieve $\epsilon$-global DP (Line 10 in Algorithm~\ref{alg_DP}). This is similar to the algorithm design in~\cite{dpseOrSheffet, azize2022privacy, hu2022near}, with two modifications.

(a) \textit{Our algorithms do not forget rewards from previous phases.} In contrast, the algorithms in \citep{dpseOrSheffet, azize2022privacy, hu2022near} run in adaptive and ``non-overlapping" phases. The sums of rewards are computed over non-overlapping sequences, which means that rewards collected in the first phases are ``thrown away" in future phases. By running
non-overlapping phases, these algorithms avoid the use of sequential composition (Proposition~\ref{prp:compo}), and use instead the ``parallel composition" property (Lemma~\ref{lem:paral_compo})
of DP to add less noise. Specifically, if the rewards are in $[0,1]$, with forgetting, it is enough to add one $\mathrm{Lap}\left( 1/\epsilon\right)$ to each sum of rewards to make the simultaneous release of all the partial sums achieving DP. In our algorithms, we do \emph{not} forget previous private sums (Line 10 in Algorithm~\ref{alg_DP}). The price of not forgetting is adding multiple Laplace noises with scale $1/\epsilon$ to the non-private sum. Here, we use the insights from the concentration inequality of Proposition~\ref{prop:conc}, \ie, as long as the number of added Laplace noises is negligible with respect to the number of added Bernoulli variables, the effect of the added noise on the dominant term is similar to having one Laplace noise. This refined analysis allows us to completely remove forgetting.

(b) Our algorithms use new indexes, \ie 
~Eq.~\eqref{ucb_score} and Eq.~\eqref{imed_score}, inspired by Proposition~\ref{prop:conc}, and are based on the $\mathrm{d}_\eps$ quantity appearing in the lower bound. In addition, the index of \dpklucb{} is instantiated with an exploration bonus of $\log(t)/n_{m_i}$. This contrasts AdaP-KLUCB and Lazy-DP-TS, which need an exploration bonus of roughly $3\log(t)/n_{m_i}$ for their regret analysis.

% This bound suggests two important changes compared to the DP bandit algorithms in the literature~\cite{dpseOrSheffet, azize2022privacy, hu2022near}: (a) a new index based on the quantity $\mathrm{d}_\eps$ that appears in the lower bound (Theorem~\ref{thm:low_bound}), and (b) that reward forgetting, an important design choice in previous DP bandit algorithms, is not necessary to match the lower bound.

% Add transition to why "no forgetting", then move to privacy analysis

Now, we present the privacy guarantee of our algorithms.
\begin{proposition}[Privacy analysis]\label{prop:priv}
    \dpklucb{} and \dpimed{} are $\epsilon$-global DP for rewards in $[0,1]$.
\end{proposition}

\paragraph{Proof Sketch.} First, given a sequence of rewards $\{r_1, \dots, r_T\} \in [0,1]^T$ and some time steps $1 = t_1< t_2 < \cdots < t_\ell = T + 1$, releasing the partial sums $\left\{ \left(\sum_{s = t_k}^{t_{k+1} - 1} r_s \right) + Y_k \right \}_{k = 1}^{\ell - 1}$ is $\epsilon$-DP, where $Y_k \sim \mathrm{Lap}(1/\epsilon)$. This is the main privacy lemma used to design algorithms of~\citet{dpseOrSheffet, azize2022privacy, hu2022near}. Now, by the post-processing property of DP, we also have that releasing the sums $\left\{ \left(\sum_{s = 1}^{t_{k+1} - 1} r_s \right) + \sum_{p = 1}^k Y_p \right \}_{k = 1}^{\ell - 1}$ is $\epsilon$-DP, by summing the outputs of the previous DP mechanism. Finally, \dpimed{} and \dpklucb{} are $\epsilon$-global DP by adaptive post-processing of the sum of rewards. The detailed proof is presented in Appendix~\ref{app:priv}.

% This inequality suggests new indexes used to design \dpimed{} and \dpklucb{}, both based on the $\mathrm{d}_\eps$ information-theoretic quantity appearing in the lower bound of Theorem~\ref{thm:low_bound}. Also, this concentration inequality suggests that reward forgetting, an important design choice in all previous DP bandit algorithms

% In this section, we first propose the algorithms named \dpklucb{} or \dpimed{} in Algorithm \ref{alg_DP}. Based on these algorithms, we derive regret upper bound matching the lower bound in the previous section except for a factor of constant $\alpha>1$.

%  The key idea of the Algorithms \ref{alg_DP} is ...
% \yulian{Note that we need the definition of double sides, i.e. $\mu>\mu'$ or $\mu<\mu'$}
% \achraf{added to the proposition}
% We define
% \begin{equation}\label{def:NewDiverg}
%     \mathrm{d}_\epsilon(\mu, \mu')=\inf_{\mu''\in [\mu\wedge \mu',\mu\vee\mu']} \mathrm{kl}(\mu'',\mu')+\epsilon|\mu''-\mu|~.
% \end{equation}

% \achraf{I suggest to rename this $\mathrm{d}_\epsilon(\mu, \mu')$  as in the lower bound (or something else). Calling it KL is misleading, since it interpolates both the kl and tv.}

% \jhonda{I agree that KL is misleading. In any case it would be better to use newcommand for $d_{\epsilon}$ so that it can be changed later.}

To have a ``good" regret bound, Proposition~\ref{prop:conc} suggests using a batching strategy where the number of batches is sublinear in $T$. For simplicity, we chose the batch sizes $B_m$ in Algorithm~\ref{alg_DP} such that $B_m\approx n_0 \base^m$, \ie, a geometric sequence with initialisation $n_0 \in\mathbb{N}$ and ratio $\alpha > 1$. More formally, we choose 
\begin{align}
B_m
&=
\left\lceil n_0 \frac{\base^{m+1}-1}{\base-1}\right\rceil
-
\left\lceil n_0 \frac{\base^{m}-1}{\base-1}\right\rceil\,,
\label{def_batch}
\end{align}
where $\left \lceil x \right\rceil$ is the smallest integer no less than $x$. When $\alpha$ is an integer, $B_m = n_0 \alpha^m$.

% Before starting the regret, we chose In the regret analysis we take the batch sizes $B_m\approx n_0 \base^m$
% for some $n_0\in\mathbb{N}$ and $\base>1$.
% To be more formal,
% we take
% $\{B_m\}$ so that
% \begin{align}
% n_m=\sum_{i=0}^m B_i
% =\left\lceil n_0 \frac{\base^{m+1}-1}{\base-1}\right\rceil,
% \n
% \end{align}
% that is, we set
% \begin{align}
% B_m=n_{m}-n_{m-1}
% &=
% \left\lceil n_0 \frac{\base^{m+1}-1}{\base-1}\right\rceil
% -
% \left\lceil n_0 \frac{\base^{m}-1}{\base-1}\right\rceil.
% \label{def_batch}
% \end{align}
% When $\alpha>1$ is an integer it simply becomes $B_m=n_0 \alpha^m$.
% Given this choice of batches, we now derive the regret upper bounds for \dpimed{} and \dpklucb{}.
\begin{theorem}[Regret upper bound of \dpimed{} and \dpklucb{}]\label{thm:upp_bound}
Assume $\mustar<1$.
Under the batch sizes given in \eqref{def_batch} with $\base>1$, and for any Bernoulli bandit $\nu$, we have
{\small\begin{align*}
    \reg_T(\dpimed{}, \nu) &\le \sum_{i \neq i^*} \frac{\base \Delta_i\log T}{\mathrm{d}_\epsilon(\mu_i, \mustar)} +o(\log T),~~
    \reg_T(\dpklucb{}, \nu) \le \sum_{i \neq i^*} \frac{\base \Delta_i\log T}{\mathrm{d}_\epsilon(\mu_i, \mustar)} +o(\log T)~.
\end{align*}}
\end{theorem}

% \begin{theorem}[Regret upper bound of \dpimed{} and \dp]\label{thm:dp_imed}
% Assume $\mustar<1$.
% Under the batch sizes given in \eqref{def_batch} with $\base>1$
% the regret bound of \dpimed{} for a Bernoulli bandit $\nu$ is 
% \label{thm:regretUpper}
%     \[ \reg_T(\dpimed{}, \nu) \le \sum_{i \neq i^*} \frac{\base \Delta_i\log T}{\mathrm{d}_\epsilon(\mu_i, \mustar)} +o(\log T).\]
% \end{theorem}

% \begin{theorem}[Regret upper bound of \dpklucb{}]
% \label{thm:regretUpper_klucb}
% Assume $\mustar<1$.
% Under the batch sizes given in \eqref{def_batch} with $\base>1$
% the regret bound of \dpklucb{} for a Bernoulli bandit $\nu$ is 
%     \[\reg_T(\dpklucb{}, \nu)\le \sum_{i \neq i^*} \frac{\base \Delta_i\log T}{\mathrm{d}_\epsilon(\mu_i, \mustar)} +o(\log T).\]
% \end{theorem}

\noindent\textbf{Comments.} (a) The regret upper bounds of \dpimed{} and \dpklucb{} match asymptotically the lower bound of Theorem~\ref{thm:low_bound} up to the constant $\alpha > 1$, where $\alpha$ is the ratio of the georemetrically increasing batch sizes $B_m$. This parameter $\alpha > 1$ can be set arbitrarily close to 1 to match the dominant term in the asymptotic regret lower bound.
%In addition, our analysis only requires that the number of batches is sublinear in $T$ as can be seen from Proposition~\ref{prop:conc}. As a result, we can also use a polynomially increasing batch size instead of $B_m\approx \alpha^m$, which fully makes the regret asymptotically optimal. We used a geometrically increasing batch size here just for simplicity.
(b) Our algorithms strictly improve over the regret upper bounds of~\cite{azize2022privacy, hu2022near}. Also, our upper bounds are the first to show a dependence in the tighter quantity $\mathrm{d}_\eps$, compared to having $\min\{\Delta_a^2, \epsilon \Delta_a \}$ in the regrets for~\cite{azize2022privacy, hu2022near}. We provide additional comments that compare our regret upper bound to AdaP-KLUCB in Appendix~\ref{app:ub_proof}.

% (a) Asymptotically matching the lower bound up to $\alpha > 1$

% (b) comparison to the bound of~\cite{azize2022privacy}

\noindent\textbf{Proof Sketch.} The proof uses similar steps as~\cite{honda2015non} for the IMED algorithm and the reduction technique for the KL-UCB algorithm by~\cite{honda2019note} with the new concentration inequality involving $\mathrm{d}_\eps$ (Proposition~\ref{prop:conc}). The main technical challenge is dealing with the adaptive batching strategy when the optimal arm has not yet converged well.
%, and using the reduction technique for the KL-UCB algorithm in~\cite{honda2019note} with the new concentration inequality using $\mathrm{d}_\eps$.
While it was sufficient to count the number of such rounds in \citep{honda2015non}, a suboptimal arm $i$ might be pulled $B_{m_i}\approx n_0 \alpha^{m_i}$ times under our batched algorithm once such an event occurs.
We control this effect by a regret decomposition that is tailored for batched pulls of arms while the property of IMED/KL-UCB index can still be naturally incorporated. 
The full proof is presented in Appendix~\ref{app:ub_proof}.

\begin{comment}
\paragraph{Proof sketch} The proof uses steps similar to those for the IMED algorithm \citep{honda2015non}
and the reduction technique for the KL-UCB algorithm in~\cite{honda2019note} with the new concentration inequality using $\mathrm{d}_\eps$.
One technical challenge is how to deal with the adaptive batching strategy when the optimal arm has not yet converged well.
While it was sufficient to simply count the number of such rounds in \cite{honda2015non}, suboptimal arms $i$ might be pulled $B_{m_i}\approx n_0 \alpha^{m_i}$ times under our batched algorithm once such an event occurred.
We control this effect by a regret decompsition that is tailored for batched pulls of arms while the property of IMED/KL-UCB index can still be naturally incorporated.
The full proof is presented in Appendix~\ref{app:ub_proof}.
\end{comment}
%  This proof sketch may 

% main steps of the proof, some comments about the changes compared to classic proof: adaptive batching, new index maybe?

\section{Experimental Analysis}
In this section, we numerically compare the performance of our algorithms, \ie, \dpklucb{} and \dpimed{}, to $\epsilon$-global DP algorithms from the literature: \dpse{}~\citep{dpseOrSheffet}, AdaP-KLUCB~\citep{azize2022privacy} and Lazy-DP-TS~\citep{hu2022near}. As a non-private benchmark, we include the IMED algorithm~\citep{honda2015non}. Since both AdaP-KLUCB and Lazy-DP-TS explore each arm once, and use arm-dependent \emph{doubling}, we chose $n_0 = 1$ and $\alpha = 2$ for \dpklucb{} and \dpimed. Also, to comply with the regret analysis in~\citep{azize2022privacy, dpseOrSheffet}, we chose $\alpha = 3.1$ in AdaP-KLUCB, and $\beta = 1/T$ in \dpse.

As in~\citet{dpseOrSheffet, azize2022privacy, hu2022near}, we consider $4$ different $5$-arm Bernoulli environments, with specific arm-means choices. We run each algorithm $20$ times for $T= 10^6$. For $\epsilon = 0.25$, we plot the mean regret in Figure~\ref{fig:experiments_global} for $\mu_1 \defn [0.75, 0.7, 0.7, 0.7, 0.7] $ in the left and $\mu_2 \defn [0.75, 0.625, 0.5, 0.375, 0.25]$ in the right. In Appendix~\ref{app:extended_exps}, we present additional results for the other environments under different budgets.

\begin{figure}[t!]
    \centering
    \includegraphics[width=0.45\linewidth]{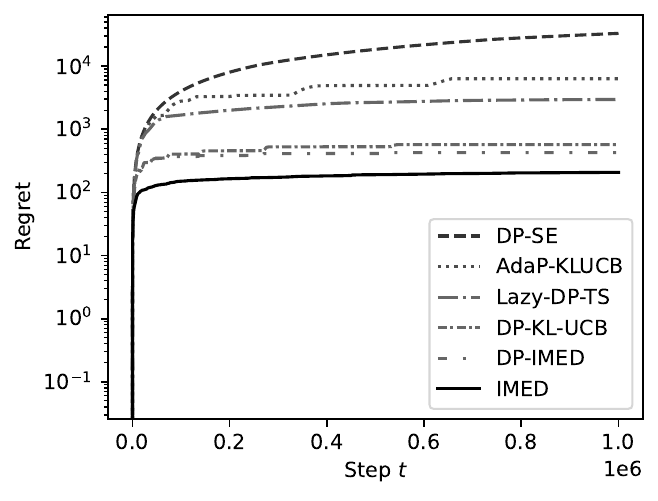}\hspace*{2em}
    \includegraphics[width=0.45\linewidth]{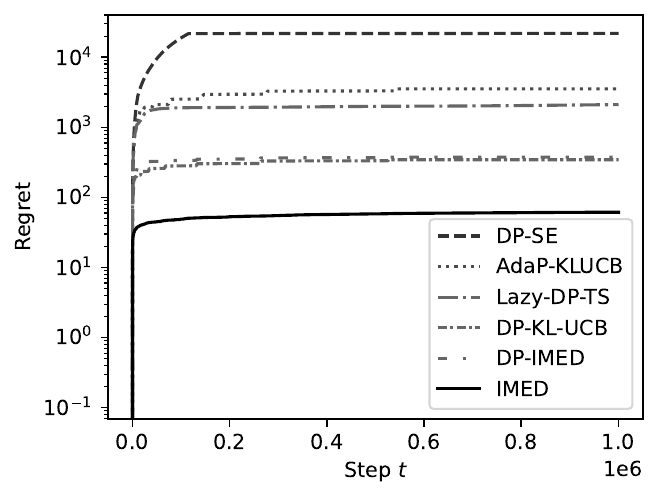}
    \caption{Evolution of the regret over time for \dpse, AdaP-KLUCB, Lazy-DP-TS, \dpklucb{}, and \dpimed{} for $\epsilon = 0.25$, and Bernoulli bandits $\mu_1$ (left) and $\mu_2$ (right).}
 %of \adaptt{}, \adapttt{}, \dpse, and TTUCB.}
    \label{fig:experiments_global}~\vspace{-2em}
\end{figure}

\textbf{Results.} \textit{\dpklucb{} and \dpimed{} achieve lower regret in all Bernoulli environments and privacy budgets} under study (\textit{up to 10 times less on an average}). This is explained by the fact that \dpklucb{} and \dpimed{} do not forget half of the samples, and also thanks to the tighter indexes, where the optimism corresponds to a bound of $\frac{\log(t)}{N_a}$ than the $\frac{3 \log(t)}{N_a}$ needed by~\cite{azize2022privacy, hu2022near}.~\vspace*{-.5em}

% \dbcomment{Can we have one set of experiments on regret of \dpklucb{} and \dpimed{} changing with $\epsilon$? It is interesting to check the regime switching behaviour of new algos wrt AdaP-KLUCB.} Working on it, see comment bellow!

% Also working on the two regimes figures (reg with respect to epsilon figure), takes more time since there I will take $T = 10^7$ to compare to the asymptotic lower bound, and also should be run for more epsilons to have the full figure 

\section{Discussions and Future Works}
We improve both regret lower bound (Theorem~\ref{thm:low_bound}) and upper bounds (Theorem~\ref{thm:upp_bound}) for Bernoulli bandits under $\epsilon$-global DP. We introduce a new information-theoretic quantity $\mathrm{d}_\eps$ (Equation~\eqref{eq:d_eps}) that tightly characterises the hardness of minimising regret under DP, and smoothly interpolates between the KL and the TV. Our proposed algorithms share ingredients with algorithms from the literature while alleviating the need to forget rewards as a design technique. This is thanks to a new tighter concentration inequality for private means of Bernoullis (Proposition~\ref{prop:conc}). Our results solve the open problem of having matching upper and lower bound up to the same constant posed by~\cite{azize2022privacy} and refute that forgetting is necessary for designing optimal DP bandit algorithms. 
An interesting future work would be to generalise our concentration inequality and, in turn, the regret upper bounds to general distribution families (e.g. sub-Gaussians, exponential families).

% the lower bound is already directly generalisable to other parametric distributions beyond bernoullis, see comment (c) after Theoerem 5

% Maybe more interesting directions would be the $(\epsilon, \delta)$-dp case, or adversarial bandits

% Acknowledgments---Will not appear in anonymized version
% \acks{We thank a bunch of people and funding agency.}

\bibliography{references}

\newpage
\appendix

% \crefalias{section}{appendix} % uncomment if you are using cleverer

\section{Outline}\label{app:outline}

The appendices are organised as follows:
\begin{itemize}
    \item In Appendix~\ref{app:adap_cont}, we extend the adaptive continual release model of~\cite{jain2023price} to bandits, and link it to $\epsilon$-global DP
    \item In Appendix~\ref{app:lb_proof}, we provide the proof of the three lemmas used to prove the regret lower bound of Theorem~\ref{thm:low_bound}
    \item In Appendix~\ref{app:conc}, we provide the complete proof of the concentration inequality of Proposition~\ref{prop:conc}
    \item  In Appendix~\ref{app:priv}, we provide the complete proof of the privacy guarantee of Proposition~\ref{prop:priv}
    \item In Appendix~\ref{app:ub_proof}, we provide the complete proof of the regret upper bounds of Theorem~\ref{thm:upp_bound}
    \item In Appendix~\ref{app:extended_exps}, we provide additional experimental results
    \item In Appendix~\ref{app:lemms}, we recall useful lemmas used throughout the paper
\end{itemize}

\section{Adaptive Continual Release Model for Bandits}
\label{app:adap_cont}

In this section, we extend the adaptive continual release model of~\cite{jain2023price} to bandits. In this model, the policy interacts with an adversary that chooses adaptively rewards based on previous outputs of the policy. 
% The difference between these two models is in the nature of the adversary: 

% (a) In Interactive DP (Definition~\ref{def:int_dp}), the adversary $B \deffn \{B_t\}_{t = 1}^{\horizon}$ is a sequence of functions that map the history of observed actions to query actions. Specifically, at step $t$, the adversary observes the history $(a_1, \dots, a_t) \in [K]^t$ of actions recommended, and comes up with a query action $q_t = B_t(a_1, \dots, a_t) \in [K]$. This query action $q_t$ is then used to generate a reward from a fixed table of rewards $\textbf{x}$; \ie $r_t = x_{t, q_t}$. The policy updates its recommendations at step $t+1$ based on $q_t$ and $r_t$.

In the following, we formalise the notion of an adaptive adversary from~\cite{jain2023price} and call it a ``reward-feeding" adversary. 

% This is in contrast to the adversary $B$ of Interactive DP, which is a "query-action" feeding adversary.

\begin{definition}[Reward-Feeding Adversary]
A reward-feeding adversary $\mathcal{A}$ is a sequence of functions $(\mathcal{A}_t)_{t = 1}^T$ such that, for $ t \in \{1, \dots, T\}$, 
\[
\mathcal{A}_t: a_1, \dots, a_t \rightarrow (r_t^L, r_t^R)~.
\]
\end{definition}

A ``reward-feeding" adversary $\mathcal{A}$ is a sequence of ``reward" functions that take as input the action-history and outputs a pair of rewards $(r_t^L, r_t^R)$. The reward-feeding adversary $\mathcal{A}$ has two channels: a left ``standard" channel $L$ and a right channel $R$. These channels are used to simulate ``neighbouring" rewards. 

Precisely, to simulate ``neighbouring" rewards, the interactive protocol between the policy $\pi$ and the reward-feeding adversary $\mathcal{A}$ has two hyper-parameters: (a) a specific ``challenge" time $t^\star \in \{1, T\}$, and (b) a binary $b \in \{L, R\}$. For steps $t \neq t^\star$, the policy observes a reward coming from the adversary's left ``standard" channel, i.e. $r_t = r_t^L$. Otherwise, when $t = t^\star$, the policy observes a reward from the channel corresponding to the secret binary $b$. 

In other words, if $b = L$, the policy $\pi$ always observes a reward from the left channel. When $b = R$, the policy observes the left channel reward for all steps, except at $t^\star$ where the policy observes a right channel reward. Thus, for any sequence of actions $(a_1, \dots, a_T)$ chosen by the policy $\pi$, and for any $t^\star$, the sequence of rewards observed by $\pi$ when $b = L$ is neighbouring to the sequence of rewards observed when $b = R$. In addition, these two sequences only differ at the reward observed at the challenge time $t^\star$, and the rewards have been adaptively chosen by the adversary.

Thus, we formalise the adaptive continual release interaction as follows:

\fbox{%
\begin{minipage}{0.9\columnwidth}
Let $b \in \{L, R\}$ and $t^\star \in \{1, \dots, T\}$\\
For $t=1,\ldots, \horizon$
\begin{enumerate}
    \item The policy $\pi$ selects an action $$a_t \sim \pi_t(\cdot \mid a_1, r_{1}, \dots, a_{t - 1}, r_{t-1}), \, a_t \in [K]$$
    \item The adversary $\mathcal{A}$ selects an adaptively chosen pair of rewards:
    $$ (r_t^L, r_t^R) = \mathcal{A}_t(a_1, \dots, a_t)$$
    \begin{itemize}
        \item If $t \neq t^\star$:
        $$ r_t = r_t^L $$
        \item If $t = t^\star$:
        $$ r_{t^\star} = r_{t^\star}^b$$
    \end{itemize}
    \item The policy $\pi$ observes the reward $r_t$
\end{enumerate}
\end{minipage}%
}

~ ~ \\
When this interaction is run with parameters $t^\star$ and $b$, we represent the interaction by $\pi \overset{b, t^\star}{\Leftrightarrow} \mathcal{A}$, and illustrate it in Figure~\ref{fig:int_adv_new}. The view of the adversary $\mathcal{A}$ in the interaction $\pi \overset{b, t^\star}{\Leftrightarrow} \mathcal{A}$ is the sequence of actions chosen by the policy $\pi$, \ie, 
\[
\View_{\mathcal{A}, \pi}^{b, t^\star} \deffn \View_\mathcal{A} (\pi \overset{b, t^\star}{\Leftrightarrow} \mathcal{A}) \deffn (a_1, \dots, a_\horizon)~.
\]

\begin{figure}[t!]
    \includegraphics[width=\linewidth]{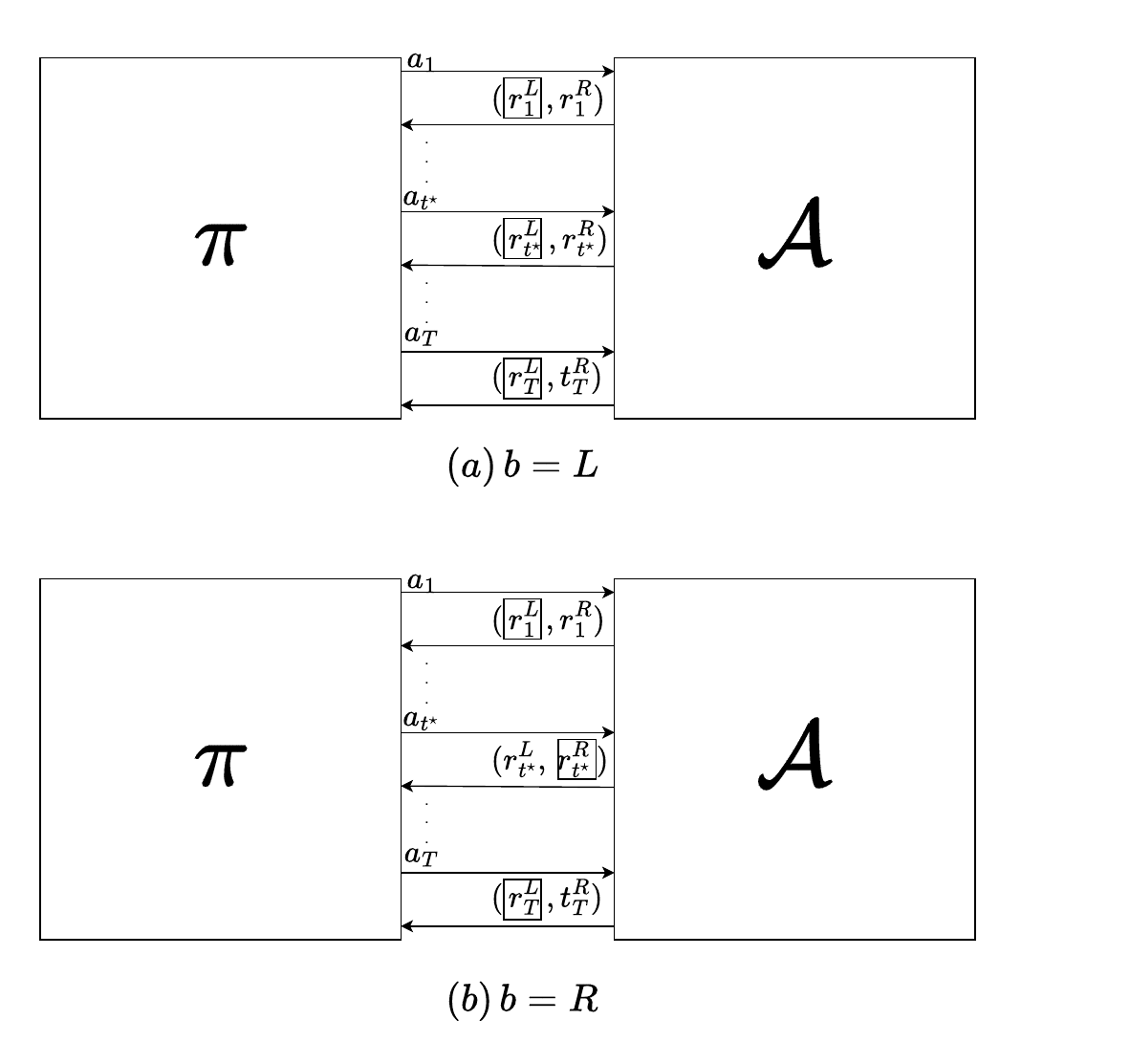}
    \caption{Interactive protocol in the adaptive continual release model between a policy $\pi$ and a reward-feeding adversary $\mathcal{A}$. The protocol in Figure (a) is run with $b=L$, while the protocol in Figure (b) is run with $b=L$. The framed part corresponds to the reward observed by the policy.}\label{fig:int_adv_new}
  \end{figure}

A policy is DP in the adaptive continual release model if the view of the adversary is indistinguishable when the interaction is run on $b = L$ and $b = R$ for any challenge step $t^\star$.

\begin{definition}[DP in the Adaptive Continual Release Model]\label{def:adap_cont_dp}{\ }
\begin{itemize}
        \item A policy $\pol $ is $(\epsilon, \delta)$-DP in the adaptive continual release model for a given $\epsilon\geq 0$ and $\delta\in [0,1)$, if for all reward-feeding adversaries $\mathcal{A}$, all subset of views $\mathcal{S} \subseteq [\arms]^\horizon$,
    \begin{align*}
 \sup_{t^\star \in \{1, \dots, T\}} \ \Pr[ \View_{\mathcal{A}, \pi}^{L, t^\star} \in \mathcal{S}] -  e^\epsilon \Pr[ \View_{\mathcal{A}, \pi}^{R, t^\star} \in \mathcal{S}] \leq \delta~.
    \end{align*}

    \item A policy $\pol $ is $\rho$-zCDP in the adaptive continual release model for a given $\rho \geq 0$, if for every $\alpha > 1$, and every reward-feeding adversary $\mathcal{A}$,   
    \begin{align*}
 \sup_{t^\star \in \{1, \dots, T\}} \ D_\alpha(\View_{\mathcal{A}, \pi}^{L, t^\star}  \| \View_{\mathcal{A}, \pi}^{R, t^\star})\leq \rho \alpha~.
    \end{align*}
\end{itemize}
\end{definition}

% The adaptive continual release model deals with the challenges of adapting DP to bandits in the following:
% \begin{itemize}
%     \item[(a)] The online nature of the bandit interaction is captured by the online nature of the reward-feeding adversary $\mathcal{A}$. This adversary provides a pair of rewards $(r_t^L, r_t^R)$ at each step $t$.
%     \item[(b)] The sequential nature of the bandit interaction is captured by the sequential nature of the reward-feeding adversary $\mathcal{A}$. This adversary provides a pair of rewards $(r_t^L, r_t^R)$  at each step $t$ that only depends on the policy's past actions $a_1, \dots, a_t$.
%     \item[(c)] Similar to View DP, the adaptive continual release model deals with partial information by considering the input to be the observed rewards. 
% \end{itemize}

\begin{remark}\label{rem:int_cont}[Expanding the View of the Reward-feeding Adversary $\mathcal{A}$] For any reward-feeding adversary $\mathcal{A}$, any policy $\pi$ and any $t^\star \in \{1, \dots, T\}$, and any $(a_1, \dots, a_T) \in [K]^T$, we have for the left view:
    \begin{align*}
        \Pr[ \View_{\mathcal{A}, \pi}^{L, t^\star} = (a_1, \dots, a_T)] &= \pi_1(a_1) \pi_2(a_2 \mid a_1, \mathcal{A}^L_1(a_1)) \dots  \times\\
        &\pi_T(a_T \mid a_1, \mathcal{A}^L_1(a_1), \dots, a_{T - 1}, \mathcal{A}^L_{T-1}(a_1, \dots, a_{T-1}))~.
    \end{align*}
    On the other hand, for the right view:
    \begin{align*}
        \Pr[ \View_{\mathcal{A}, \pi}^{R, t^\star} = (a_1, \dots, a_T)] &= \pi_1(a_1) \pi_2(a_2 \mid a_1, \mathcal{A}^L_1(a_1)) \dots \times \\
        &\pi_{t^\star + 1}(a_{t^\star + 1} \mid a_1, \mathcal{A}^L_1(a_1), \dots, a_{t^\star}, \mathcal{A}^R_{t^\star}(a_1, \dots, a_{t^\star})) \dots \times \\
        &\pi_T(a_T \mid a_1, \mathcal{A}^L_t(a_1), \dots, a_{T - 1}, \mathcal{A}^L_{T-1}(a_1, \dots, a_{t-1}))~.
    \end{align*}
 
    Let us define $$\mathcal{A}^{L, t^\star}(a_1, \dots, a_T) \deffn (\mathcal{A}^L_1(a_1), \mathcal{A}^L_2(a_1, a_2), \dots,  \mathcal{A}^L_T(a_1, \dots, a_T))$$ to be the list of rewards that the policy observes when the protocol is run on the left channel. Also,
    $$\mathcal{A}^{R, t^\star}(a_1, \dots, a_T) \deffn (\mathcal{A}^L_1(a_1), \dots, \mathcal{A}^R_{t^\star}(a_1, \dots, a_{t^\star}) \dots \mathcal{A}^L_T(a_1, \dots, a_T))$$ is the list of rewards that the policy observes when the protocol is run on the right channel and $t^\star$. 
 
    We observe that, for any $(a_1, \dots, a_T) \in [K]^T$,
    \begin{itemize}
        \item[(a)] $\Pr[ \View_{\mathcal{A}, \pi}^{L, t^\star} = (a_1, \dots, a_T)] = \mathcal{V}^\pi ((a_1, \dots, a_T) \mid  \mathcal{A}^{L, t^\star}(a_1, \dots, a_T))$.
        \item[(b)] $\Pr[ \View_{\mathcal{A}, \pi}^{R, t^\star} = (a_1, \dots, a_T)] = \mathcal{V}^\pi ((a_1, \dots, a_T) \mid  \mathcal{A}^{R, t^\star}(a_1, \dots, a_T))$. 
        \item [(c)] $\mathcal{A}^{L, t^\star}(a_1, \dots, a_T)$ and $\mathcal{A}^{R, t^\star}(a_1, \dots, a_T)$ are neighbouring lists of rewards, and only differ at the $t^\star$-th element.
    \end{itemize}
    This remark will help connect the adaptive continual release model with View DP later.
\end{remark}
 
\begin{remark}\label{rem:tree_input}[Reward-feeding Adversary as a Tree Reward Input] A reward-feeding adversary can be represented by a tree of rewards.  Each node in the tree corresponds to a reward input. The tree has a depth of size $T$. At depth $t \in [T]$ of the tree reside all possible rewards the policy can observe at step $t$. Going from depth $t$ to depth $t + 1$ depends on the action $a_{t+1}$. Finally, the policy only observes the reward corresponding to its trajectory in the tree. An example of the tree is presented in Figure~\ref{fig:inp_rep}.c for $T=3$ and $K=2$.
 
 A policy $\pi$ is DP in the adaptive continual release model if and only if $\pi$ is DP when interacting with two neighbouring trees of rewards. Two trees of rewards are neighbouring if they only differ in rewards at one depth $t^\star \in [T]$.
     
\end{remark}

\begin{figure}[t!]
    \centering
    \begin{minipage}{0.4\textwidth}
    \centering
        \includegraphics[width=\linewidth]{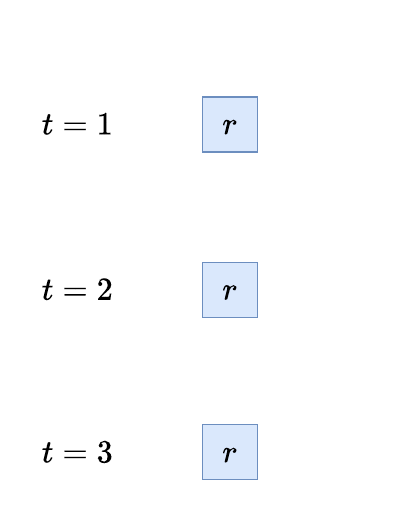}
        {(a) List of rewards}
    \end{minipage}\hfill
    \begin{minipage}{0.6\linewidth}
    \centering
        \includegraphics[width=\linewidth]{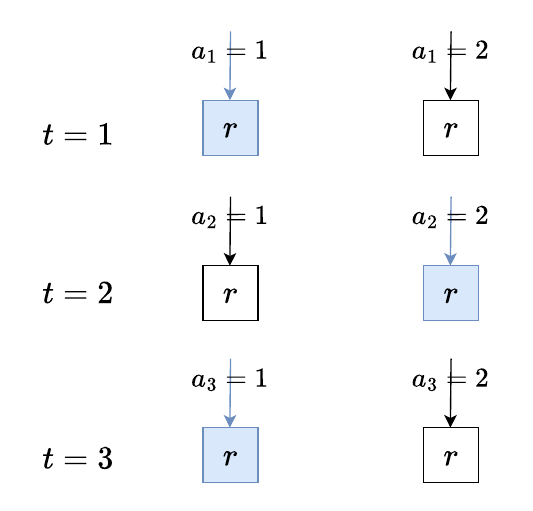}\\
        {(b) Table of rewards}
    \end{minipage}\vfill
    \begin{minipage}{\textwidth}
    \centering
        \includegraphics[width=\linewidth]{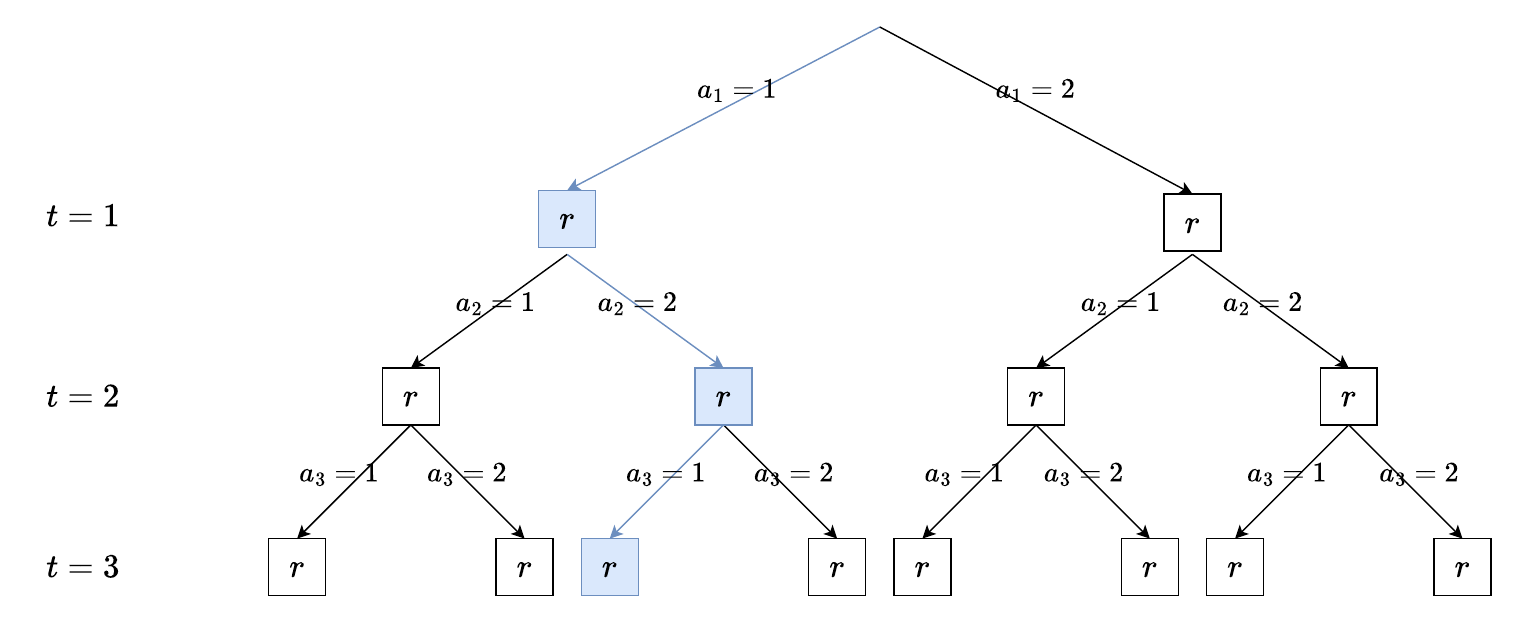}
        {(c) Tree of rewards}
        %\label{fig:tree_rew}
    \end{minipage}
    \caption{Different reward representations for $T = 3$ and $K = 2$. The highlighted rewards are the rewards observed by the policy for the trajectory $(a_1, a_2, a_3) = (1,2,1)$}\label{fig:inp_rep}
\end{figure}

Now, we relate DP in the adaptive continual release model with View DP and Table DP.

\begin{proposition}[Link between the Adaptive Continual Release Model, View DP, and Table DP]\label{prop:adap_table_view}
    For any policy $\pi$, we have that
       \begin{itemize}
           \item[(a)] $\pi$ is DP in the adaptive continual release model $\Rightarrow$ $\pi$ is Table DP.
           \item[(b)] $\pi$ is $\epsilon$-DP in the adaptive continual release model $\Leftrightarrow$ $\pi$ is $\epsilon$-Table DP $\Leftrightarrow$ $\pi$ is $\epsilon$-View DP.
       \end{itemize}
   \end{proposition}
   
Proposition~\ref{prop:adap_table_view} shows that the adaptive continual release model is stronger than Table DP. For pure $\epsilon$-DP, the adaptive continual release model, Table DP and View DP are all equivalent.

To prove this proposition, we use the following reduction.

\begin{reduction}[From table of rewards to ``reward-feeding" adversaries]\label{red:table_to_adv}
For a pair of reward tables $\textbf{x}, \textbf{x'} \in (\real^\arms)^\horizon$, we define $\mathcal{A}(\textbf{x}, \textbf{x'})$ to be the ``reward-feeding" adversary defined by
\[
\mathcal{A}(\textbf{x}, \textbf{x'})_t: a_1, \dots, a_t \rightarrow (x_{t, a_t}, x'_{t, a_t})~.
\]

In other words, at step $t$, the adversary $\mathcal{A}(\textbf{x}, \textbf{x'})$ only uses the last action $a_t$ and returns the $a_t$-th column from $x_t$ on the left channel, and the $a_t$-th column from $x'_t$ on the right channel.

For neighbouring tables $\textbf{x}$ and $ \textbf{x'}$ which only differ at some step $t^\star$, it is possible to show that, for every $S \in \real^T$, we have
    \begin{itemize}
        \item $\Pr[ \View_{\mathcal{A}(\textbf{x}, \textbf{x'}), \pi}^{L, t^\star} \in \mathcal{S}] = \mathcal{M}^\pi_\textbf{x}(S)$.
        \item $\Pr[ \View_{\mathcal{A}(\textbf{x}, \textbf{x'}), \pi}^{R, t^\star} \in \mathcal{S}] = \mathcal{M}^\pi_\textbf{x'}(S)$.
    \end{itemize}

    In other words, the batch mechanism $\mathcal{M}^\pi$ combined with neighbouring tables can be ``simulated" using a specific type of ``reward-feeding" adversaries that only care about the last action from the history. 
\end{reduction}

% \begin{reduction}[From "reward-feeding" adversaries to list of rewards]\label{red:table_to_adv}

% \end{reduction}

\begin{proof}
    (a) Suppose that $\pi$ is DP in the adaptive continual release model.

    Let $t^\star \in [T]$, and $x \sim x'$ be two tables of rewards in $(\real^\arms)^\horizon$ that only differ at step $t^\star$.
    Using Reduction~\ref{red:table_to_adv}, we build $\mathcal{A}(x, x')$.

    For this construction, we have that $\mathcal{M}^\pi_x = \View_{\mathcal{A}(\textbf{x}, \textbf{x'}), \pi}^{L, t^\star}$ and $\mathcal{M}^\pi_{x'} = \View_{\mathcal{A}(\textbf{x}, \textbf{x'}), \pi}^{R, t^\star}$.

    Since $\pi$ is DP in the adaptive continual release model, $\View_{\mathcal{A}(\textbf{x}, \textbf{x'}), \pi}^{L, t^\star}$ and $\View_{\mathcal{A}(\textbf{x}, \textbf{x'}), \pi}^{L, t^\star} $ are indistinguishable. Thus, $\mathcal{M}^\pi_x$ and $\mathcal{M}^\pi_{x'}$ are indistinguishable, \ie, $\mathcal{M}^\pi$ is DP and $\pi$ is Table DP. 
    
(b) To prove this part, it is enough to show that $\epsilon$-View DP implies $\epsilon$-DP in the adaptive continual release model.

Suppose that $\pi$ is $\epsilon$-View DP, \ie $\mathcal{V}^\pi$ is $\epsilon$-DP.
Let $\mathcal{A}$ be a ``reward-feeding" adversary, and $(a_1, \dots, a_T) \in [K]^T$ a sequence of arms.

Using Remark~\ref{rem:int_cont} and the notation defined there, we have
\begin{align*}
    \Pr[ \View_{\mathcal{A}, \pi}^{L, t^\star} = (a_1, \dots, a_T)] &= \mathcal{V}^\pi ((a_1, \dots, a_T) \mid  \mathcal{A}^{L, t^\star}(a_1, \dots, a_T))\\
    &\leq e^\epsilon \mathcal{V}^\pi ((a_1, \dots, a_T) \mid  \mathcal{A}^{R, t^\star}(a_1, \dots, a_T))\\
    &= e^\epsilon \Pr[ \View_{\mathcal{A}, \pi}^{L, t^\star} = (a_1, \dots, a_T)],
\end{align*}
where the inequality holds because $\mathcal{V}^\pi$ is DP, and~  $\mathcal{A}^{L, t^\star}(a_1, \dots, a_T)$ and $\mathcal{A}^{R, t^\star}(a_1, \dots, a_T)$ are neighbouring lists of rewards.

Finally, this means that $\pi$ is $\epsilon$-DP in the adaptive continual release model, since for pure DP, it is enough to check the atomic events $(a_1, \dots, a_T)$.

Note that the proof breaks if we consider composite events, which are necessary for approximate DP proofs.
\end{proof}

\noindent \textbf{Summary of the relationship between definitions.} We introduced three increasingly stronger input representations and their corresponding DP definitions: list of rewards with View DP, table of rewards with Table DP, and tree of rewards with DP in the adaptive continual release. These representations are summarised in Figure~\ref{fig:inp_rep} for $T = 3$ and $K=2$.

In general, DP in the adaptive continual release is stronger than Table DP, which is stronger than View DP. For $\epsilon$-pure DP, these three definitions are equivalent, with the same privacy budget $\epsilon$. More care is needed for other variants of DP, where going from one definition to another happens with a loss in the privacy budgets (Proposition 1 in~\cite{azize2022privacy}).

% Besides reward representation, Interactive DP is a definition that deals with another degree of freedom: which actions are "fed" to the policy. In Interactive DP, these actions do not come from the policy itself like previous definitions, but from a "query-action" adversary that chooses the query action depending on the interaction history. This additional freedom helps protect users' privacy even if they do not follow the actions recommended by the policy. Also, the Interactive DP definition decouples the rewards from actions which helps to prove stronger group privacy properties. For policies closed under post-processing, \ie policies that do not care about the source of the actions, Interactive DP is equivalent to its "normal" notion counterpart.
% In definition~\ref{def:int_dp}, we used the table of rewards as the reward representation, and thus, the "normal" counterpart is Table DP. It is also possible to use other reward representations for Interactive DP, \ie lists or trees. All policies considered later are closed under post-processing.
\section{Lower Bound Proof}\label{app:lb_proof}
In this section, we present the proof of the three main lemma used to prove Theorem~\ref{thm:low_bound}. We adopt the same notation introduced in the proof of Theorem~\ref{thm:low_bound}.

\begin{lemma}[Controlling $\mathbb{P}_{\gamma \pi} (\Omega \cap L \cap A)$, aka Double change of environment]\label{lem:doub_env}
    We show that
    \begin{align}
    \mathbb{P}_{\gamma \pi} \left(\Omega \cap L \cap A\right) \leq e^{(1 + \alpha) n_2 \left(\epsilon \mathrm{TV}(P_2, P'_2) + \mathrm{kl}(\mu_2',\mu''_2) \right) }  \frac{O(T^a)}{T - n_2},
    \end{align}
    for any $a>0$.
\end{lemma}

% \underline{Step 3: controlling $\mathbb{P}_{\gamma \pi} (\Omega \cap L \cap A) $}

\begin{proof}
We have
\begin{align*}
    &\mathbb{P}_{\gamma \pi} \left(\Omega \cap L \cap A\right)\\
    &= \sum_{\textbf{a}} \int_{\textbf{r}} \int_{\textbf{r'}} \mathds{1}(\Omega \cap L \cap A) \prod_{t=1}^T \pi_t(a_t \mid a_1 , r_1 , \dots , a_{t-1} , r_{t-1} ) c_{a_t} (r_t, r'_t) \dd r_{t} \dd r'_{t} \\
    &\stackrel{(a)}{\leq} \sum_{\textbf{a}} \int_{\textbf{r}} \int_{\textbf{r'}} \mathds{1}(\Omega \cap L \cap A) e^{\epsilon  \mathrm{dham}(r, r')} \prod_{t=1}^T \pi_t(a_t \mid a_1 , r'_1 , \dots , a_{t-1} , r'_{t-1} ) c_{a_t} (r_t, r'_t) \dd r_{t} \dd r'_{t}\\
    &\stackrel{(b)}{\leq} e^{\epsilon (1 + \alpha) n_2 \mathrm{TV}(P_2, P'_2)} \sum_{\textbf{a}} \int_{\textbf{r}} \int_{\textbf{r'}} \mathds{1}(\Omega \cap L \cap A) \prod_{t=1}^T \pi_t(a_t \mid a_1 , r'_1 , \dots , a_{t-1} , r'_{t-1} ) c_{a_t} (r_t, r'_t) \dd r_{t} \dd r'_{t}\\
    &\stackrel{(c)}{\leq} e^{\epsilon (1 + \alpha) n_2 \mathrm{TV}(P_2, P'_2)} \sum_{\textbf{a}} \int_{\textbf{r}} \int_{\textbf{r'}} \mathds{1}(\Omega \cap A) \prod_{t=1}^T \pi_t(a_t \mid a_1 , r'_1 , \dots , a_{t-1} , r'_{t-1} ) c_{a_t} (r_t, r'_t) \dd r_{t} \dd r'_{t}\\
    &\stackrel{(d)}{=}  e^{\epsilon (1 + \alpha) n_2 \mathrm{TV}(P_2, P'_2)} \sum_{\textbf{a}} \int_{\textbf{r'}} \mathds{1}(\Omega \cap A) \prod_{t=1}^T \pi_t(a_t \mid a_1 , r'_1 , \dots , a_{t-1} , r'_{t-1} ) p'_{a_t} (r'_t)  \dd r'_{t}\\
    &= e^{\epsilon (1 + \alpha) n_2 \mathrm{TV}(P_2, P'_2)} \sum_{\textbf{a}} \int_{\textbf{r'}} \mathds{1}(\Omega \cap A) e^{\sum_{t=1}^{T} \log \frac{\dd P_{a_t}'(r'_t)}{\dd P_{a_t}''(r'_t)}} \prod_{t=1}^T \pi_t(a_t \mid a_1 , r'_1 , \dots , a_{t-1} , r'_{t-1} ) p''_{a_t} (r'_t) \dd r'_{t}\\
    &\stackrel{(e)}{\leq} e^{\epsilon (1 + \alpha) n_2 \mathrm{TV}(P_2, P'_2)} e^{(1 + \alpha)\mathrm{kl}(\mu_2',\mu''_2)n_2} \sum_{\textbf{a}} \int_{\textbf{r'}} \mathds{1}(\Omega)  \prod_{t=1}^T \pi_t(a_t \mid a_1 , r'_1 , \dots , a_{t-1} , r'_{t-1} ) p''_{a_t} (r'_t)  \dd r'_{t} \\
    &= e^{(1 + \alpha) n_2 \left(\epsilon \mathrm{TV}(P_2, P'_2) + \mathrm{kl}(\mu_2',\mu''_2) \right) } \mathbb{P}_{\nu'' \pi} \left(N_2(T) \leq n_2 \right),
\end{align*}
where:

(a) is because $\pi$ is $\epsilon$-DP;

(b) is by definition of $L$;

(c) is because $\mathds{1}(\Omega \cap L \cap A) \leq \mathds{1}(\Omega \cap A)$;

(d) by definition of the coupling, and because $\Omega \cap A$ doesn't depend on $(r_t)_{t=1}^T$;

(e) by definition of $A$.

Then, using Markov inequality and the consistency of $\pi$, we get
\begin{align*}
    \mathbb{P}_{\nu'' \pi} \left(N_2(T) \leq n_2 \right) &= \mathbb{P}_{\nu'' \pi} \left( T - N_2(T) \geq T - n_2 \right)\\
    &= \mathbb{P}_{\nu'' \pi} \left(N_1(T) \geq T - n_2 \right)\\
    &\leq \frac{\mathbb{E}_{\nu'' \pi}(N_1(T))}{T - n_2} = \frac{O(T^\alpha)}{T - n_2},
\end{align*}
for any $a > 0$, since arm 1 is sub-optimal in environment $\nu''$ and $\pi$ is consistent.

All in all, we have that, for any $a>0$,
\begin{align*}
    \mathbb{P}_{\gamma \pi} \left(\Omega \cap L \cap A\right) \leq e^{(1 + \alpha) n_2 \left(\epsilon \mathrm{TV}(P_2, P'_2) + \mathrm{kl}(\mu_2',\mu''_2) \right) }  \frac{O(T^a)}{T - n_2}~.
\end{align*}
\end{proof}

\begin{lemma}[Controlling $\mathbb{P}_{\gamma \pi} (\Omega \cap L \cap A^c) $]\label{lem:p_omega_L_A_c}
    Choosing $n_2 = n_2(T)$ a function such that $n_2(T) \rightarrow \infty$ when $T \rightarrow \infty$, then 
\[
    \mathbb{P}_{\gamma \pi} (\Omega \cap L \cap A^c) = o_T(1),
\]
asymptotically in $T$.
\end{lemma}

% \underline{Step 4: }
\begin{proof}
First, we have
\begin{align*}
    \mathbb{P}_{\gamma \pi} \left(\Omega \cap L \cap A^c\right) &\leq \mathbb{P}_{\gamma \pi} \left(\Omega \cap A^c\right)~.
\end{align*}

Let us introduce the notation $r'_{a, s} \triangleq r'_{\tau_{a,s}}$ where $\tau_{a,s} \triangleq \min \{t \in \mathbb{N}: N_a(t) = s\}$. Then, 
$$\sum_{t=1}^{T} \log \frac{\dd P_{a_t}'(r'_t)}{\dd P_{a_t}''(r'_t)} = \sum_{s=1}^{N_2(T)} \log \frac{\dd P_{2}'(r'_{2,s})}{\dd P_{2}''(r'_{2,s})} = \sum_{s=1}^{N_2(T)} W_s,$$
where $W_s \triangleq  \log \frac{\dd P_{2}'(r'_{2,s})}{\dd P_{2}''(r'_{2,s})} $ are i.i.d bounded random variables, with positive mean $\mathbb{E}_{\gamma \pi}[W_s] = \mathrm{kl}(\mu'_2, \mu''_2)$. This is true since under the coupling $\gamma$, the marginal of $r'_{2,s}$ is $P'_2$.

Then, we get
\begin{align*}
    \mathbb{P}_{\gamma \pi} \left(\Omega \cap A^c \right) &\leq \mathbb{P}_{\gamma \pi} \left(\exists m \leq n_2: \sum_{s=1}^{m} W_s >  (1 + \alpha) \mathrm{kl}(\mu_2', \mu''_2) n_2\right)\\
    &\leq \mathbb{P}_{\gamma \pi} \left(\frac{\max_{m \leq n_2} \sum_{s=1}^{m} W_s}{n_2} >  (1 + \alpha) \mathrm{kl}(\mu_2', \mu''_2)\right)~.
\end{align*}

Using Asymptotic maximal Hoeffding inequality (Lemma~\ref{lem:asum_hoeff}), we have that 
\[
\lim_{n \rightarrow \infty} \mathbb{P}_{\gamma \pi} \left(\frac{\max_{m \leq n} \sum_{s=1}^{m} W_s}{n} >  (1 + \alpha) \mathrm{kl}(\mu_2', \mu''_2) \right) = 0~.
\]

Thus, by choosing $n_2 = n_2(T)$ a function such that $n_2(T) \rightarrow \infty$ when $T \rightarrow \infty$, then 
\begin{align*}
    \mathbb{P}_{\gamma \pi} (\Omega \cap L \cap A^c) = o_T(1),
\end{align*}
asymptotically in $T$.
\end{proof}

\begin{lemma}[Controlling $\mathbb{P}_{\gamma \pi} \left(\Omega \cap L^c \right)$]\label{lem:p_omega_L_c}
    choosing $n_2 = n_2(T)$ a function such that $n_2(T) \rightarrow \infty$ when $T \rightarrow \infty$, then 
\[
    \mathbb{P}_{\gamma \pi} \left(\Omega \cap L^c \right) = o_T(1),
\]
asymptotically in $T$.
\end{lemma}

% \underline{Step 5: controlling $\mathbb{P}_{\gamma \pi} \left(\Omega \cap L^c \right)$}
\begin{proof}
First, by the construction of the couplings, only rewards coming from arm 2 are different, \ie,
\[
\mathrm{dham}(r, r') \triangleq \sum_{t=1}^T \mathds{1} (r_t \neq r'_t) = \sum_{t=1}^T \mathds{1} (A_t = 2) \mathds{1} (r_t \neq r'_t)~.
\]

Let us introduce the notation $r_{a, s} \triangleq r_{\tau_{a,s}}$ where $\tau_{a,s} \triangleq \min \{t \in \mathbb{N}: N_a(t) = s\}$. Then, 
$$\mathrm{dham}(r, r') = \sum_{s=1}^{N_2(T)} \mathds{1} (r_{2,s} \neq r'_{2,s}) = \sum_{s=1}^{N_2(T)} Z_s,$$
where $Z_s \triangleq \mathds{1} (r_{2,s} \neq r'_{2,s})$ are i.i.d Bernoulli random variables with positive mean $\mathbb{E}_{\gamma \pi}[Z_s] = \mathbb{P}_{\gamma \pi}(r_{2,s} \neq r'_{2,s}) = \mathrm{TV}(P_2, P'_2)$.

\begin{align*}
    \mathbb{P}_{\gamma \pi} \left(\Omega \cap L^c \right) &\leq \mathbb{P}_{\gamma \pi} \left(\exists m \leq n_2: \sum_{s=1}^{m} Z_s >  (1 + \alpha) n_2 \mathrm{TV}(P_2, P'_2)\right)\\
    &\leq \mathbb{P}_{\gamma \pi} \left(\frac{\max_{m \leq n_2} \sum_{s=1}^{m} Z_s}{n_2} >  (1 + \alpha) \mathrm{TV}(P_2, P'_2)\right)~.
\end{align*}

Using Asymptotic maximal Hoeffding inequality (Lemma~\ref{lem:asum_hoeff}), we have that 
\[
\lim_{n \rightarrow \infty} \mathbb{P}_{\gamma \pi} \left(\frac{\max_{m \leq n} \sum_{s=1}^{m} Z_s}{n} >  (1 + \alpha) \mathrm{TV}(P_2, P'_2) \right) = 0~.
\]

Thus, by choosing $n_2 = n_2(T)$ a function such that $n_2(T) \rightarrow \infty$ when $T \rightarrow \infty$, then 
\begin{align}
    \mathbb{P}_{\gamma \pi} \left(\Omega \cap L^c \right) = o_T(1),
\end{align}
asymptotically in $T$.
\end{proof}
\section{Concentration Inequality Proof}\label{app:conc}

\begin{lemma}[Tail Bound of Cumulative Laplacian Noise]\label{lem:TailCumLap}
Let \( Z_{m} = \sum_{l=1}^m Y_{l} \) where \( Y_{l} \sim \mathrm{Lap}(1/\epsilon) \) are i.i.d. Laplace random variables with parameter \( 1/\epsilon \). Then, for \( z > 0 \), we have
\[
\mathbb{P}[Z_{m} \geq z] \leq \exp\left(-f(z)\right),
\]
where \(f(z)=\epsilon z - 1-m\log (1+m\epsilon z)\).
\end{lemma}
%We defer the proof of the above lemma in Appendix \ref{app_sec:Concen}.

\begin{proof}
For a random variable \( Y \sim \mathrm{Lap}(1/\epsilon) \), the probability density function is
\[
f_Y(y) = \frac{\epsilon}{2} \exp(-\epsilon |y|)~.
\]
The moment-generating function (MGF) is given by
\[
M_Y(t) = \mathbb{E}[\exp(tY)] = \frac{\epsilon^2}{\epsilon^2 - t^2}, \quad |t| < \epsilon~.
\]
The random variable \( Z_{m} = \sum_{l=1}^m Y_{l} \) is the sum of \(m\) i.i.d. Laplace random variables. The MGF of \(Z_{m}\) is the product of the MGFs of the individual \(Y_{l}\):
\[
M_{Z_{m}}(t) = \left(M_Y(t)\right)^{m}~.
\]
Thus, we have
\[
M_{Z_{m}}(t) = \left(\frac{\epsilon^2}{\epsilon^2 - t^2}\right)^{m}, \quad |t| < \epsilon~.
\]
To bound \(\mathbb{P}[Z_{m} \geq z]\), we use the Chernoff bound:
\begin{align*}
   \mathbb{P}[Z_{m} \geq z] 
   &\leq \inf_{0<t < \epsilon} \mathbb{E}[\exp(tZ_{m} - tz)]\nn
   &= \inf_{0<t < \epsilon} \exp\left(-tz\right) M_{Z_{m}}(t)\nn
   &=\inf_{0<t < \epsilon} \exp\left(-tz + m \log\left(\frac{\epsilon^2}{\epsilon^2 - t^2}\right)\right)\nn
   &= \inf_{0<t < \epsilon} \exp\left(-tz-m \log\left(1-\frac{t^2}{\epsilon^2}\right)\right)~. 
\end{align*}
Consider
\[
f_t(z)=t z +m\log \left( 1 - \frac{t^2}{\epsilon^2}\right)~.
\]
Letting $t=\epsilon\sqrt{1-c} \in (0,\epsilon)$ for $c=1 \land 1/(m\epsilon z)$ we have
\begin{align*}
f_t(z)
&= \epsilon z \sqrt{1-c}+m\log c
\nn
&\ge \epsilon z
-\epsilon z c
+
m\log (1 \land 1/(m\epsilon z))
\since{by $\sqrt{1-c}\ge 1-c$ for $c\le 1$}
\nn
&= \epsilon z
-(\epsilon z \land 1/m)
+
m\log (1 \land 1/(m\epsilon z))
\nn
&\ge \epsilon z
-1
-
m\log (1 \lor m\epsilon z)
\nn
&\ge \epsilon z
-1
-
m\log (1 + m\epsilon z)~.
\end{align*}
Then, we have
\[
f_t(z)
\ge \epsilon z - 1-m\log (1+m\epsilon z)=f(z),
\]
for $z\ge 0$.
Thus, we obtain
\[
\mathbb{P}[Z_{m} \geq z]  \le \exp\left(-f(z)\right)~.
\]
\end{proof}

\begin{lemma}[Concentration bound of private summation]
\label{lem:PrivateSumUpper}
For $\mu\in(0,1)$ and $\epsilon>0$, let
\begin{align}
\tilde{S}_{n,m}=
\sum_{i=1}^{n} X_{i}+
\sum_{j=1}^m Y_{j},\qquad
X_i \sim \mathrm{Ber}(\mu),\,
Y_{j}\sim \mathrm{Lap}(1/\epsilon)\n
\end{align}
be the sum of independent $n$ Bernoulli random variables (RVs) with mean $\mu$
and $m$ Laplace RVs with scale $1/\epsilon$.
Then, for $x \ge  n\mu$
\[
\Pr\left[
\tilde{S}_{n,m} \ge x
\right]
\le A_{\epsilon}(n,m, x, \mu)\e^{-n \mathrm{d}_\epsilon(x/n,  \mu)},
\]
where  
\begin{align}
\lefteqn{
A_{\epsilon}(m, n, x, \mu)
}\nn
&=(x-n \mu)\max_{y \in [\mu,x/n]}
\left\{
\e(1+ m\epsilon(x-yn))^{m}\log \frac{1}{\mu}
\right\}+\e(1+ m\epsilon (x-n \mu))^{m}+1~.   \n
\end{align}
Similarly, for $x \le n\mu$,
\[
\Pr\left[
\tilde{S}_{n,m} \le x
\right]\le A_{\epsilon}(m, n, x, \mu)\e^{-n \mathrm{d}_\epsilon(x/n,  \mu)},
\] 
where
\begin{align}
\lefteqn{
A_{\epsilon}(m, n, x, \mu)
}\nn
&=(n\mu-x)\max_{y \in [x/n, \mu]}
\left\{
\e(1+ m\epsilon(yn-x))^{m}\log \frac{1}{1-\mu}
\right\}+\e(1+ m\epsilon (n\mu-x))^{m}+1~.\n
\end{align}
\end{lemma}

%\jhonda{Concentration inequality particularised to our setting. Proof needs modification for this form.}

\begin{proof}[Proof of Lemma \ref{lem:PrivateSumUpper}]
For $\mu\in(0,1)$ and $\epsilon>0$, the private summation can be written as
\begin{align}
\tilde{S}_{n,m}=
\sum_{i=1}^{n} X_{i}+
\sum_{j=1}^m Y_{j}, \qquad X_i \sim \mathrm{Ber}(\mu), Y_{j}\sim \mathrm{Lap}(1/\epsilon)~.
\end{align}
Re-define the non-private summation and the sum of the noise by
\begin{align}
{S}_{n}=
\sum_{i=1}^{n} X_{i},\quad  Z_{m}=\sum_{j=1}^m Y_{j}
\end{align}
and denote density of $ Z_{m}$ by $f_{m}(z)$.
Then, we can upper bound the probability by
\begin{align}
\Pr\left[
\tilde{S}_{n,m} \ge x
\right]
&=
\Pr\left[
S_{n} + Z_{m}
\ge x
\right]\nonumber
\\
&=
\int_{-\infty}^{\infty}
f_{m}(z)\Pr[S_{n}\ge x-z]\rd z\nonumber
\\
&=
\int_{-\infty}^{0}
f_{m}(z)\Pr[S_{n}\ge x-z]\rd z
+
\int_{0}^{\infty}
f_{m}(z)\Pr[S_{n}\ge x-z]\rd z \nonumber
\\
&\le
\int_{-\infty}^{0}
f_{m}(z)\Pr[S_{n}\ge x]\rd z
+
\int_{0}^{\infty}
f_{m}(z)\Pr[S_{n}\ge x-z]\rd z\nonumber
\\
&=
\underbrace{\frac{1}{2}\Pr[S_{n}\ge x]}_{\text{(I)}}
+
\underbrace{\int_{0}^{\infty}
f_{m}(z)\Pr[S_{n}\ge x-z]\rd z}_{\text{(II)}}~. \label{eq:NoisySumDecom}
\end{align}
Here, $\Pr[S_{n}\ge x-z]$ can be upper bounded by Chernoff bound. Let $\bar{P}(x-z)$ be such an upper bound. Then, from Lemma~\ref{lem:ChernoffKL}, we have
\begin{align}
\bar{P}(x-z)=\e^{-n\cdot \mathrm{kl}((x-z)/n,\mu)}, \quad\text{for} \quad x-z\ge n\mu~.
\end{align}
Based on this upper bound, we can bound the second term in \eqref{eq:NoisySumDecom}:
\begin{align}
    \text{(II)} &= \int_{0}^{\infty}
f_{m}(z)\Pr[S_{n}\ge x-z]\rd z\nn
&\le
\int_{0}^{\infty}
f_{m}(z)\bar{P}(x-z)\rd z
\nn
&=
[-F_{m}(z)\bar{P}(x-z)]_0^{\infty}
+
\int_{0}^{\infty}
F_{m}(z)(-\bar{P}'(x-z))\rd z
\since{integration by parts}
\nn
&=
F_{m}(0)\bar{P}(x)
+
\int_{0}^{\infty}
F_{m}(z)(-\bar{P}'(x-z))\rd z
\nn
&=
\frac{1}{2}\bar{P}(x)
+
\int_{0}^{\infty}
F_{m}(z)(-\bar{P}'(x-z))\rd z,
\label{eq:integration_by_parts}
\end{align}
where $F_{m}(z)=\int_{z}^{\infty}f_m(z)\rd z=\Pr[Z_{m}\ge z]$ is the (complement) cumulative distribution. From Lemma \ref{lem:TailCumLap}, we have 
\[
F_{m}(z)=\Pr[Z_{m}\ge z]  \leq \exp\left(-f(z)\right),
\]
where \(f(z)=\epsilon z - 1-m\log (1+m\epsilon z)\). Thus, we can bound the second term in \eqref{eq:integration_by_parts}:
\begin{align}
\int_{0}^{\infty}&
F_{m}(z)(-\bar{P}'(x-z))\rd z
\nn
&=
\int_{0}^{x-n \mu}
F_{m}(z)(-\bar{P}'(x-z))\rd z
+
\int_{x-n \mu}^{\infty}
F_{m}(z)(-\bar{P}'(x-z))\rd z
\since{$F_{m}(z)$ is decreasing}
\nn
&\le
\int_{0}^{x-n \mu}
F_{m}(z)(-\bar{P}'(x-z))\rd z
+
F_{m}(x-n \mu)
\int_{x-n \mu}^{\infty}
(-\bar{P}'(x-z))\rd z
\nn
&=
\int_{0}^{x-n \mu}
F_{m}(z)(-\bar{P}'(x-z))\rd z
+
F_{m}(x-n \mu)
\bar{P}(n \mu)
\nn
&\le
\int_{0}^{x-n \mu}
\e^{-f(z)}(-\bar{P}'(x-z))\rd z
+
\e^{-f(x-n \mu)}\cdot 1~.\label{eq:TheSecond}
\end{align}
We now focus on bounding the first term in RHS of the last inequality. Observe that 
\begin{align}
-\bar{P}'(z)=\mathrm{kl}'(z/n, \mu)\e^{-n \cdot\mathrm{kl}(z/n, \mu)},
\end{align}
where $\mathrm{kl}'(x,y)=\frac{\partial \mathrm{kl}(x,y)}{\partial x}$ is the derivative with respect to the first argument.
%where $\mathrm{kl}'$ is the derivative with respect to the first argument $z/n$.
Then, for $x-z\ge n \mu$, we have
\begin{comment}
\begin{align}
\int_{0}^{x-n \mu}&
\e^{-f(z)}(-\bar{P}'(x-z))\rd z
\nn
=&
\int_{0}^{x-n \mu}
\e(1+ m\epsilon z)^{m} \mathrm{kl}'((x-z)/n, \mu)
\e^{-\epsilon z}\e^{-n \mathrm{kl}((x-z)/n, \mu)}\rd z %\since{let $y=x-z$}
\nn
=&\int_{\mu}^{x/n}
n\e(1+ m\epsilon (x-y n))^{m} \mathrm{kl}'(y, \mu)
\e^{-\epsilon (x-y n)}\e^{-n \cdot\mathrm{kl}(y, \mu)}\rd y \since{ let $y:=(x-z)/n $}
\nn
{\le}&
(x-n \mu)\max_{y \in [\mu,x/n]}
\left\{
\e(1+ m\epsilon  (x-y n))^{m} \mathrm{kl}'(y, \mu)
\right\}
\e^{-\inf_{y \in [\mu,x/n]}\{\epsilon (x-y n)+n \cdot\mathrm{kl}(y, \mu)\}}
\nn
=&(x-n \mu)\max_{y \in [\mu,x/n]}
\left\{
\e(1+ m\epsilon  (x-y n))^{m} \mathrm{kl}'(y, \mu)
\right\}
\e^{-n\inf_{y \in [\mu,x/n]}\{\epsilon (x/n-y)+ \mathrm{kl}(y, \mu)\}}\nn
 =& (x-n \mu)\max_{y \in [\mu,x/n]}
\left\{
\e(1+ m\epsilon  (x-y n))^{m} \mathrm{kl}'(y, \mu)
\right\}
\e^{-n \cdot \mathrm{d}_\epsilon(x/n,  \mu)}
% (x-n_m \mu_i)\max_{z \in [0,x-n_m \mu_i]}
% \left\{
% \e(1+ m\epsilon z)^m \mathrm{KL}'(z/n_m, \mu_i)
% \right\}
% \e^{-n_m \inf_{z \in [0,x/n_m-\mu_i]}\{\epsilon z+\mathrm{KL}(z, \mu_i)\}}
% \since{let $z:= z+\mu_i$}
% \nn
% &=
% (x-n_m \mu_i)
% \max_{z \in [0,x-n_m \mu_i]}
% \left\{
% \e(1+ m\epsilon z)^m \KL'(z/n_m, \mu_i)
% \right\}
% \e^{-n_m \KL(x/n_m, \mu_i, \epsilon)}
\label{P2}
\end{align}
\end{comment}
%\jhonda{beginning of the fix with a slight improvement.}
\begin{align}
\int_{0}^{x-n \mu} &
\e^{-f(z)}(-\bar{P}'(x-z))\rd z
\nn
=&
\int_{0}^{x-n \mu}
\e(1+ m\epsilon z)^{m} \mathrm{kl}'((x-z)/n, \mu)
\e^{-\epsilon z}\e^{-n \cdot \mathrm{kl}((x-z)/n, \mu)}\rd z %\since{let $y=x-z$}
\nn
=&\int_{\mu}^{x/n}
n\e(1+ m\epsilon (x-y n))^{m} \mathrm{kl}'(y, \mu)
\e^{-\epsilon (x-y n)}\e^{-n \cdot\mathrm{kl}(y, \mu)}\rd y \since{ let $y:=(x-z)/n $}
\nn
\le&
\e^{-\inf_{y \in [\mu,x/n]}\{\epsilon (x-y n)+n \cdot\mathrm{kl}(y, \mu)\}}
\int_{\mu}^{x/n}
n\e(1+ m\epsilon (x-y n))^{m} \mathrm{kl}'(y, \mu)
\rd y
\nn
\le&
\e^{-n \cdot \mathrm{d}_\epsilon(x/n,  \mu)}
\int_{\mu}^{x/n}
n\e(1+ m\epsilon (x-\mu n))^{m} \mathrm{kl}'(y, \mu)
\rd y
\nn
=&
\e^{-n \cdot \mathrm{d}_\epsilon(x/n,  \mu)}
n\e(1+ m\epsilon (x-\mu n))^{m}\mathrm{kl}(x/n, \mu)
\nn
\le&
\e^{-n \cdot \mathrm{d}_\epsilon(x/n,  \mu)}
n\e(1+ m\epsilon (x-\mu n))^{m}\mathrm{kl}(1, \mu)
\nn
=&
\e^{-n \cdot \mathrm{d}_\epsilon(x/n,  \mu)}
n\e(1+ m\epsilon (x-\mu n))^{m}\log \frac{1}{\mu}
\label{P2}
\end{align}
%\jhonda{end of the modification}
where $\mathrm{d}_\epsilon(x/n,  \mu)$ is defined in \eqref{eq:d_eps}.
%%%%%%%%%%%%%%%%%%%%%%%%%%%%%%
%%%%%%%%%%%%%Until here
%%%%%%%%%%%%%%%%%%%%%%%%%%%%%%%
Now, we bound the second term in \eqref{eq:TheSecond}:
\begin{align}
\e^{-f(x-n \mu)}
=& 
\e(1+ m\epsilon (x-n \mu))^{m}
\e^{-n \epsilon(x/n-\mu)}\nn
\le &\e(1+ m\epsilon (x-n \mu))^{m} \e^{-n \mathrm{d}_\epsilon(x/n,  \mu)}~.
\label{P3}
\end{align}
Note that we have for $x\ge n\mu$
\begin{align}
\Pr[S_{n}\ge x]
\le\bar{P}(x)
&\le
\e^{-n \mathrm{kl}(x/n, \mu)}
\since{by Lemma \ref{lem:ChernoffKL}}\nn
&\le \e^{-n \mathrm{d}_\epsilon(x/n,  \mu)}~.
\label{P1}
\end{align}
Putting \eqref{P2}, \eqref{P3}, and  \eqref{P1} together we have for $x/n \ge \mu$
\[\Pr\left[
\tilde{S}_{n,m} \ge x
\right]\le A_{\epsilon}(m, n, x, \mu)\e^{-n \cdot\mathrm{d}_\epsilon(x/n,  \mu)}\]
where 
\begin{align*}
    A_{\epsilon}(m, n, x, \mu)
    =(x-n \mu)\max_{y \in [\mu,x/n]}
\left\{
\e(1+ m\epsilon  (x-y n))^{m} \log \frac{1}{\mu}
\right\}+\e(1+ m\epsilon (x-n \mu))^{m}+1~.
\end{align*}

Similarly, we can get for $x/n \le \mu$
\[
\Pr\left[
\tilde{S}_{n,m} \le x
\right]\le A_{\epsilon}(m, n, x, \mu)\e^{-n \mathrm{d}_\epsilon(x/n,  \mu)},
\]
where 
\begin{align*}
\lefteqn{
A_{\epsilon}(m, n, x, \mu)
}\nn
&=(x-n \mu)\max_{y \in [\mu,x/n]}
\left\{
\e(1+ m\epsilon  (x-y n))^{m} \log \frac{1}{1-\mu}
\right\}+\e(1+ m\epsilon (x-n \mu))^{m}+1~.
\end{align*}
\end{proof}

\begin{corollary}[Concentration bound of private mean]\label{corol}
Consider $\tilde{S}_{n,m}$ given in Lemma~\ref{lem:PrivateSumUpper}.
Let $x\in[0,1]$.
Let $\{n_m\}_{m\in\mathbb{N}}$ be a sequence such that $m/n_m=o(1)$.
Then, for any $a>0$
there exists a constant $\Aa>0$ such that
for all $m\in \mathbb{N}$
\begin{align}
\Pr\left[
\frac{\tilde{S}_{n_m,m}}{n_m}
\ge x
\right]
&\le
A_a
\e^{-n_m  (\mathrm{d}_\epsilon(x, \mu)-a)},\qquad x\ge \mu~.\nn
\Pr\left[
\frac{\tilde{S}_{n_m,m}}{n_m}
\le x
\right]
&\le
A_a
\e^{-n_m  (\mathrm{d}_\epsilon(x, \mu)-a)},\qquad x\le \mu.
\n
%\label{concentration_inequality}
\end{align}
\end{corollary}

\begin{proof}[Proof of Corollary \ref{corol}] From Lemma \ref{lem:PrivateSumUpper}, we have for $x\ge \mu$
\begin{align}
     &A_{\epsilon}(m, n_m, x, \mu)\nn
 =&n_m(x- \mu)\max_{y \in [\mu,x]}
\left\{
\e(1+ (m+1)\epsilon n_m(x-y))^{m+1} \log\frac{1}{\mu}
\right\}+\e(1+ (m+1)\epsilon n_m(x- \mu))^{m+1}+1~.
\end{align}
    For $y \in [\mu,x]$, 
   \[A_{\epsilon}(m, n_m, x, \mu) \le A(n_m)=n_m
\e(1+ (m+1)\epsilon n_m)^{m+1} \log\frac{1}{\mu}+\e(1+ (m+1)\epsilon n_m)^{m+1}+1~.
\] 
Since existing $b$ to make $1+x\le be^x$ hold, we have the result. 
The proof for the case of $x\le \mu$ is completely analogous.
%Similar to $x\le \mu$ case.

% When \(\tilde{\mu}_{i,m} \in [0,1]\):
% If the original estimator \(\tilde{\mu}_{i,m}\) lies within the interval \([0,1]\), the clipping function does nothing:
% \[
% [\tilde{\mu}_{i,m}]_0^1 = \tilde{\mu}_{i,m}.
% \]
% Thus, for \(x \in [0,1]\), the event \([\tilde{\mu}_{i,m}]_0^1 \ge x\) is equivalent to \(\tilde{\mu}_{i,m} \ge x\). That is:
% \[
% \Pr\left[[\tilde{\mu}_{i,m}]_0^1 \ge x\right] = \Pr\left[\tilde{\mu}_{i,m} \ge x\right].
% \]
% When \(\tilde{\mu}_{i,m} \notin [0,1]\):
% If the original estimator \(\tilde{\mu}_{i,m}\) lies outside the interval \([0,1]\), the clipping ensures that:
% \[
% [\tilde{\mu}_{i,m}]_0^1 = 1, \quad \text{if } \tilde{\mu}_{i,m} > 1,
% \]
% \[
% [\tilde{\mu}_{i,m}]_0^1 = 0, \quad \text{if } \tilde{\mu}_{i,m} < 0.
% \]
% However, if \(x \in [0,1]\), these cases do not affect the probability of the event \([\tilde{\mu}_{i,m}]_0^1 \ge x\), because such cases are outside the range of \(x\).
\end{proof}

%\yulian{modification until here}

\section{Privacy Analysis}\label{app:priv}
First, we provide a simple lemma to motivate the intuition behind the algorithm design. Then, we provide a complete proof of Proposition~\ref{prop:priv}.

\begin{lemma}[Continual release of noisy rewards]\label{lem:priv}
    Let rewards $\{r_1, \dots, r_T\} \in [0,1]^T$. Let $1 = t_1< t_2 \dots < t_\ell = T + 1$ be $\ell$ time-step, with $\ell  \leq T$. Then, the mechanism
        \begin{align*}
    \begin{pmatrix}
r_1\\ 
r_2\\ 
\vdots \\
r_T
\end{pmatrix} \overset{\mathcal{C}}{\rightarrow}  \begin{pmatrix}
&r_1 + \dots + r_{t_2 - 1} + Y_1\\ 
&r_1 + \dots + r_{t_3 - 1} + Y_1 + Y_2\\
\hfill &\vdots \\
&r_1 + \dots + r_{T} + Y_1 + Y_2 + \dots + Y_{\ell - 1}
\end{pmatrix}
\end{align*}
    is $\epsilon$-DP, where $(Y_1, \dots, Y_\ell) \sim^\text{iid} \mathrm{Lap}(1/\epsilon)$.
\end{lemma}

\begin{proof}[Proof of Lemma~\ref{lem:priv}]
    First, consider trying to release the following partial sums 
\begin{align*}
    \begin{pmatrix}
r_1\\ 
r_2\\ 
\vdots \\
r_T
\end{pmatrix} {\rightarrow}  \begin{pmatrix}
&r_1 + \dots + r_{t_2 - 1}\\ 
&r_{t_2} + \dots + r_{t_3 - 1}\\
\hfill &\vdots \\
&r_{t_{\ell - 1}} + \dots + r_T
\end{pmatrix}~.
\end{align*}
Because the rewards are in $[0,1]$, the sensitivity of each partial sum is $1$. Since each partial sum is computed on non-overlapping sequences, combining the Laplace mechanism (Theorem~\ref{thm:laplace}) with the parallel composition property of DP (Lemma~\ref{lem:paral_compo}) gives that  

\begin{align*}
    \begin{pmatrix}
r_1\\ 
r_2\\ 
\vdots \\
r_T
\end{pmatrix} \overset{\mathcal{P}}{\rightarrow}  \begin{pmatrix}
&r_1 + \dots + r_{t_2 - 1} + Y_1\\ 
&r_{t_2} + \dots + r_{t_3 - 1} + Y_2\\
\hfill &\vdots \\
&r_{t_{\ell - 1}} + \dots + r_T + Y_{\ell - 1}
\end{pmatrix}
\end{align*}
is $\epsilon$-DP, where $(Y_1, \dots, Y_{\ell - 1}) \sim^\text{iid} \mathrm{Lap}(1/\epsilon)$.

Consider the post-processing function $f: (x_1, \dots x_{\ell -1}) \rightarrow (x_1, x_1 + x_2, \dots, x_1 + x_2 + \dots + x_{\ell -1})$. Then, we have that that $\mathcal{C} = f \circ \mathcal{P}$. So, by the post-processing property of DP,  $\mathcal{C}$ is $\epsilon$-DP.
\end{proof}

\begin{proof}[Proof of Proposition~\ref{prop:priv}]
    Let $\pi$ be either \dpimed{} or \dpklucb{}.
    Let $\textbf{r} \defn \{r_1, \dots, r_T\}$ and $\textbf{r'} \defn \{r'_1, \dots, r'_T\}$ be two neighbouring reward lists, that only differ at $t^\star \in \{1, \dots, T\}$. Fix $\textbf{a} \defn (a_1, \dots, a_T) \in [K]^T$. We want to show that
    \[
    \mathcal{V}^\pi_\textbf{r} (\textbf{a}) \leq e^\epsilon  \mathcal{V}^\pi_\textbf{r'} (\textbf{a})~.
    \]

    \underline{Step 1: Probability decomposition and time-steps before $t^\star$:}
    \begin{align*}
      \frac{\mathcal{V}^\pi_\textbf{r}(\textbf{a})}{\mathcal{V}^\pi_\textbf{r'}(\textbf{a})} &=  \prod_{t = 1}^\horizon  \frac{\pol_t(a_t | a_1, r_1, \dots a_{t-1},r_{t-1})}{ \pol_t(a_t | a_1, r'_1, \dots a_{t-1},r'_{t-1})} \\
      &= \prod_{t = t^\star + 1}^\horizon \frac{  \pol_t(a_t | a_1, r_1, \dots a_{t-1},r_{t-1})}{  \pol_t(a_t | a_1, r'_1, \dots a_{t-1},r'_{t-1})},
    \end{align*}
    since for $t < t^\star$, $r_t = r'_t$. Let us denote by $\mathrm{Pr}(\textbf{a}^{> t^\star} \mid \textbf{a}^{\leq t^\star}, \textbf{r}) \defn \prod_{t = t^\star + 1}^\horizon  \pol_t(a_t | a_1, r_1, \dots a_{t-1},r_{t-1})$ the probability of the policy recommending the sequence $(a_{t^\star + 1}, \dots, a_T)$, when interacting with $\textbf{r} = \{r_1, \dots, r_T\}$ and already recommending $a_1, \dots, a_{t^\star}$ in the first steps.

    Let us denote by $t_1, \dots, t_\ell$ the time-steps of the beginning of the phases when $\pi$ interacts with $\textbf{r}$, and $t'_1, \dots, t'_{\ell'}$ the time-steps of the beginning of the phases when $\pi$ interacts with $\textbf{r'}$. Also, let $t_{k_\star}$ be the beginning of the phase for which $t^\star$ belongs in list $\textbf{r}$ phases. Similarly, let $t'_{k'_\star}$ be the beginning of the phase for which $t^\star$ belongs in list $\textbf{r'}$ phases.

    Since $(a_1, \dots, a_T)$ is fixed, and $r_t = r'_t$ for $t < t^\star$, then $t_{k_\star} = t'_{k'_\star}$ and $k^\star = k'_\star$, \ie, $t^\star$ falls at the same phase in $\textbf{r}$ and $\textbf{r'}$.

    \underline{Step 2: Considering the noisy sum of rewards at phase $k^\star$:}

    Let $\tilde{S}^p_{k^\star} = \sum_{s = t_{k^\star}}^{t_{k^\star + 1} - 1} r_s + Y_{k_\star}$ be the noisy partial sum of rewards collected at phase $k^\star$ for $\textbf{r}$, where $Y_{k^\star} \sim \mathrm{Lap}(1/\epsilon)$. Similarly, tet $\tilde{S'}^p_{k^\star} = \sum_{s = t_{k^\star}}^{t_{k^\star + 1} - 1} r'_s + Y'_{k_\star}$ be the noisy partial sum of rewards collected at phase $k^\star$ for $\textbf{r'}$, where $Y'_{k^\star} \sim \mathrm{Lap}(1/\epsilon)$.  We make two main observations:

    (a) If the value of the noisy partial sum at phase $k^\star$ is exactly the same between the neighbouring $\textbf{r}$ and $\textbf{r'}$, then the policy $\pi$ will recommend the sequence of actions $\textbf{a}^{> t^\star}$ with the same probability under $\textbf{r}$ and $\textbf{r'}$:
     \begin{equation}\label{eq:step1}
         \mathrm{Pr} (\textbf{a}^{> t^\star} \mid \textbf{a}^{\leq t^\star}, \textbf{r} , \tilde{S}^p_{k^\star} = s) =  \mathrm{Pr} (\textbf{a}^{> t^\star} \mid \textbf{a}^{\leq t^\star}, \textbf{r'} , \tilde{S'}^p_{k^\star} = s)~.
     \end{equation}
    This is due to the structure of the algorithm $\pi$, where the reward at step $t^\star$ only affects the statistic $\tilde{S}^p_{k^\star}$, and nothing else.

    (b) Since rewards are $[0,1]$, using the Laplace mechanism, we have that
    \begin{equation}\label{eq:step2}
        \mathrm{Pr}(\tilde{S}^p_{k^\star} = s \mid \textbf{a}^{\leq t^\star}, \textbf{r}) \leq e^\epsilon \mathrm{Pr}(\tilde{S'}^p_{k^\star} = s \mid \textbf{a}^{\leq t^\star}, \textbf{r'})~.
     \end{equation}

     \underline{Step 3: Combining Eq.~\ref{eq:step1} and Eq.~\ref{eq:step2}, aka post-processing:}

     We have
     \begin{align*}
         \mathrm{Pr} (\textbf{a}^{> t^\star} \mid \textbf{a}^{\leq t^\star}, \textbf{r}) &= \int_{s \in \real} \mathrm{Pr}(\tilde{S}^p_{k^\star} = s \mid \textbf{a}^{\leq t^\star}, \textbf{r}) \mathrm{Pr} (\textbf{a}^{> t^\star} \mid \textbf{a}^{\leq t^\star}, \textbf{r} , \tilde{S}^p_{k^\star} = s)\\
         &\leq \int_{s \in \real} e^\epsilon \mathrm{Pr}(\tilde{S'}^p_{k^\star} = s \mid \textbf{a}^{\leq t^\star}, \textbf{r'}) \mathrm{Pr} (\textbf{a}^{> t^\star} \mid \textbf{a}^{\leq t^\star}, \textbf{r'} , \tilde{S'}^p_{k^\star} = s)\\
         &= e^\epsilon \mathrm{Pr} (\textbf{a}^{> t^\star} \mid \textbf{a}^{\leq t^\star}, \textbf{r'})~.
     \end{align*}

     This concludes the proof:
     \begin{align*}
      \frac{\mathcal{V}^\pi_\textbf{r}(\textbf{a})}{\mathcal{V}^\pi_\textbf{r'}(\textbf{a})} = \frac{\mathrm{Pr} (\textbf{a}^{> t^\star} \mid \textbf{a}^{\leq t^\star}, \textbf{r})}{\mathrm{Pr} (\textbf{a}^{> t^\star} \mid \textbf{a}^{\leq t^\star}, \textbf{r'})} \leq e^\epsilon~.
    \end{align*}
 \end{proof}
\section{Regret Analysis Proof}\label{app:ub_proof}

\begin{lemma}[Explicit solution of $\dep$]
If $\mu,\mu' \in (0,1)$ and $\mu\le \mu'$, we have
    \begin{align}
\dep(\mu, \mu') \defn \inf_{z\in[\mu, \mu']}\left\{
\mathrm{kl}(z,\mu')+\epsilon (z-\mu)
\right\},
\end{align}
under Bernoulli cases, then \[
z^* = \max\left(\mu,  \frac{\mu' }{\mu'  + (1 - \mu')e^{\epsilon}}\right).
\]
solves the optimization problem. Thus, we have
\begin{equation}\label{eq:explicitSolu}
    \dep(\mu, \mu')=\left\{\begin{aligned}
        &\mathrm{kl}\left(\mu,\mu'\right), & \text{if} \quad \mu \ge  \frac{\mu' }{\mu' + (1 - \mu')e^{\epsilon}},\\
        &\mathrm{kl}\left(\frac{\mu' }{\mu' + (1 - \mu')e^{\epsilon}}, \mu'\right)+ \epsilon \left( \frac{\mu' }{\mu' + (1 - \mu')e^{\epsilon}}-\mu\right),   & \text{if} \quad \mu \le  \frac{\mu' }{\mu' + (1 - \mu')e^{\epsilon}}.
    \end{aligned}\right.
\end{equation}
For $\mu\ge \mu'$,
\[
z^* = \min\left(\frac{\mu' e^\epsilon}{\mu' e^\epsilon + (1-\mu')}, \mu\right).
\]
and
\begin{equation}\label{eq:expSolu2}
    \dep(\mu, \mu')=\left\{\begin{aligned}
        &\mathrm{kl}\left(\mu,\mu'\right), & \text{if} \quad \mu \le  \frac{\mu' e^\epsilon}{\mu' e^\epsilon + (1-\mu')},\\
        &\mathrm{kl}\left( \frac{\mu' e^\epsilon}{\mu' e^\epsilon + (1-\mu')}, \mu'\right)+ \epsilon \left(\mu- \frac{\mu' e^\epsilon}{\mu' e^\epsilon + (1-\mu')}\right),   & \text{if} \quad \mu \ge   \frac{\mu' e^\epsilon}{\mu' e^\epsilon + (1-\mu')}.
    \end{aligned}\right.
\end{equation}
\end{lemma}
\begin{proof}
The Kullback-Leibler divergence between two Bernoulli random variables with means \(z\) and \(\mu'\) is given by
\[
\mathrm{kl}(z, \mu') = z \log \frac{z}{\mu'} + (1 -z) \log \frac{1 - z}{1 - \mu'}~.
\]
% Given this, we proceed with solving the problem of:
% \[
% \dep(\mu, \mu') = \inf_{z\in [\mu, \mu']} \left\{\mathrm{kl}(z, \mu') + \epsilon (z - \mu)\right\}.
% \]
The optimization problem is
\[
\dep(\mu, \mu') = \inf_{z\in [\mu, \mu']} \left\{z \log \frac{z}{\mu'} + (1 -z) \log \frac{1 - z}{1 - \mu'} + \epsilon (z - \mu)\right\}~.
\]
To find the optimal \(z^*\), take the derivative of the objective function with respect to \(z\) and let it equal to $0$:
\[
\frac{\partial}{\partial z} \left[z \log \frac{z}{\mu'} + (1 -z) \log \frac{1 - z}{1 - \mu'} + \epsilon (z - \mu)\right] = 0~.
\]
By calculation, we have
\[
\log \frac{z (1 - \mu')}{\mu' (1 - z)} + \epsilon = 0~.
\]
% Exponentiating both sides:
% \[
% \frac{\mu'' (1 - \mu')}{\mu' (1 - \mu'')} = e^{-\epsilon}.
% \]
Rearrange for \(z\), to obtain
\[
z = \frac{\mu' }{\mu'  + (1 - \mu')e^\epsilon}~.
\]
The optimal \(z^*\) must lie within the interval \([\mu, \mu']\), hence we have
\[
z^* = \max\left(\mu, \min\left(\mu', \frac{\mu' }{\mu'  + (1 - \mu')e^\epsilon}\right)\right)~.
\]
We always have $\frac{\mu' }{\mu'  + (1 - \mu')e^\epsilon}\le\mu'$, so we can remove the min part:
\[
z^* = \max\left(\mu,  \frac{\mu' }{\mu' + (1 - \mu')e^{\epsilon}}\right)~.
\]
%\[\KL(\mu,\mu',\epsilon)=\KL(\mu''^*,\mu')+\epsilon (\mu''^*-\mu)\]
Thus, we obtain 
\begin{equation*}
    \dep(\mu, \mu')=\left\{\begin{aligned}
        &\mathrm{kl}\left(\mu,\mu'\right) & \text{if} \quad \mu \ge  \frac{\mu' }{\mu' + (1 - \mu')e^{\epsilon}}\\
        &\mathrm{kl}\left(\frac{\mu' }{\mu' + (1 - \mu')e^{\epsilon}}, \mu'\right)+ \epsilon \left( \frac{\mu' }{\mu' + (1 - \mu')e^{\epsilon}}-\mu\right)   & \text{if} \quad \mu \le  \frac{\mu' }{\mu' + (1 - \mu')e^{\epsilon}}
    \end{aligned}\right.
\end{equation*}

Now, we consider $\mu\ge \mu'$, 
\[
\dep(\mu, \mu') = \inf_{z\in[\mu', \mu]}\left\{
\mathrm{kl}(z,\mu')+\epsilon (\mu-z)\right\}~.
\]
So, we need to minimise
\[
\dep(\mu, \mu') = \inf_{z\in [\mu', \mu]}\left\{z \log \frac{z}{\mu'} + (1-z) \log \frac{1-z}{1-\mu'} + \epsilon (\mu - z)\right\}
\]
over \( z \in [\mu', \mu] \).
Differentiating the objective function with respect to \( z \) and setting it equal to $0$, we have
\[
 \log \frac{z}{\mu'} - \log \frac{1-z}{1-\mu'} - \epsilon=0.
\]
Solving for \( z \), we get
\[
z^* = \frac{\mu' e^\epsilon}{\mu' e^\epsilon + (1-\mu')}\ge \mu'~.
\]
Projecting the solution to $[\mu',\mu]$, then we have that the optimal solution is 
\[
z^* = \min\left(\frac{\mu' e^\epsilon}{\mu' e^\epsilon + (1-\mu')}, \mu\right)~.
\]
Thus, the explicit solution is
\begin{equation*}
    \dep(\mu, \mu')=\left\{\begin{aligned}
        &\mathrm{kl}\left(\mu,\mu'\right), & \text{if} \quad \mu \le  \frac{\mu' e^\epsilon}{\mu' e^\epsilon + (1-\mu')},\\
        &\mathrm{kl}\left( \frac{\mu' e^\epsilon}{\mu' e^\epsilon + (1-\mu')}, \mu'\right)+ \epsilon \left(\mu- \frac{\mu' e^\epsilon}{\mu' e^\epsilon + (1-\mu')}\right),   & \text{if} \quad \mu \ge   \frac{\mu' e^\epsilon}{\mu' e^\epsilon + (1-\mu')}.
    \end{aligned}\right.
\end{equation*}
\end{proof}

\begin{lemma}\label{lem_deriv_first}
For any $\mu,\mu'\in[0,1]$,
\begin{align}
\left|
\frac{\mathrm{~d} \{\mathrm{d}_\epsilon\left(\mu, \mu'\right)\}}{\mathrm{d} \mu}
\right|
\le \epsilon.\n
\end{align}
\end{lemma}
\begin{proof}
For $\mu\le \mu'$, from \eqref{eq:explicitSolu}, we have the explicit solution. If $\mu \ge  \frac{\mu' }{\mu' + (1 - \mu')e^{\epsilon}}$,
\[
\mathrm{d}_\epsilon\left(\mu, \mu'\right)=\mathrm{kl}(\mu, \mu') = \mu \log \frac{\mu}{\mu'} + (1 - \mu) \log \frac{1 - \mu}{1 - \mu'}
\]
Its derivative with respect to \( \mu \) is
\[
\frac{\mathrm{~d} \{\mathrm{d}_\epsilon\left(\mu, \mu'\right)\}}{\mathrm{d} \mu}=\frac{\mathrm{d} }{\mathrm{d} \mu} \mathrm{kl}(\mu, \mu') =
\log \frac{\mu(1 - \mu')}{\mu'(1 - \mu)}~.
\]
We have the condition
\[
\mu' \geq \mu \geq \frac{\mu'}{\mu' + (1 - \mu') e^\epsilon}~.
\]
Since \( \mu' \geq \mu \), we note that
\[
\frac{\mathrm{~d} \{\mathrm{d}_\epsilon\left(\mu, \mu'\right)\}}{\mathrm{d} \mu}\leq 0~.
\]
Similarly, since \( \mu \geq \frac{\mu'}{\mu' + (1 - \mu')e^\epsilon} \), we substitute this into the derivative
\[
\frac{\mu(1 - \mu')}{\mu'(1 - \mu)}
\geq \frac{\left(\frac{\mu'}{\mu' + (1 - \mu') e^\epsilon} \right)(1 - \mu')}{\mu' \left(1 - \frac{\mu'}{\mu' + (1 - \mu') e^\epsilon}\right)}= \frac{1}{e^\epsilon}~.
\]
% Simplify the denominator:
% \[
% 1 - \frac{\mu'}{\mu' + (1 - \mu') e^\epsilon}
% = \frac{(1 - \mu') e^\epsilon}{\mu' + (1 - \mu') e^\epsilon}.
% \]
Thus,
% \[
% \frac{\frac{\mu' (1 - \mu')}{\mu' + (1 - \mu') e^\epsilon}}{\mu' \cdot \frac{(1 - \mu') e^\epsilon}{\mu' + (1 - \mu') e^\epsilon}}
% = \frac{1}{e^\epsilon}.
% \]
\[
-\epsilon\le \log \frac{\mu(1 - \mu')}{\mu'(1 - \mu)} \le 0~.
\]
If $ \mu \le  \frac{\mu' }{\mu' + (1 - \mu')e^{\epsilon}}$, then
\[
\frac{\mathrm{~d} \{\mathrm{d}_\epsilon\left(\mu, \mu'\right)\}}{\mathrm{d} \mu}=-\epsilon~.
\]
Therefore, for  $\mu\le \mu'$,
\[
-\epsilon \le \frac{\mathrm{~d} \{\mathrm{d}_\epsilon\left(\mu, \mu'\right)\}}{\mathrm{d} \mu}\le 0~.
\]
% Thus, the lower bound of \( \frac{\mathrm{d} }{\mathrm{d} \mu} \mathrm{kl}(\mu, \mu') \) is:
% \[
% -\epsilon.
% \]
Now, we consider the case of $\mu\ge \mu'$. From the explicit solution \eqref{eq:expSolu2}, when $\mu \ge  \frac{\mu' e^\epsilon}{\mu' e^\epsilon + (1-\mu')}$, $ \frac{\mathrm{~d} \{\mathrm{d}_\epsilon\left(\mu, \mu'\right)\}}{\mathrm{d} \mu}=\epsilon$ and the result holds. Let's consider $\mu \le  \frac{\mu' e^\epsilon}{\mu' e^\epsilon + (1-\mu')}$, similar to the above argument, we have 
\[
\frac{\mathrm{~d} \{\mathrm{d}_\epsilon\left(\mu, \mu'\right)\}}{\mathrm{d} \mu}=\frac{\mathrm{d} }{\mathrm{d} \mu} \mathrm{kl}(\mu, \mu') =
\log \frac{\mu(1 - \mu')}{\mu'(1 - \mu)}\in [0,\epsilon].
\]
Thus, we have the result in the lemma.
\end{proof}

\begin{lemma}\label{lem_deriv_second}
For any $0\le \mu\le \mu'<1$,
\[
\frac{\mathrm{~d} \{\mathrm{d}_\epsilon\left(\mu, \mu'\right)\}}{\mathrm{d} \mu'}
\le
\frac{1}{1-\mu'}~.
\]
\end{lemma}
\begin{proof}
%\jhonda{Needs to be completed.}
Considering the definition of $\dep$ in \eqref{eq:d_eps}, we have for  $0\le \mu\le \mu'<1$
\[\dep(\mu,\mu')=\inf_{z\in[\mu,\mu']}\mathrm{kl}(z,\mu')+\epsilon(z-\mu)~.\]
We first show 
\begin{align}
\frac{\mathrm{~d} \{\mathrm{d}_\epsilon\left(\mu, \mu'\right)\}}{\mathrm{d} \mu'}
&\le
\frac{\mathrm{~d} \{\mathrm{kl}\left(\mu, \mu'\right)\}}{\mathrm{d} \mu'}~.
\end{align}
From the explicit solution in \eqref{eq:explicitSolu}, we have if \(\mu \ge  \frac{\mu' }{\mu' + (1 - \mu')e^{\epsilon}}\), then $ \dep(\mu, \mu')=\mathrm{kl}\left(\mu,\mu'\right)$. So the inequality holds. If \( \mu \le  \frac{\mu' }{\mu' + (1 - \mu')e^{\epsilon}}\), let $f(\mu')= \frac{\mu' }{\mu' + (1 - \mu')e^{\epsilon}}$, then $f'(\mu')= \frac{e^\epsilon }{(\mu' + (1 - \mu')e^{\epsilon})^2}$. In this case, $ \dep(\mu, \mu')=\mathrm{kl}\left(f(\mu'), \mu'\right)+ \epsilon \left( f(\mu')-\mu\right)$ where $\mu\le f(\mu')\le \mu'$. By calculation, we have for this case,
\begin{align*}
    \frac{\mathrm{~d} \{\mathrm{d}_\epsilon\left(\mu, \mu'\right)\}}{\mathrm{d} \mu'}=f'(\mu')\left(\log\frac{f(\mu')}{\mu'}-\log\frac{1-f(\mu')}{1-\mu'}+\epsilon\right)+\frac{\mu'-f(\mu')}{\mu'(1-\mu')}~.
\end{align*}
Note that $\log\frac{f(\mu')}{\mu'}-\log\frac{1-f(\mu')}{1-\mu'}+\epsilon=0$ and $\mu\le f(\mu')$.
And we bound
%Then let's bound the \(\frac{\mathrm{~d} \{\mathrm{kl}\left(\mu, \mu'\right)\}}{\mathrm{d} \mu'}\):
\begin{align}
\frac{\mathrm{~d} \{\mathrm{kl}\left(\mu, \mu'\right)\}}{\mathrm{d} \mu'}
&=
\frac{1-\mu}{1-\mu'}-\frac{\mu}{\mu'}
\nn
&=
\frac{1}{1-\mu'}\frac{\mu'-\mu}{\mu'}
\nn
&\le
\frac{1}{1-\mu'}~.\n
\end{align}
Thus, we have the result.
\end{proof}

\begin{theorem}[Regret upper bound of \dpimed{}]\label{thm:dp_imed}
Assume $\mustar<1$.
Under the batch sizes given in \eqref{def_batch} with $\base>1$,
the regret bound of \dpimed{} for a Bernoulli bandit $\nu$ is 
\label{thm:regretUpper}
    \[
    \reg_T(\dpimed{}, \nu) \le \sum_{i \neq i^*} \frac{\base \Delta_i\log T}{\mathrm{d}_\epsilon(\mu_i, \mustar)} +o(\log T)~.
    \]
\end{theorem}

\begin{proof}[Proof of Theorem \ref{thm:regretUpper}]
Let $\calT$ be the set of rounds $t$ such that Lines \ref{line_selection1}--\ref{line_selection2} are run, that is,
the rounds such that the arm selection occurred.
For $t\in \calT$, we define $\tmu_i(t)$ as $\tmu_{i,n_m}$ when $N_i(t-1)=n_m$.
Let $j$ be any optimal arm, that is, $j$ such that $\Delta_j=0$.
By the batched structure of the algorithm, we have
\begin{align}
\operatorname{Regret}(T)
& =\sum_{i\neq i^*}\sum_{t=1}^T \left(\mustar-\mu_i\right)\onex{i(t)=i}
\nn
& \le n_0\sum_{i\neq i^*}\left(\mustar-\mu_i\right)
+\sum_{i\neq i^*}\left(\mustar-\mu_i\right)
\sum_{t=1}^T \sum_{m=0}^{\infty}B_{m+1}\onex{i(t)=i,\,N_i(t-1)=n_m,\, t\in\calT}
\nn
&
\le
n_0\sum_{i\neq i^*}\left(\mustar-\mu_i\right)
\nn
&\quad+
\sum_{i\neq i^*}\left(\mustar-\mu_i\right)
\underbrace{\sum_{t=1}^T \sum_{m=0}^{\infty}B_{m+1}\onex{i(t)=i,\,N_i(t-1)=n_m,\, t\in\calT,\, \tmu_j(t)<\mustar-\de}}_{(\mathrm{A})}
\nn
&\quad+\sum_{i\neq i^*}\left(\mustar-\mu_i\right)
\underbrace{\sum_{t=1}^T \sum_{m=0}^{\infty}B_{m+1}\onex{i(t)=i,\,N_i(t-1)=n_m,\, t\in\calT,\, \tmu_j(t)\ge\mustar-\de}}_{(\mathrm{B})},
\label{eq:decomp}
\end{align}
where $\delta>0$ is a small constant. 
(A) and (B) correspond to the regret before and after the convergence, respectively.

Note that we have
\begin{align}
n_m
=\left\lceil n_0 \frac{\base^{m+1}-1}{\base-1}\right\rceil
\le n_0 \frac{\base^{m+1}-1}{\base-1}+1,
\label{nm}
\end{align}
and
\begin{align}
B_m=n_{m}-n_{m-1}
\le n_0 \frac{\base^{m+1}-1}{\base-1} - n_0 \frac{\base^{m}-1}{\base-1}+1
\le 2n_0 \base^m~.
\label{bound_batch}
\end{align}

%Note that we have
%\begin{align}
%n_m
%&=
%\sum_{i=0}^m B_m
%\nn
%&=
%n_0\frac{\base^{m+1}-1}{\base-1},
%\label{nm}
%\end{align}
%which is used in the evaluation of both terms.

\paragraph{Pre-convergence Term.}
First consider (A).
Define
\begin{align}
\bar{I}_j=
\max_{m: \tilde{\mu}_{j,m}< \mustar-\de} \left\{n_m \dep([\tilde{\mu}_{j,m}]_0^1, \mustar-\de)+\log n_m
\right\},
\label{def_barI}
\end{align}
where we define $\bar{I}_j=-\infty$ if $\tilde{\mu}_{j,m}\ge \mustar-\de$ for all $m\in\mathbb{Z}_+$.
Then,
$\left\{i(t)=i,\,t\in\calT,\, \tmu_j(t)< \mustar-\de\right\}$ implies that
\begin{align}
I_i(t) = I^*(t)\le I_j(t)\le
N_j(t-1) \dep([\tilde{\mu}_{j}(t)]_0^1, \mustar-\de)+\log N_j(t)
\le\bar{I}_j,
\n
\end{align}
where $I^*(t)=\max_{i'}I_i(t)$ is the optimal arm obtained by Line~\ref{line_compute_index}
%\ref{line_selection1}
in Algorithm \ref{alg_DP}.
By this fact we have
\begin{align}
(\mathrm{A})
&\le
\sum_{t=1}^T \sum_{m=0}^{\infty}B_{m+1}\onex{i(t)=i,\,N_i(t-1)=n_m,\, t\in\calT,\,I_i(t)\le \bar{I}_j}
\nn
&\le
\sum_{t=1}^T \sum_{m=0}^{\infty}B_{m+1}\onex{i(t)=i,\,N_i(t-1)=n_m,\, t\in\calT,\,\log N_i(t-1)\le \bar{I}_j}
\nn
&\le
\sum_{t=1}^T \sum_{m=0}^{\infty}B_{m+1}\onex{i(t)=i,\,N_i(t-1)=n_m,\, t\in\calT,\,n_m \le \e^{\bar{I}_j}}
\nn
&\le
\sum_{m=0}^{\infty}B_{m+1}\sum_{t=1}^T \onex{i(t)=i,\,N_i(t-1)=n_m,\, t\in\calT,\, n_0\frac{\base^m-1}{\base-1} \le \e^{\bar{I}_j}}
\nn
&=
\sum_{m=0}^{\left\lfloor \log_{\base}((\base-1)\e^{\bar{I}_j}/n_0+1)\right\rfloor}
B_{m+1}
\sum_{t=1}^T 
\onex{i(t)=i,\,N_i(t-1)=n_m,\, t\in\calT}~.
\n
\end{align}
Since $\{i(t)=i,\,N_i(t-1)=n_m\}$ can occur at most once for each $m$, we have
\begin{align}
(\mathrm{A})
&\le
\sum_{m=0}^{\left\lfloor \log_{\base}((\base-1)\e^{\bar{I}_j}/n_0+1)\right\rfloor}
B_{m+1}
\nn
&=
n_{\lfloor \log_{\base}((\base-1)\e^{\bar{I}_j}/n_0+1)\rfloor+1}-n_0
\nn
&=
\left\lceil
n_0 
\frac{\base^{\left\lfloor \log_{\base}((\base-1)\e^{\bar{I}_j}/n_0+1)\right\rfloor+2}-1}{\base-1}
\right\rceil
-n_0
\nn
&\le
n_0 
\frac{\base^2 ((\base-1)\e^{\bar{I}_j}/n_0+1)-1}{\base-1}
-n_0+1
\nn
&=
\base^2 \e^{\bar{I}_j}+\base n_0+1
\nn
&=
\base^2\max_{m: \tilde{\mu}_{j,m}< \mustar-\de} \left\{n_m\e^{n_m \dep([\tilde{\mu}_{j,m}]_0^1, \mustar-\de)}
\right\}
+\base n_0+1
\nn
&\le
\base^2
\sum_{m=0}^{\infty}
\onex{\tilde{\mu}_{j,m}< \mustar-\de}
n_m\e^{n_m \dep([\tilde{\mu}_{j,m}]_0^1, \mustar-\de)}
+\base n_0+1~.
\label{imed_term1}
\end{align}

Now, let us consider the expectation of \eqref{imed_term1}.
%Similar to \eqref{eq:Aterm}, 
When $\tilde{\mu}_{j, m}<\mustar-\de$
we have
\begin{align}
  \mathrm{d}_\epsilon([\tilde{\mu}_{j, m_j}]_0^1, \mustar-\delta)
&=\mathrm{d}_\epsilon([\tilde{\mu}_{j, m_j}]_0^1, \mustar)-
\left.\int_{\mustar-\delta}^{\mustar} \frac{\mathrm{~d} \{\mathrm{d}_\epsilon\left([\tilde{\mu}_{j, m_j}]_0^1, \mu\right)\}}{\mathrm{d} \mu}\right|_{\mu=u} \mathrm{~d} u\nn
&\ge
\mathrm{d}_\epsilon([\tilde{\mu}_{j, m_j}]_0^1, \mustar)-\frac{\de}{1-\mustar}
\since{by Lemma \ref{lem_deriv_second}}
\nn
  &=\mathrm{d}_\epsilon([\tilde{\mu}_{j, m_j}]_0^1, \mustar)-\delta',
\n
\end{align}
where we set $\delta'=\de/(1-\mustar)$.

Let \(P(x)=\Pr[\mathrm{d}_\epsilon([\tilde{\mu}_{j,m_j}]_0^1, \mustar) \ge x,\tmu_{j,m}<\mustar-\delta]\).
If $\tilde{\mu}_{j,m}< \mustar-\de$, then $0\le \mathrm{d}_\epsilon([\tilde{\mu}_{j,m_j}]_0^1, \mustar) \le d_1:= \mathrm{d}_\epsilon(0, \mustar)$. Hence, we have
\begin{align}
\mathbb{E}&\left[
\onex{\tilde{\mu}_{j,m}< \mustar-\de}
n_m\e^{n_m \dep([\tilde{\mu}_{j,m}]_0^1, \mustar-\de)}
\right]
\nn
&\le
\mathbb{E}\left[
\onex{\tilde{\mu}_{j,m}< \mustar-\de}
n_m\e^{n_m (\dep([\tilde{\mu}_{j,m}]_0^1, \mustar)-\de')}
\right]
\nn
&=\int_0^{d_1} n_m \e^{n_m (x-\delta')} \mathrm{~d}(-P( x))\nn
&=[n_m \e^{n_m (x-\delta')} (-P(x))]_{x=0}^{d_1}+\int_0^{d_1} n_m^2 \e^{n_m (x-\delta')}P(x) \mathrm{~d}x
\nn
&\le n_m \e^{-n_m \delta'}+\int_0^{d_1} n_m^2 \e^{n_m (x-\delta')}P(x) \mathrm{~d}x
\n~.
\end{align}

Let $c_x \in [0,\mustar]$ be such that $\mathrm{d}_\epsilon(c_x, \mustar)=x$.
Then
\begin{align}
\left\{\mathrm{d}_\epsilon([\tilde{\mu}_{j,m_j}]_0^1, \mustar) \ge x,\tmu_{j,m}<\mustar-\delta\right\}
&\Leftrightarrow
\left\{\tilde{\mu}_{j,m_j}<c_x,\tmu_{j,m}<\mustar-\delta\right\}~.
\n
\end{align}
Therefore,
\begin{align}
P(x)
=
\Pr[\tilde{\mu}_{j,m_j}<c_x,\tmu_{j,m}<\mustar-\delta]
\le \Aa\e^{n_m a}\e^{-n_{m}\mathrm{d}_\epsilon(c_x, \mustar)}
=\Aa\e^{n_m a}\e^{-n_m x},
\end{align}
for any $a>0$ by Corollary~\ref{corol}.
Thus, we have
\begin{align}
\mathbb{E}&\left[
\onex{\tilde{\mu}_{j,m}< \mustar-\de}
n_m\e^{n_m \dep([\tilde{\mu}_{j,m}]_0^1, \mustar-\de)}
\right]
\nn
&\le n_m \e^{-n_m \delta'}+
\int_0^{d_1} n_m^2 \e^{n_m (x-\delta')}\Aa\e^{n_m a}\e^{-n_m x} \mathrm{~d}x
\nn
&=
n_m \e^{-n_m \delta'}+
d_1 n_m^2 \Aa\e^{-n_m (\de'-a)}~.
\label{imed_wait}
\end{align}
By letting $a<\delta'$ and combining \eqref{imed_term1} with \eqref{imed_wait},
we obtain
\begin{align}
\mathbb{E}[(\mathrm{A})]
&\le
\base^2
\sum_{m=0}^{\infty}
\left(n_m \e^{-n_m \delta'}+
d_1 n_m^2 \Aa\e^{-n_m (\de'-a)}
\right)
+\base n_0+1
\nn
&\le
\base^2
\sum_{n=0}^{\infty}
\left(n \e^{-n \delta'}+
d_1 n^2 \Aa\e^{-n (\de'-a)}
\right)
+\base n_0+1
\nn
&=
\base^2
\left(
\frac{\e^{-(\de'-a)}}{(1-\e^{-(\de'-a)})^2}
+
d_1 \Aa
\frac{\e^{-(\de'-a)}(\e^{-(\de'-a)}+1)}{(1-\e^{-(\de'-a)})^3}
\right)
+\base n_0+1
\nn
&=
\base^2
\left(
\frac{\e^{-(\de'-a)}}{(1-\e^{-(\de'-a)})^2}
+
d_1 \Aa
\frac{\e^{-(\de'-a)}(\e^{-(\de'-a)}+1)}{(1-\e^{-(\de'-a)})^3}
\right)
+\base n_0+1
\nn
&=O(1)~.
%&\le
%\sum_{m=0}^{\infty} n_02^{m}\int_{x=1}^T d_1 A_a  x^2 e^{ax} e^{-x \delta'}  \mathrm{d}x\nn
%    \le& \int_{y=1}^{\log(T/n_0+1)-1}n_02^y \int_{x=1}^T d_1 A_a  x^2 e^{ax} e^{-x \delta'}  \mathrm{d}x \mathrm{d}y \nn
%     =&\frac{n_0 d_1 A_a}{a - \delta'} \left[ e^{T(a - \delta')} \left( T^2 - \frac{2T}{a - \delta'} + \frac{2}{(a - \delta')^2} \right) - e^{(a - \delta')} \left( 1 - \frac{2}{a - \delta'} + \frac{2}{(a - \delta')^2} \right) \right] \nn
%    &\cdot \frac{1}{\ln 2} \left( \frac{T/n_0 + 1}{2} - 2 \right) \le o(\log T).
\label{bound_pre}
\end{align}

\paragraph{Post-convergence Term}
Next we consider (B).
Since \(\mathrm{d}_\epsilon(\mu,\mu)=0\) for any $\mu\in[0,1]$, we have
\begin{align}
I^*(t) \leq \max _{i': \tilde{\mu}_{i'}(t)=\tilde{\mu}^*(t)} I_{i'}(t)=\max _{i': \tilde{\mu}_{i'}(t)=\tilde{\mu}^*(t)} \log N_{i'}(T) \leq \log T~.
\n
\end{align}
On the other hand,
$i(t)=i,\,N_i(t-1)=n_m,\, t\in\calT,\, \tmu_j(t)\ge\mustar-\de$ implies that
\begin{align*}
I^*(t)
=I_i(t)
\geq n_m {\mathrm{d}}_\epsilon\left([\tilde{\mu}_{i}(t)]_0^1, [\tilde{\mu}^*(t)]_0^1\right)
=n_m {\mathrm{d}}_\epsilon\left([\tilde{\mu}_{i}(t)]_0^1, \mustar-\de\right),
\end{align*}
from which we have
\begin{align}
\left\{i(t)=i,\,N_i(t-1)=n_m,\, t\in\calT,\, \tmu_j(t)\ge\mustar-\de\right\}
\subset
\left\{
n_m {\mathrm{d}}_\epsilon\left([\tilde{\mu}_{i}(t)]_0^1, \mustar-\de\right)\le\log T
\right\}~.
\n
\end{align}
So, we have
\begin{align}
(\mathrm{B})
&=
\sum_{t=1}^T \sum_{m=0}^{\infty}B_{m+1}\onex{i(t)=i,\,N_i(t-1)=n_m,\, t\in\calT,\, \tmu_j(t)\ge\mustar-\de}
\nn
&\le
\sum_{t=1}^T \sum_{m=0}^{\infty}B_{m+1}\onex{i(t)=i,\,N_i(t-1)=n_m,\,
n_m {\mathrm{d}}_\epsilon\left([\tilde{\mu}_{i}(t)]_0^1, \mustar-\de\right)\le\log T}
\nn
&=
\sum_{m=0}^{\infty}B_{m+1}
\onex{n_m {\mathrm{d}}_\epsilon\left([\tilde{\mu}_{i,n_m}]_0^1, \mustar-\de\right)\le\log T}
\sum_{t=1}^T
\onex{i(t)=i,\,N_i(t-1)=n_m}
\nn
&\le
\sum_{m=0}^{\infty}B_{m+1}
\onex{n_m {\mathrm{d}}_\epsilon\left([\tilde{\mu}_{i,n_m}]_0^1, \mustar-\de\right)\le\log T}
\label{post_intermediate}\\
&\le
\sum_{m=0}^{\infty}B_{m+1}
\onex{n_m ({\mathrm{d}}_\epsilon([\tilde{\mu}_{i,n_m}]_0^1, \mustar)-\delta')\le\log T},
\n
\end{align}
where recall that $\delta'=\delta/(1-\mustar)$ and the last inequality follows from Lemma~\ref{lem_deriv_second}.

Let
\begin{align}
H=\frac{\log T}{\mathrm{d}_\epsilon(\mu_i, \mustar)-2\de'}
\label{def_H}
\end{align}
and
\begin{align}
m^*=
\inf\{m \in \mathbb{N}:n_m \ge H\}~.\n
\end{align}
Then, by $n_{m^*-1}<H$ and \eqref{nm}
we have
\begin{align}
H>n_{m^*-1}=
\left\lceil n_0 \frac{\base^{m^*}-1}{\base-1}\right\rceil
\ge
n_0 \frac{\base^{m^*}-1}{\base-1},\n
\end{align}
which implies
\begin{align}
m^*\le
\log_{\base}\left(\frac{(\base-1)H}{n_0}+1\right)~.\label{bound_mstar}
\end{align}

Now, the post-convergence term can be bounded as follows:
\begin{align}
\mathbb{E}[(\mathrm{B})]
&\le
\sum_{m=0}^{\infty}B_{m+1}
\Pr\left[n_m ({\mathrm{d}}_\epsilon([\tilde{\mu}_{i,n_m}]_0^1, \mustar)-\delta')\le\log T\right]
\nn
&\le
\sum_{m=0}^{m^*-1}
B_{m+1}
+
\sum_{m=m^*}^{\infty}B_{m+1}
\Pr\left[n_m ({\mathrm{d}}_\epsilon([\tilde{\mu}_{i,n_m}]_0^1, \mustar)-\delta')\le\log T\right]
\nn
&\le
n_{m^*}-n_0
+
\sum_{m=m^*}^{\infty}B_{m+1}
\Pr\left[H \left(\mathrm{d}_\epsilon([\tilde{\mu}_{i, m}]_0^1, \mustar)-\de'\right)\le \log T\right]
\nn
&<
n_0 \frac{\base^{m^*+1}-1}{\base-1}
+1-n_0
+\sum_{m=m^*}^{\infty} B_{m+1}\Pr\left[H \left(\mathrm{d}_\epsilon([\tilde{\mu}_{i, m}]_0^1, \mustar)-\de'\right)\le \log T\right]\nn
&\le
n_0 \frac{\base \left(\frac{(\base-1)H}{n_0}+1\right)-1}{\base-1}
+1-n_0
+\sum_{m=m^*}^{\infty} B_{m+1} \Pr\left[\mathrm{d}_\epsilon([\tilde{\mu}_{i, m}]_0^1, \mustar)\le \mathrm{d}_\epsilon(\mu_{i}, \mustar) -\de'\right]
\nn
&\phantom{wwwwwwwwwwwwwwwwwwwwwwwwwwwwwwwwwwwww}\since{by \eqref{def_H} and \eqref{bound_mstar}}
\nn
&\le \base H+1-\frac{n_0 \base}{\base-1}
+\sum_{m=m^*}^{\infty} B_{m+1} \Pr\left[\tilde{\mu}_{i, m} \ge \mu_i+\de'/\epsilon\right]
\phantom{wwwwi}\since{by Lemma~\ref{lem_deriv_first}}\nn
%\label{eq:Convex}\\
& \le \base H+\sum_{m=m^*}^{\infty} B_{m+1} \Aap\e^{a'n_m}\e^{-n_m(\mathrm{d}_\epsilon(\mu_{i}+\de'/\eps, \mustar) )}
\phantom{wwwwwwwwa}\since{by Corollary~\ref{corol}}
\nn
& \le
\base H+
\sum_{m=m^*}^{\infty} 2n_0\base^{m+1} \Aap\e^{-n_0\frac{\base^{m+1}-1}{\base-1}(\mathrm{d}_\epsilon(\mu_{i}+\de'/\eps, \mustar)-a')}
\phantom{wwwwi}\since{by \eqref{bound_batch}}
\label{ref_Lambda1}
\\
&=
\base H+
\Aap\e^{\Lambda}
\sum_{m=m^*}^{\infty} 2n_0\base^{m+1}\e^{-\base ^{m+1}\Lambda}
\label{ref_Lambda2}
\\
&\le
\base H+
2n_0\Aap\e^{\Lambda}
\int_{m^*}^{\infty}\base^{x+1}\e^{-\base ^{x}\Lambda}\mathrm{~d} x
\nn
&=\base H+
\frac{2\base n_0 \Aap\e^{\Lambda}}{\mathrm{ln}(\base) \Lambda}
e^{-\base^{m^*}\Lambda}
\nn
&=
\frac{\base \log T}{\mathrm{d}_\epsilon(\mu_i, \mustar)-2\de'}+o(1)~.
\label{bound_post}
\end{align}
Here, in \eqref{ref_Lambda1} we took $a'<\mathrm{d}_\epsilon(\mu_{i}+\de'/\eps, \mustar)$
and
in \eqref{ref_Lambda2} we defined
\begin{align}
\Lambda=
\frac{n_0(\mathrm{d}_\epsilon(\mu_{i}+\de'/\eps, \mustar)-a')}{\base-1}~.
\n
\end{align}
We complete the proof
by combining \eqref{eq:decomp}, \eqref{bound_pre}, and \eqref{bound_post}, and letting $\delta'=\frac{\delta}{1-\mustar} \downarrow 0$.
\end{proof}

\begin{theorem}[Regret upper bound of \dpklucb{}]
\label{thm:regretUpper_klucb}
Assume $\mustar<1$.
Under the batch sizes given in \eqref{def_batch} with $\base>1$,
the regret bound of \dpklucb{} for a Bernoulli bandit $\nu$ is 
\[
\reg_T(\dpklucb{}, \nu)\le \sum_{i \neq i^*} \frac{\base \Delta_i\log T}{\mathrm{d}_\epsilon(\mu_i, \mustar)} +o(\log T)~.
\]
\end{theorem}

\begin{proof}[Proof of Theorem \ref{thm:regretUpper_klucb}]
By the same argument as the analysis for DP-IMED
we have
\begin{align}
\operatorname{Regret}(T)
&
\le
n_0\sum_{i\neq i^*}\left(\mustar-\mu_i\right)
\nn
&\quad+
\sum_{i\neq i^*}\left(\mustar-\mu_i\right)
\underbrace{\sum_{t=1}^T \sum_{m=0}^{\infty}B_{m+1}\onex{i(t)=i,\,N_i(t-1)=n_m,\, t\in\calT,\, \bar{\mu}^{\star}(t)<\mustar-\de}}_{(\mathrm{A})}
\nn
&\quad+\sum_{i\neq i^*}\left(\mustar-\mu_i\right)
\underbrace{\sum_{t=1}^T \sum_{m=0}^{\infty}B_{m+1}\onex{i(t)=i,\,N_i(t-1)=n_m,\, t\in\calT,\, \bar{\mu}^{\star}(t)\ge\mustar-\de}}_{(\mathrm{B})},
\label{eq:decomp2}
\end{align}
where $\bar{\mu}^{\star}(t)=\max_i \bar{\mu}_i(t)$.

We use a transformation of these terms that is similar to \citet{honda2019note} but more suitable for the batched algorithm.
First, we have
\begin{align}
(\mathrm{A})
&=
\sum_{t=1}^T \sum_{m=0}^{\infty}B_{m+1}\onex{i(t)=i,\,N_i(t-1)=n_m,\, t\in\calT,\, \bar{\mu}^{\star}(t)<\mustar-\de}
\nn
&\le
\sum_{t=1}^T \sum_{m=0}^{\infty}B_{m+1}\onex{i(t)=i,\,N_i(t-1)=n_m,\, t\in\calT,\, \bar{\mu}_j(t)<\mustar-\de}.
\n
\end{align}
Let
\begin{align}
\bar{I}_j'=
\max_{m: \tilde{\mu}_{j,m}< \mustar-\de} \left\{n_m \dep([\tilde{\mu}_{j,m}]_0^1, \mustar-\de)
\right\}.\n
\end{align}
Since
\begin{align}
\{\bar{\mu}_j(t)< \mustar-\de\}
&\Leftrightarrow
\left\{
\sup\left\{
\mu: \dep([\tmu_{j}(t)]_0^1, \mu)\le \frac{\log t}{N_j(t-1)}
\right\}
<\mustar-\de
\right\}
\nn
&\Rightarrow
\left\{
\dep([\tmu_{j}(t)]_0^1, \mustar-\de)> \frac{\log t}{N_j(t-1)}
,\,
\tmu_{j}(t)<\mustar-\de
\right\}
\nn
&\Leftrightarrow
\left\{
t< \e^{N_{j}(t-1)\dep([\tmu_{j}(t)]_0^1, \mustar-\de)},\,
\tmu_{j}(t)<\mustar-\de
\right\},
\n
\end{align}
we see that
\begin{align}
(\mathrm{A})
&\le
\sum_{t=1}^T \sum_{m=0}^{\infty}B_{m+1}\onex{i(t)=i,\,N_i(t-1)=n_m,\, t\in\calT,\,
t< \e^{N_j(t-1)\dep([\tmu_{j,m}]_0^1, \mustar-\de)},\,
\tmu_{j,m}<\mustar-\de
}
\nn
&\le
\sum_{t=1}^T \sum_{m=0}^{\infty}B_{m+1}\onex{i(t)=i,\,N_i(t-1)=n_m,\,
t< \e^{\bar{I}_j'}}
\nn
&\le
\sum_{t=1}^T \sum_{m=0}^{\infty}B_{m+1}\onex{i(t)=i,\,N_i(t-1)=n_m,\,
n_m< \e^{\bar{I}_j'}-1}
\since{by $N_i(t-1)\le t-1$}
\nn
&=
\sum_{m=0}^{\infty}B_{m+1}
\onex{n_m< \e^{\bar{I}_j'}-1}
\sum_{t=1}^T
\onex{i(t)=i,\,N_i(t-1)=n_m}
\nn
&\le
\sum_{m=0}^{\infty}B_{m+1}
\onex{n_m< \e^{\bar{I}_j'}-1}
\nn
&\le
(\base+1) \e^{\bar{I}_j'}
\phantom{www}\since{by $n_m=\sum_{i=0}^m B_{m}$ and $B_{m+1}\le \base n_m$}
\nn
&\le
(\base+1) \e^{\bar{I}_j},
\n
\end{align}
where $\bar{I}_j$ is defined in \eqref{def_barI}.
The evaluation of this expectation is the one same as \eqref{imed_term1}, which results in $\E[(\mathrm{A})]=O(1)$.

Now, we consider the second term.
Note that $i(t)=i$ implies $\bar{\mu}^{\star}(t)=\bar{\mu}_i(t)$
and
we also have
\begin{align}
\{\bar{\mu}_i(t)\ge \mustar-\de\}
&\Leftrightarrow
\left\{
\sup\left\{
\mu: \dep([\tmu_{j}(t)]_0^1, \mu)\le \frac{\log t}{N_{j}(t)}
\right\}
\ge \mustar-\de
\right\}
\nn
&\Rightarrow
\left\{
\dep([\tmu_{i}(t)]_0^1, \mustar-\de)\le \frac{\log t}{N_{i}(t)}
\right\}.
\n
\end{align}
Then, we have
\begin{align}
(\mathrm{B})
&=
\sum_{t=1}^T \sum_{m=0}^{\infty}B_{m+1}\onex{i(t)=i,\,N_i(t-1)=n_m,\, t\in\calT,\, \bar{\mu}^{\star}(t)\ge\mustar-\de}
\nn
&=
\sum_{t=1}^T \sum_{m=0}^{\infty}B_{m+1}\onex{i(t)=i,\,N_i(t-1)=n_m,\, t\in\calT,\, \bar{\mu}_i(t)\ge\mustar-\de}
\nn
&\le
\sum_{t=1}^T \sum_{m=0}^{\infty}B_{m+1}\onex{i(t)=i,\,N_i(t-1)=n_m,\, t\in\calT,\,
\dep([\tmu_{i,n_m}]_0^1, \mustar-\de)\le \frac{\log t}{n_m}
}
\nn
&\le
\sum_{m=0}^{\infty}B_{m+1}
\onex{
\dep([\tmu_{i,n_m}]_0^1, \mustar-\de)\le \frac{\log t}{n_m}
}
\sum_{t=1}^T
\onex{i(t)=i,\,N_i(t-1)=n_m}
\nn
&\le
\sum_{m=0}^{\infty}B_{m+1}
\onex{
\dep([\tmu_{i,n_m}]_0^1, \mustar-\de)\le \frac{\log t}{n_m}
},\n
\end{align}
whose expectation is analysed in \eqref{post_intermediate}.
\end{proof}

\paragraph{Comparison to the regret bound of AdaP-KLUCB in~\cite{azize2022privacy}}
Theorem 8 in~\cite{azize2022privacy} shows that for $\tau>3$,  AdaP-KLUCB yields a regret
\begin{equation}\label{eq:adap_klucb}
   \reg_{\horizon}(\adapklucb{}, \nu) \leq \sum_{a: \Delta_a > 0}\left ( \frac{C_1(\tau) \Delta_a }{\min\{ \mathrm{kl}(\mu_a, \mu^*) , C_2 \epsilon \Delta_a\}} \log(\horizon) + \frac{3 \tau}{\tau - 3} \right )\,, 
\end{equation}
where $C_1(\tau)$ and $C_2>0$ are defined as
$$\inf_{\beta \in \mathbf{B} } \max \left \{ \frac{(1+\beta)\alpha }{\mathrm{kl}(\mu_a, \mu^*)} , \frac{(1 + \tau) }{(c(\beta) - \gamma_{\episode, \horizon} ) \epsilon \Delta_a} \right \} \log(\horizon) \defn \frac{\frac{1}{4} C_1(\tau)}{\min \{\mathrm{kl}(\mu_a, \mu^*), C_2 \epsilon \Delta_a \}} \log(\horizon),$$ such that $\tau$ is a constant that controls the optimism in AdaP-KLUCB, $\mathbf{B} \defn \{\beta>0: c(\beta) > \gamma_{\episode, \horizon} \}$, for $\beta > 0$,  $c(\beta) \in [0,1]$ is defined such that:
$\mathrm{kl}(\mu_a + c(\beta) \Delta_a, \mu^* ) = \frac{d(\mu_a, \mu^*)}{1+\beta}$, and $\gamma_{\episode, \horizon}$ such that $\mathrm{kl}(\mu_a + \gamma_{\episode, \horizon} \Delta_a, \mu_a ) = \frac{ \log(\horizon)}{2^\episode}$ for $T$ the horizon and $\ell$ the phase. 

In general, $C_1$ and $C_2$ may depend on $\mu_a$ and $\mu^\star$, and thus are not ``constants". 
In contrast, our bound in Theorem~\ref{thm:upp_bound}
matches the asymptotic lower bound of Theorem~\ref{thm:low_bound} up to the exact constant $\alpha >1$ that controls the geometrically increasing batches and which can be chosen arbitrarily close to 1. In addition, our analysis only requires that the number of batches is sublinear in $T$ as can be seen from Proposition~\ref{prop:conc}. As a result, we can also use a polynomially increasing batch size instead of $B_m\approx \alpha^m$, which fully makes the regret \emph{asymptotically optimal}. We used a geometrically increasing batch size here just for simplicity.

% Note that $\alpha$ in AdaP-KLUCB and in \dpklucb{}/\dpimed{} means different things.
% \jhonda{beginning of the new version}
% \jhonda{Concentration inequality in a form of general interest.
% Proof needs modification for this form.}
\section{Extended Experiments}\label{app:extended_exps}
In this section, we present additional experiments comparing the algorithms in four bandit environments with Bernoulli distributions, as defined by~\cite{dpseOrSheffet}, namely
\begin{align*}
&\mu_1 = \{0.75,0.70,0.70,0.70,0.70  \},~~~\mu_2 = \{0.75,0.625,0.5,0.375,0.25  \},\\
&\mu_3 = \{0.75,0.53125,0.375,0.28125,0.25  \},~~~\mu_4 = \{0.75,0.71875,0.625,0.46875,0.25  \}.
\end{align*}
and four budgets $\epsilon \in \{0.01, 0.1, 0.5,1\}$. The results are presented in Figure~\ref{fig:exp1} for $\mu_1$, Figure~\ref{fig:exp2} for $\mu_2$, Figure~\ref{fig:exp3} for $\mu_3$ and Figure~\ref{fig:exp4} for $\mu_4$. 

For all the environments and privacy budgets tested, \dpimed{} and \dpklucb{} achieve the lowest regret.

We also plot the regret as a function of the privacy budget $\epsilon$ in Figure~\ref{fig:priv_reg}. The algorithm chosen is \dpimed{} with $\alpha = 1.1$, $T = 10^7$ and for bandit environment $\mu = [0.8, 0.1, 0.1, 0.1, 0.1]$. We discretise the $[0,1]$ interval into 100 values of $\epsilon$. For each $\epsilon$, we run the algorithm $20$ times and plot the mean and standard deviation of the regret in $[0,1]$. We also plot the asymptotic regret lower bound in Figure~\ref{fig:priv_reg} for $T= 10^7$ and $\mu$ as a function of $\epsilon$. 

The performance of our algorithm $\dpimed{}$ matches the regret lower bound. We also remark that the change between the high and the low privacy regimes happens smoothly.

% \begin{figure}[t!]
%     \centering
%     \begin{minipage}{0.4\textwidth}
%     \centering
%         \includegraphics[width=\linewidth]{figures/list.pdf}
%         {(a) List of rewards}
%     \end{minipage}\hfill
%     \begin{minipage}{0.6\linewidth}
%     \centering
%         \includegraphics[width=\linewidth]{figures/table.pdf}\\
%         {(b) Table of rewards}
%     \end{minipage}\vfill
%     \begin{minipage}{\textwidth}
%     \centering
%         \includegraphics[width=\linewidth]{figures/tree.pdf}
%         {(c) Tree of rewards}
%         %\label{fig:tree_rew}
%     \end{minipage}
%     \caption{Different reward representations for $T = 3$ and $K = 2$. The highlighted rewards are the rewards observed by the policy for the trajectory $(a_1, a_2, a_3) = (1,2,1)$}\label{fig:inp_rep}
% \end{figure}

\begin{figure}[thb]
    \centering
    \begin{minipage}{0.475\textwidth}
    \centering
        \includegraphics[width=\textwidth]{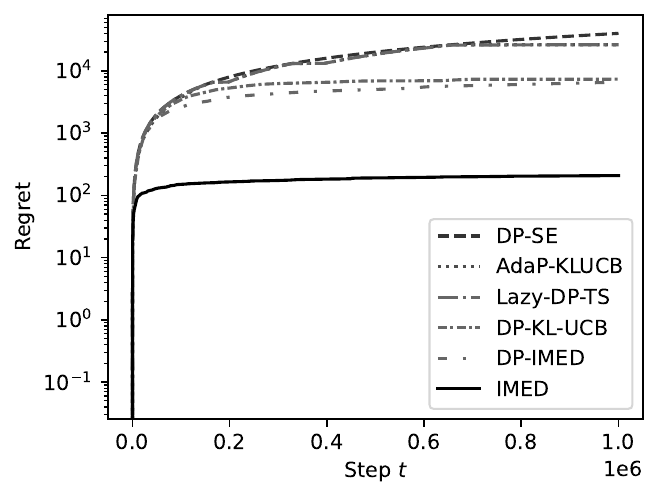}
        {(a) $\epsilon = 0.01$}    
    \end{minipage}
    \hfill
    \begin{minipage}{0.475\textwidth} 
    \centering
        \includegraphics[width=\textwidth]{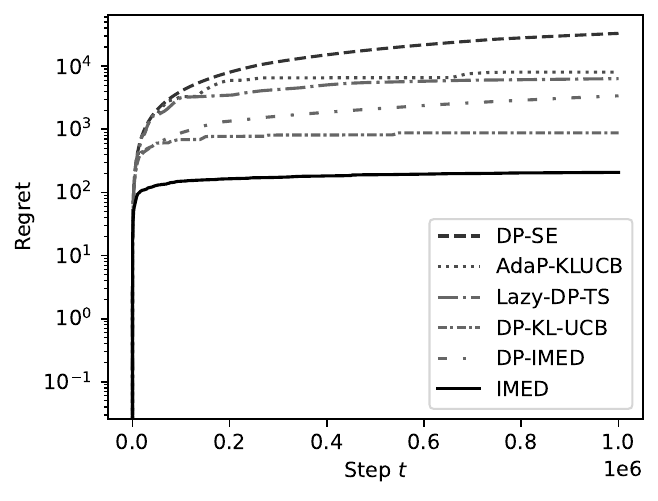}
        {(b) $\epsilon = 0.1$}   
    \end{minipage}
    \vskip\baselineskip
    \begin{minipage}{0.475\textwidth} 
    \centering
        \includegraphics[width=\textwidth]{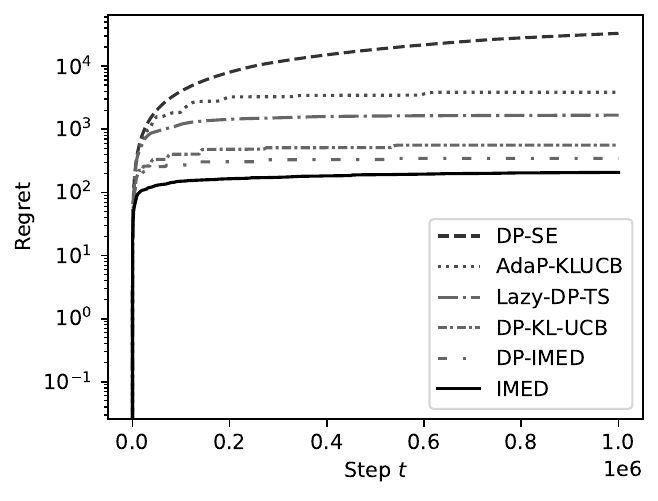}
        {(c) $\epsilon = 0.5$}    
    \end{minipage}
    \hfill
    \begin{minipage}{0.475\textwidth}  
    \centering
        \includegraphics[width=\textwidth]{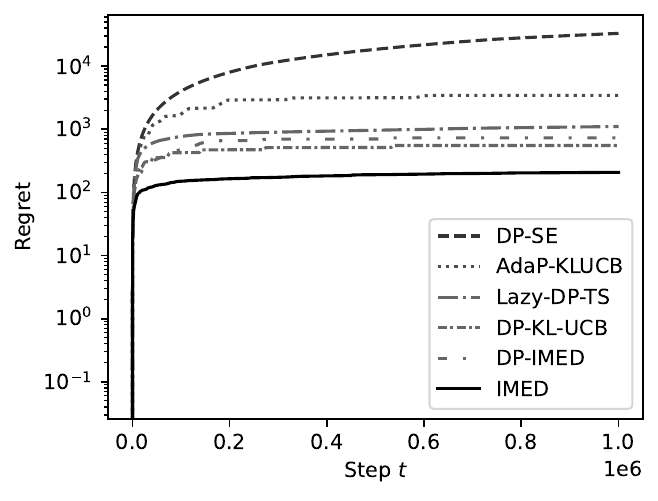}
        {(d) $\epsilon = 1$}   
        
    \end{minipage}
    \caption{Evolution of regret over time for $\mu_1$ for different budgets $\epsilon$.} 
    \label{fig:exp1}
\end{figure}

\begin{figure}[thb]
    \centering
    \begin{minipage}{0.475\textwidth}
    \centering
        \includegraphics[width=\textwidth]{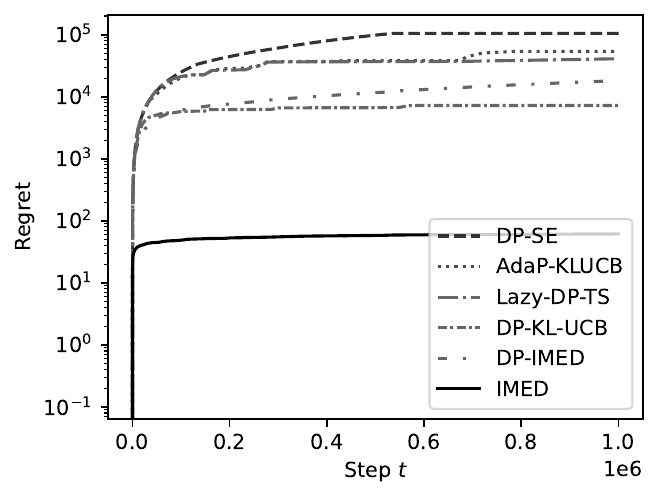}
        {(a) $\epsilon = 0.01$}    
    \end{minipage}
    \hfill
    \begin{minipage}{0.475\textwidth} 
    \centering
        \includegraphics[width=\textwidth]{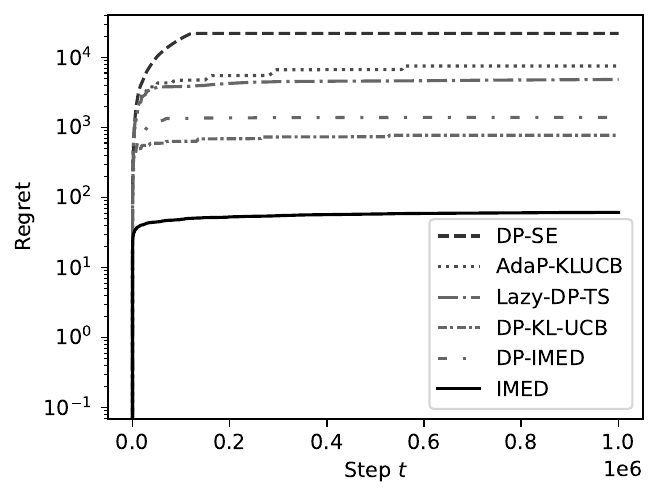}
        {(b) $\epsilon = 0.1$}   
    \end{minipage}
    \vskip\baselineskip
    \begin{minipage}{0.475\textwidth} 
    \centering
        \includegraphics[width=\textwidth]{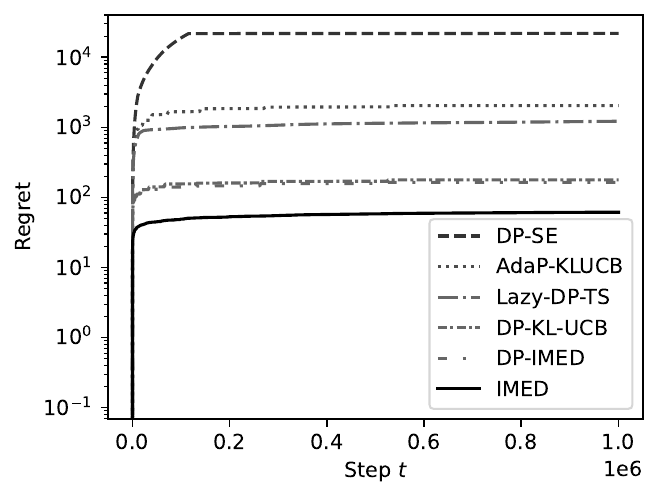}
        {(c) $\epsilon = 0.5$}    
    \end{minipage}
    \hfill
    \begin{minipage}{0.475\textwidth}  
    \centering
        \includegraphics[width=\textwidth]{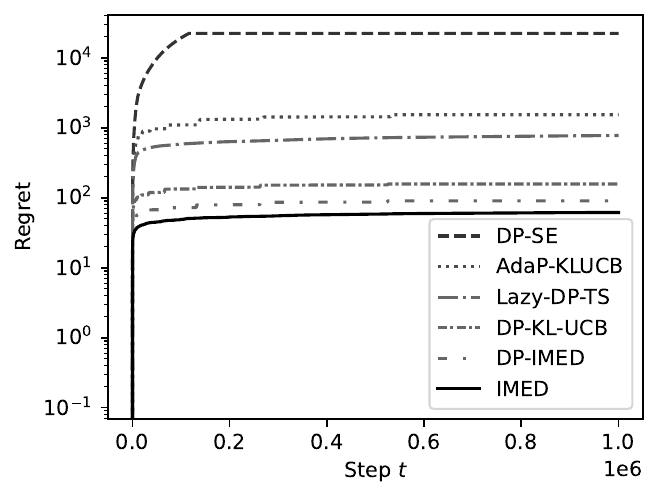}
        {(d) $\epsilon = 1$}   
        
    \end{minipage}
    \caption{Evolution of regret over time for $\mu_2$ for different budgets $\epsilon$.} 
    \label{fig:exp2}
\end{figure}

\begin{figure}[thb]
    \centering
    \begin{minipage}{0.475\textwidth}
    \centering
        \includegraphics[width=\textwidth]{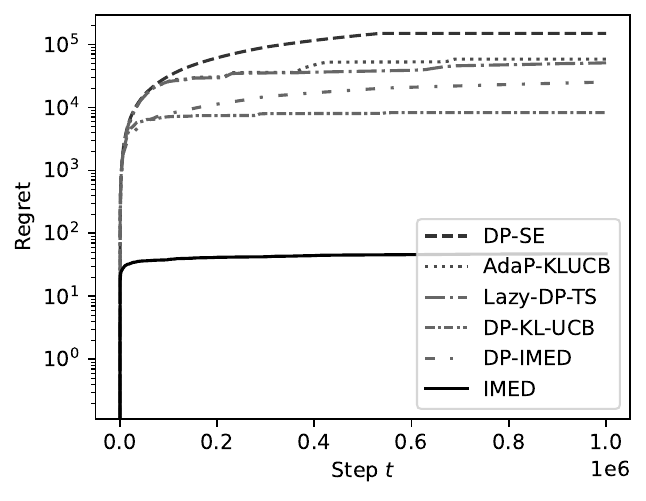}
        {(a) $\epsilon = 0.01$}    
    \end{minipage}
    \hfill
    \begin{minipage}{0.475\textwidth} 
    \centering
        \includegraphics[width=\textwidth]{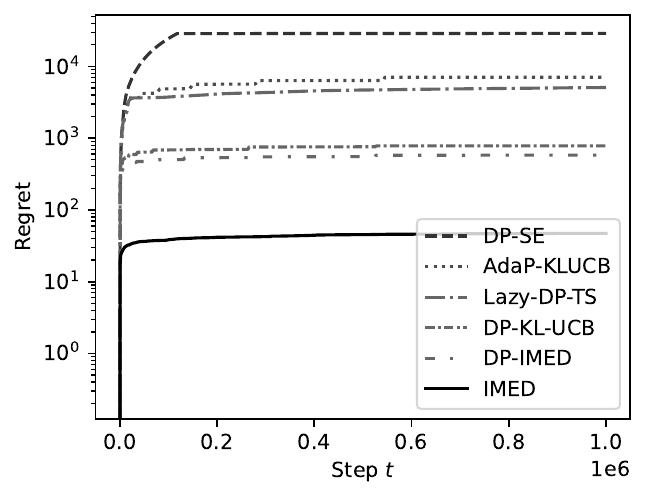}
        {(b) $\epsilon = 0.1$}   
    \end{minipage}
    \vskip\baselineskip
    \begin{minipage}{0.475\textwidth} 
    \centering
        \includegraphics[width=\textwidth]{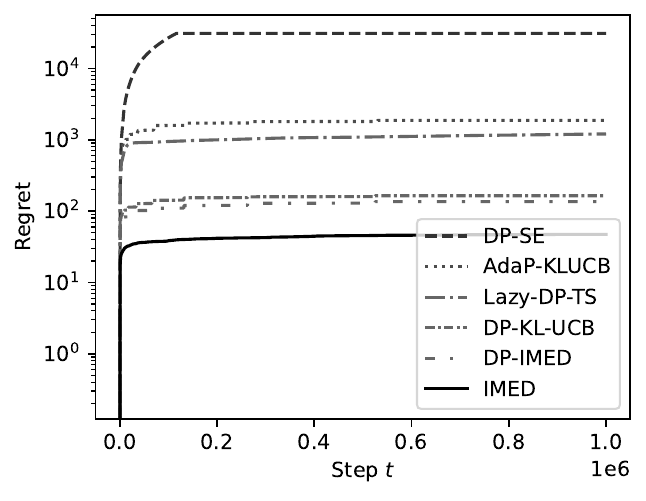}
        {(c) $\epsilon = 0.5$}    
    \end{minipage}
    \hfill
    \begin{minipage}{0.475\textwidth}  
    \centering
        \includegraphics[width=\textwidth]{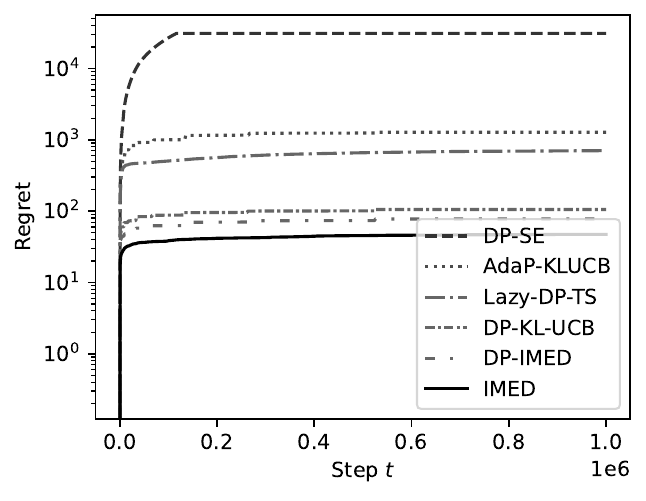}
        {(d) $\epsilon = 1$}   
    \end{minipage}
    \caption{Evolution of regret over time for $\mu_3$ for different budgets $\epsilon$.} 
    \label{fig:exp3}
\end{figure}

\begin{figure}[t!]
    \centering
    \begin{minipage}{0.475\textwidth}
    \centering
        \includegraphics[width=\textwidth]{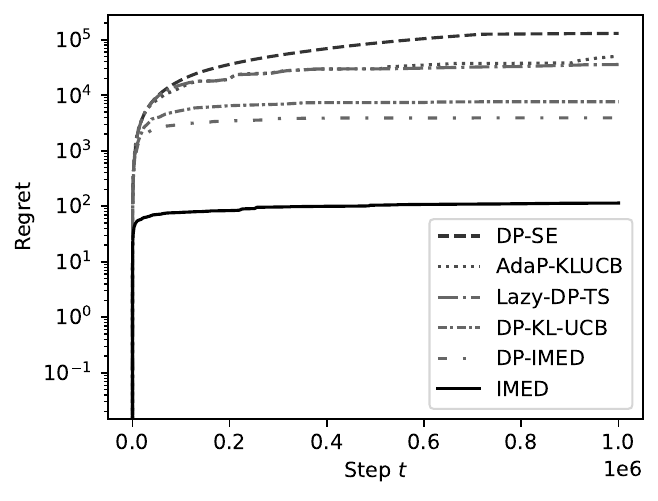}
        {(a) $\epsilon = 0.01$}    
    \end{minipage}
    \hfill
    \begin{minipage}{0.475\textwidth} 
    \centering
        \includegraphics[width=\textwidth]{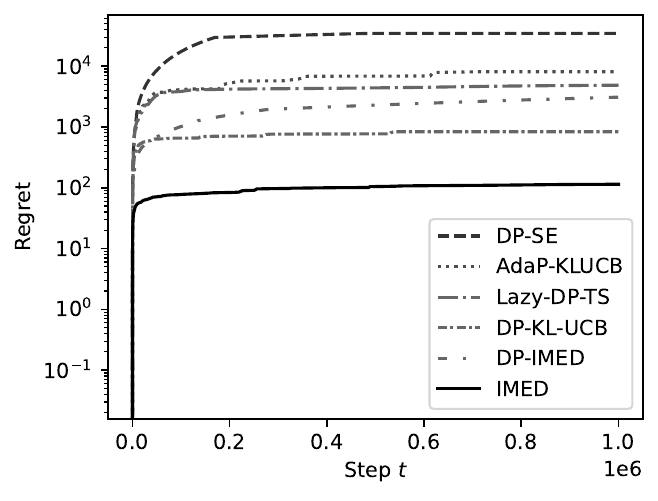}
        {(b) $\epsilon = 0.1$}   
    \end{minipage}
    \vskip\baselineskip
    \begin{minipage}{0.475\textwidth} 
    \centering
        \includegraphics[width=\textwidth]{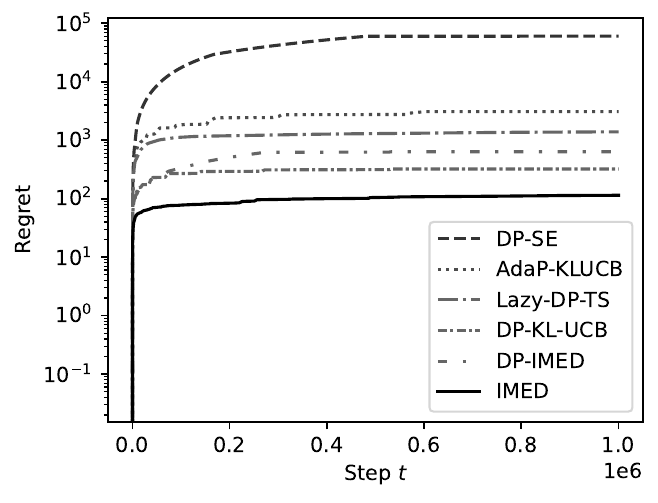}
        {(c) $\epsilon = 0.5$}    
    \end{minipage}
    \hfill
    \begin{minipage}{0.475\textwidth}  
    \centering
        \includegraphics[width=\textwidth]{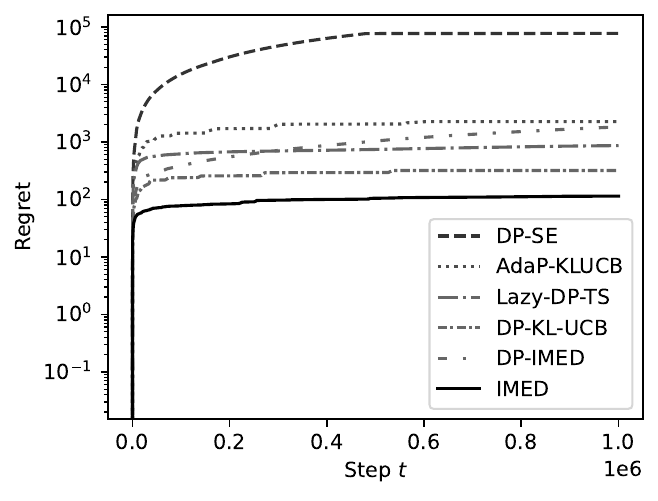}
        {(d) $\epsilon = 1$}   
    \end{minipage}
    \caption{Evolution of regret over time for $\mu_4$ for different budgets $\epsilon$.} 
    \label{fig:exp4}
\end{figure}

\begin{figure}
    \centering
    \includegraphics[width=0.8\linewidth]{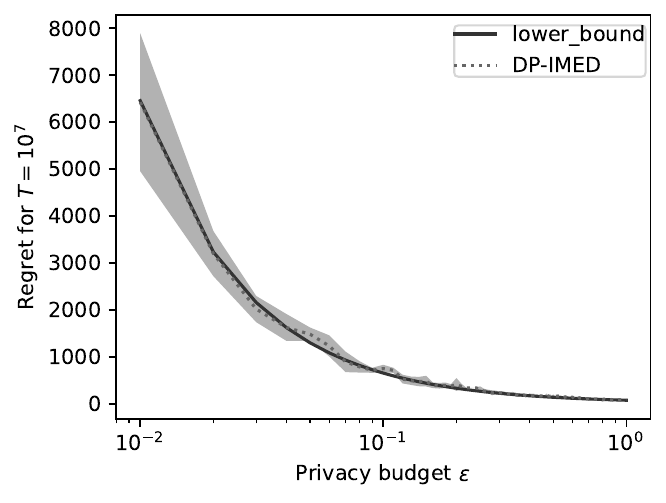}
    \caption{Evolution of the regret for $T = 10^7$ with respect to $\epsilon$ for \dpimed{} on $\mu \defn [0.8, 0.1, 0.1, 0.1, 0.1]$, compared to the asymptotic regret lower bound of Theorem~\ref{thm:low_bound}.}
    \label{fig:priv_reg}
\end{figure}

\section{Existing technical results and Definitions}\label{app:lemms}

\begin{proposition}[Post-processing~\citep{dwork2014algorithmic}]\label{prp:post_proc}
    Let $\mech$ be a mechanism and $f$ be an arbitrary randomised function defined on $\mech$'s output. If $\mech$ is $\epsilon$-DP, then $f\circ\mech$ is $\epsilon$-DP.
\end{proposition}
The post-processing property ensures that any quantity constructed only from a private output is still private, with the same privacy budget. This is a consequence of the data processing inequality.

\begin{proposition}[Group Privacy~\citep{dwork2014algorithmic}]\label{prp:grp_priv}
    Let $D$ and $D'$ be two datasets in $\mathcal{X}^n$. If $\mech$ is $(\epsilon, \delta)$-DP, then for any event $E  \in \mathcal{F}$
    \begin{align}\label{eq:grp_DP}
        \mech_{D}(A) \leq e^ {\epsilon \dham(D, D')} \mech_{D'}(A)~.
    \end{align}    
\end{proposition}
Group privacy translates the closeness of output distributions on neighbouring input datasets to a closeness of output distributions on any two datasets $D$ and $D'$ that depend on the Hamming distance $\dham(D, D')$. This property will be the basis for proving lower bounds in Section~\ref{sec:lower_bound}.

\begin{proposition}[Simple Composition]\label{prp:compo}
    Let $\mech^1, \dots, \mech^k$ be $k$ mechanisms. We define the mechanism $$\mathcal{G}: D \rightarrow \bigotimes_{i=1}^k \mathcal{M}^i_{ D}$$ as the $k$ composition of the mechanisms $\mech^1, \dots, \mech^k$. 

    \begin{itemize}
        \item If each $\mech^i$ is $(\epsilon_i, \delta_i)$-DP, then $\mathcal{G}$ is $(\sum_{i = 1}^k \epsilon_i, \sum_{i = 1}^k \delta_i)$-DP.
        \item If each $\mech^i$ is $\rho_i$-zCDP, then $\mathcal{G}$ is $\sum_{i = 1}^k \rho_i$-zCDP.
    \end{itemize}

\end{proposition}

Composition is a fundamental property of DP. Composition helps to analyse the privacy of sophisticated algorithms, by understanding the privacy of each building block, and summing directly the privacy budgets. Proposition~\ref{prp:compo} can be improved in two directions. (a) It is possible to show that the result is still true if the mechanisms are chosen adaptively, and that the mechanism at step $i$ takes as auxiliary input the outputs of the last $i- 1$ mechanisms. (b) Advanced composition theorems~\cite{kairouz2015composition} for $(\epsilon, \delta)$-DP improve the dependence on $k$ the number of composed mechanisms. Specifically, if the same mechanism is composed $k$ times, Proposition~\ref{prp:compo} concludes that the composed mechanism is $(k \epsilon, k \delta)$-DP. Advanced composition~\cite{kairouz2015composition} shows that the k-fold adaptively composed mechanism is $(\epsilon', \delta' + k \delta)$-DP for any $\delta'$ where $\epsilon' \deffn \sqrt{2 k \log(1/ \delta') \epsilon} + k \epsilon (e^\epsilon - 1)$. Roughly speaking, advanced composition provides a $(\sqrt{k} \epsilon, \delta)$-DP guarantee, improving by $\sqrt{k}$ the $(k \epsilon, k \delta)$-DP guarantee of simple composition.

In addition to the classic composition theorems, we provide here an additional property of interest: parallel composition. 

\begin{lemma}[Parallel Composition]\label{lem:paral_compo}
    % Let $\mathcal M$ be a mechanism that takes a \textbf{set} as input.
    Let $\mech^1, \dots, \mech^k$ be $k$ mechanisms, such that $k < n$, where $n$ is the size of the input dataset.
    Let $t_1, \ldots t_{k}, t_{k + 1}$ be indexes in $[1, n]$ such that $1 = t_1  < \cdots < t_k < t_{k+1}  - 1 = n$.\\
    Let's define the following mechanism
    $$
        \mathcal G : \{ x_1, \dots, x_n \} \rightarrow \bigotimes_{i=1}^k \mech^i_{ \{x_{t_i}, \ldots, x_{t_{i + 1} - 1} \}}
    $$
    %If $\mathcal{M}$ is $\epsilon$-DP, $\mathcal G$ is $\epsilon$-DP
    
    $\mathcal{G}$ is the mechanism that we get by applying each $\mech^i$ to the $i$-th partition of the input dataset $\{x_1, \dots, x_n\}$ according to the indexes $t_1  < \cdots < t_k < t_{k+1}$.
    %  \ie $(x_1, x_2, \ldots, x_T) \overset{\mathcal{G}}{\rightarrow} (o_1, o_2, \ldots, o_k)$, where $o_i \sim \mathcal{M}^i_{ \{x_{t_i}, \ldots, x_{t_{i + 1} - 1} \}}$.
    
    \begin{itemize}
        \item If each $\mech^i$ is $(\epsilon, \delta)$-DP, then $\mathcal G$ is $(\epsilon, \delta)$-DP
        \item If each $\mech^i$ is $\rho_i$-zCDP, then $\mathcal G$ is $\rho$-zCDP
    \end{itemize}
\end{lemma}

In parallel composition, the $k$ mechanisms are applied to different ``non-overlapping" parts of the input dataset. If each mechanism is DP, then the parallel composition of the $k$ mechanisms is DP, \emph{with the same privacy budget}. This property will be the basis for designing private bandit algorithms in Section~\ref{sec:alg_des}. 

\begin{theorem}[The Laplace Mechanism~\citep{dwork2014algorithmic}]\label{thm:laplace}
    Let $f: \mathcal{X} \rightarrow \real^k$ be a deterministic algorithm with $\ell_1$ sensitivity $s_1(f) \deffn \underset{D \sim D'}{\max} \left \Vert f(D) - f(D')\right \Vert_{1}$. Let $$\mathcal{M}_L(f, \epsilon) \deffn f + (Y_1, \dots, Y_k),$$ where $Y_i$ are i.i.d from $\mathrm{Lap}\left(\frac{s_1(f)}{\epsilon}\right)$, where the Laplace distribution centred at 0 with scale b, denoted $\mathrm{Lap}(b)$, is the distribution with probability density function
    \[
        \mathrm{Lap}(x | b) \deffn \frac{1}{2b} \exp \left( - \frac{|x|}{b} \right),
    \]
    for any $x \in \real$.
    
    The mechanism $\mathcal{M}_L(f, \epsilon)$ is called the Laplace mechanism and satisfies $\epsilon$-DP.
\end{theorem}

\begin{lemma}[Chernoff Tail Bound via KL Divergence \citep{boucheron2003concentration}]
\label{lem:ChernoffKL}
Let \( X_1, X_2, \dots, X_n \) be independent Bernoulli random variables with success probabilities \( p_1, p_2, \dots, p_n \). Define \( S_n = \sum_{i=1}^n X_i \), and let \( \mu = \mathbb{E}[S_n] = \sum_{i=1}^n p_i \). Then the following bounds hold:

\begin{itemize}
    \item Upper Tail Bound: for any \( a > \mu \) 
   \[
   P(S_n \geq a) \leq \exp\left(-n \cdot \mathrm{kl}\left(\frac{a}{n} , \frac{\mu}{n}\right)\right),
   \]
   where \( \mathrm{kl}(p , q) \) is defined as
   \[
   \mathrm{kl}(p, q) = p \log\frac{p}{q} + (1 - p) \log\frac{1 - p}{1 - q}~.
   \]
   \item Lower Tail Bound: for any \( a < \mu \) 
   \[
   P(S_n \leq a) \leq \exp\left(-n \cdot \mathrm{kl}\left(\frac{a}{n} , \frac{\mu}{n}\right)\right)~.
   \]
\end{itemize}
\end{lemma}

\begin{lemma}[Asymptotic Maximal Hoeffding Inequality]\label{lem:asum_hoeff} Assume that $X_i$ has positive mean $\mu$ and that $X_i - \mu$ is $\sigma$-sub-Gaussian. Then,
$$\forall \epsilon>0, \lim_{n \rightarrow \infty} \mathbb{P}\left ( \frac{\max_{s \leq n} \sum_{i = 1}^s X_i}{n} \leq (1 + \epsilon) \mu \right) = 1~.$$
\end{lemma}

\end{document}